\documentclass{article} % For LaTeX2e
\usepackage{iclr2021_conference,times}

\usepackage{amsmath,amsfonts,bm,amsthm}

\newcommand{\norm}[1]{\left\Vert#1\right\Vert}

\newcommand{\set}[1]{\left\{#1\right\}}
\newcommand{\parr}[1]{\left (#1\right )}

\newcommand{\Real}{\mathbb R}

\newcommand{\too}{\rightarrow}

\newcommand{\diag}{\textrm{diag}} %diagonal matrix
\newcommand{\eg}{{e.g.}}
\newcommand{\ie}{{i.e.}}

\newtheorem{theorem}{Theorem}
\newtheorem{lemma}{Lemma}

\newtheorem{definition}{Definition}

\def\eqref#1{equation~\ref{#1}}
\def\1{\bm{1}}

\def\ve{{\bm{e}}}

\def\vp{{\bm{p}}}

\def\vu{{\bm{u}}}
\def\vv{{\bm{v}}}

\def\vx{{\bm{x}}}
\def\vy{{\bm{y}}}
\def\vz{{\bm{z}}}

\def\mA{{\bm{A}}}
\def\mB{{\bm{B}}}

\def\mI{{\bm{I}}}

\def\mU{{\bm{U}}}
\def\mV{{\bm{V}}}

\DeclareMathAlphabet{\mathsfit}{\encodingdefault}{\sfdefault}{m}{sl}
\SetMathAlphabet{\mathsfit}{bold}{\encodingdefault}{\sfdefault}{bx}{n}

\def\gG{{\mathcal{G}}}

\def\gM{{\mathcal{M}}}
\def\gN{{\mathcal{N}}}

\def\gS{{\mathcal{S}}}
\def\gT{{\mathcal{T}}}

\def\gX{{\mathcal{X}}}

\def\Real{{\mathbb{R}}}

\newcommand{\E}{\mathbb{E}}

\DeclareMathOperator*{\argmin}{arg\,min}

\usepackage{hyperref}
\usepackage{url}

\usepackage{graphicx}
\usepackage[utf8]{inputenc} % allow utf-8 input
\usepackage[T1]{fontenc}    % use 8-bit T1 fonts
\usepackage{hyperref}       % hyperlinks
\usepackage{url}            % simple URL typesetting
\usepackage{booktabs}       % professional-quality tables
\usepackage{amsfonts}       % blackboard math symbols
\usepackage{nicefrac}       % compact symbols for 1/2, etc.
\usepackage{microtype}      % microtypography
\usepackage{multirow}
\usepackage[colorinlistoftodos]{todonotes}
\usepackage{dsfont}
\usepackage{comment}
\usepackage{amsmath,bm,amssymb} % define this before the line numbering.
\usepackage{color}
\usepackage{wrapfig}
\usepackage{adjustbox}
\usepackage{booktabs}
\usepackage{xr}
\usepackage{enumerate}
\usepackage{float}

\title{Isometric Autoencoders}

\author{Amos Gropp, Matan Atzmon \& Yaron Lipman \\
Weizmann Institute of Science \\
\texttt{\{amos.gropp,matan.atzmon,yaron.lipman\}@weizmann.ac.il} \\

}

\iclrfinalcopy % Uncomment for camera-ready version, but NOT for submission.

\begin{document}
\maketitle

\begin{abstract}

High dimensional data is often assumed to be concentrated on or near a low-dimensional manifold. Autoencoders (AE) is a popular technique to learn representations of such data by pushing it through a neural network with a low dimension bottleneck while minimizing a reconstruction error. Using high capacity AE often leads to a large collection of minimizers, many of which represent a low dimensional manifold that fits the data well but generalizes poorly.

Two sources of bad generalization are: extrinsic, where the learned manifold possesses extraneous parts that are far from the data; and intrinsic, where the encoder and decoder introduce arbitrary distortion in the low dimensional parameterization. An approach taken to alleviate these issues is to add a regularizer that favors a particular solution; common regularizers promote sparsity, small derivatives, or robustness to noise.

In this paper, we advocate an isometry (\ie, local distance preserving) regularizer. Specifically, our regularizer encourages: (i) the decoder to be an isometry; and (ii) the encoder to be the decoder's pseudo-inverse, that is, the encoder extends the inverse of the decoder to the ambient space by orthogonal projection. In a nutshell, (i) and (ii) fix both intrinsic and extrinsic degrees of freedom and provide a non-linear generalization to principal component analysis (PCA). 

Experimenting with the isometry regularizer on dimensionality reduction tasks produces useful low-dimensional data representations.

% AutoEncoders(AE) is a fundamental technique for learning meaningful data representations of data in an unsupervised fashion. The common training task of AE, reconstructing the input data, may promote poor generalization, thus raising the need for \emph{regularized} AE.
 
% In this paper, we advocate isometric AE which aim towards i) reconstructing the data manifold by an isometry to a low dimensional space; ii) learning representations that preserve input geometrical properties. Our key observation is that by mutually regularizing the AE by the above two properties leads to better approximation of the data manifold.
 
\end{abstract}

\section{Introduction}
\label{s:intro}
A common assumption is that high dimensional data $\gX \subset \Real^D$ is sampled from some distribution $p$ concentrated on, or near, some lower $d$-dimensional submanifold $\gM\subset \Real^D$, where $d < D$. The task of estimating $p$ can therefore be decomposed into: (i) approximate the manifold $\gM$; and (ii) approximate $p$ restricted to, or concentrated near $\gM$. 

In this paper we focus on task (i), mostly known as \emph{manifold learning}. A common approach to approximate the $d$-dimensional manifold $\gM$, \eg, in \citep{tenenbaum2000global,roweis2000nonlinear,belkin2002laplacian, maaten2008visualizing, mcqueen2016nearly, mcinnes2018umap}, is to embed $\gX$ in $\Real^d$. This is often done by first constructing a graph $\gG$ where nearby samples in $\gX$ are conngected by edges, and second, optimizing for the locations of the samples in $\Real^d$ striving to minimize edge length distortions in $\gG$.  

Autoencoder (AE) can also be seen as a method to learn low dimensional manifold representation of high dimensional data $\gX$. AE is trying to reconstruct $\gX$ as the image of its low dimensional embedding. When restricting AE to linear encoders and decoders it learns linear subspaces; with mean squared reconstruction loss they reproduce principle component analysis (PCA). Using higher capacity neural networks as the encoder and decoder, allows complex manifolds to be approximated. To avoid overfitting, different regularizers are added to the AE loss. Popular regularizers include sparsity promoting \citep{ranzato2007efficient,ranzato2008sparse,glorot2011deep}, contractive or penalizing large derivatives \citep{rifai2011learning,Rifai2011ContractiveAE}, and denoising \citep{vincent2010stacked,poole2014analyzing}. Recent AE regularizers directly promote distance preservation of the encoder \citep{pai2019dimal,peterfreund2020loca}.

In this paper we advocate a novel AE regularization promoting isometry (\ie, local distance preservation), called Isometric-AE (I-AE). Our key idea is to promote the decoder to be isometric, and the encoder to be its \emph{pseudo-inverse}. Given an isometric encoder $\Real^d\too\Real^D$, there is no well-defined inverse $\Real^D\too\Real^d$; we define the pseudo-inverse to be a projection on the image of the decoder composed with the inverse of the decoder restricted to its image. 

Locally, the I-AE regularization therefore encourages: (i) the differential of the decoder $\mA\in\Real^{D\times d}$ to be an isometry, \ie, $\mA^T\mA = \mI_d$, where $\mI_d$ is the $d\times d$ identity matrix; and (ii) the differential of the encoder, $\mB\in\Real^{d\times D}$ to be the pseudo-inverse (now in the standard linear algebra sense) of the differential of the decoder $\mA\in\Real^{D\times d}$, namely, $\mB=\mA^+$. In view of (i) this implies $\mB=\mA^T$.
This means that \emph{locally} our decoder and encoder behave like PCA, where the encoder and decoder are linear transformations satisfying (i) and (ii); That is, the PCA encoder can be seen as a composition of an orthogonal projection on the linear subspace spanned by the decoder, followed by an orthogonal transformation (isometry) to the low dimensional space.

\begin{wraptable}[16]{r}{0.26\textwidth}
\vspace{-10pt}
\begin{tabular}{c}
\includegraphics[width=0.24\textwidth]{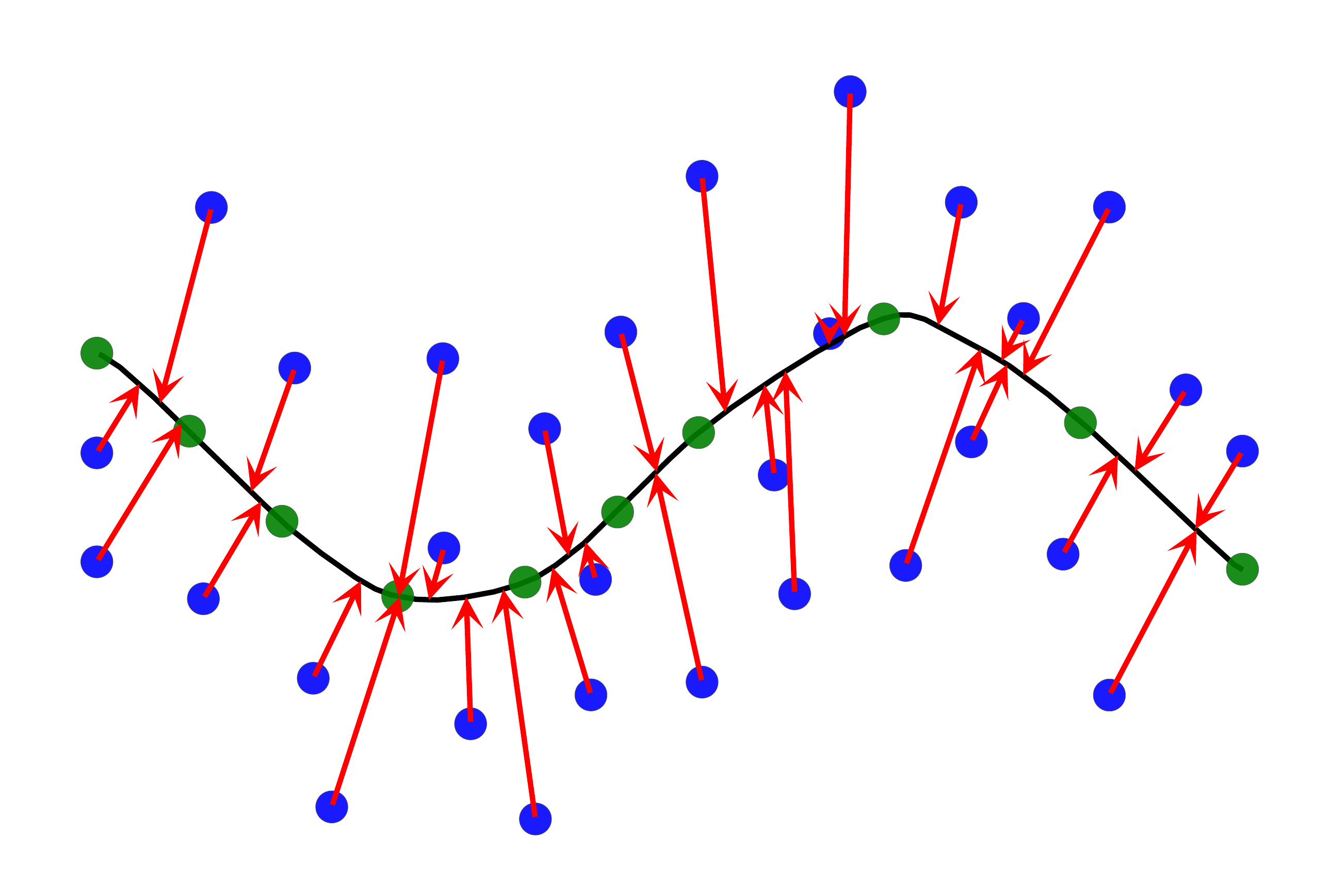}\\
\includegraphics[width=0.24\textwidth]{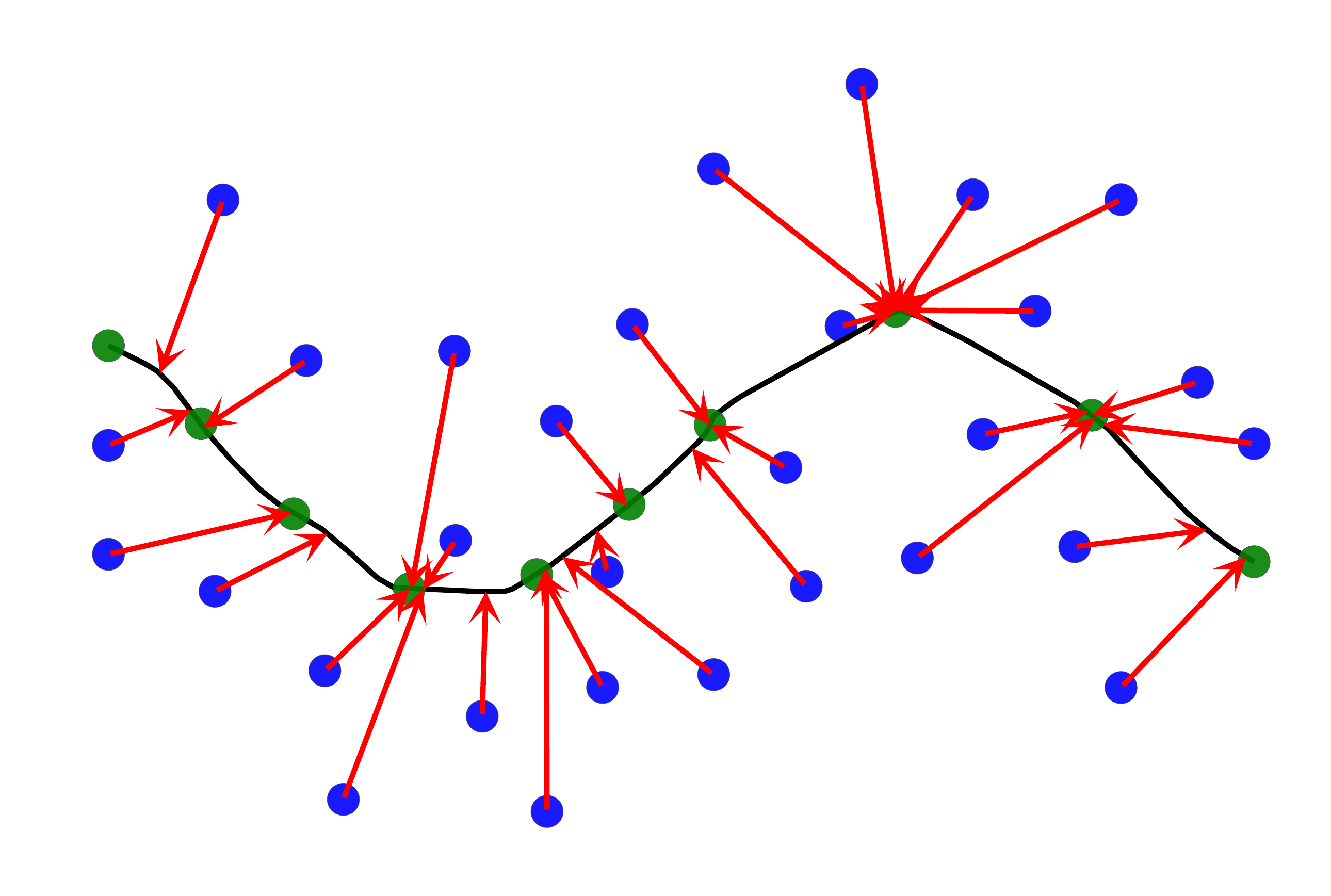}
\end{tabular}\caption{Top: I-AE; bottom: CAE.}\label{tab:teaser}
\end{wraptable} 
In a sense, our method can be seen as a version of denoising/contractive AEs (DAE/CAE, respectively). DAE and CAE promote a projection from the ambient space onto the data manifold, but can distort distances and be non-injective. Locally, using differentials again, projection on the learned manifold means $(\mA\mB)^2=\mA\mB$. Indeed, as can be readily checked conditions (i) and (ii) above imply $\mA(\mB\mA)\mB=\mA\mB$. This means that I-AE also belongs to the same class of DAE/CAE, capturing the variations in tangent directions of the data, $\gM$, while ignoring orthogonal variations which often represent noise \citep{vincent2010stacked,alain2014regularized}.  The benefit in I-AE is that its projection on the data manifold is locally an isometry, preserving distances and sampling the learned manifold evenly. The inset illustrates a simple experiment comparing contractive AE (CAE-bottom) and isometric AE (I-AE-top). Both AEs are trained on the green data points; the red arrows depict projection of points (in blue) in vicinity of the data onto the learned manifold (in black) as calculated by applying the encoder followed by the decoder. Note that CAE indeed projects on the learned manifold but not evenly, tending to shrink space around data points; in contrast I-AE provides a more even sampling of the learned manifold.

Experiments confirm that optimizing the I-AE loss results in a close-to-isometric encoder/decoder explaining the data. We further demonstrate the efficacy of I-AE for dimensionality reduction of different standard datatsets, showing its benefits over manifold learning and other AE baselines.

\section{Related works}

\paragraph{Manifold learning.}
Manifold learning generalize classic dimensionality reduction methods such as PCA \citep{doi:10.1080/14786440109462720} and MDS \citep{kruskal1964multidimensional,sammon1969nonlinear}, by aiming to preserve the local geometry of the data.   \cite{tenenbaum2000global} use the nn-graph to approximate the geodesic distances over the manifold, followed 
by MDS to preserve it in the lower dimension. \cite{roweis2000nonlinear, belkin2002laplacian,donoho2003hessian} use spectral methods to minimize different distortion energy functions over the graph matrix. \cite{coifman2005geometric, coifman2006diffusion} approximate the heat diffusion over the manifold by a random walk over the nn-graph, to gain a robust distance measure on the manifold. Stochastic neighboring embedding algorithms \citep{hinton2003stochastic, maaten2008visualizing} captures the local geometry of the data as a mixture of Gaussians around each data points, and try to find a low dimension mixture model by minimizing the KL-divergence.
In a relatively recent work, \cite{mcinnes2018umap} use iterative spectral and embedding optimization using fuzzy sets. 
%\paragraph{}
Several works tried to adapt classic manifold learning ideas to neural networks and autoencoders.
\cite{pai2019dimal} suggest to embed high dimensional points into a low dimension with a neural network by constructing a metric between pairs of data points and minimizing the metric distortion energy. \cite{kato2019rate} suggest to learn an isometric decoder by using noisy latent variables. They prove under certain conditions that it encourages isometric decoder. \cite{peterfreund2020loca} suggest autoencoders that promote the isometry of the encoder over the data by approximating its differential gram matrix using sample covariance matrix. 
\cite{10.1145/3240508.3240607} encourage distance preserving autoencoders by minimizing metric distortion energy in common feature space.

\paragraph{Modern autoencoders.}
There is an extensive literature on extending autoencoders to a generative model (task (ii) in section \ref{s:intro}). That is, learning a probability distribution in addition to approximating the data manifold $\gM$. Variational autoencoder (VAE) \cite{Kingma2014AutoEncodingVB} and its variants \cite{makhzani2015adversarial,Burda2016ImportanceWA,sonderby2016ladder,higgins2017beta,abc,park2019,zhao2019infovae} are examples to such methods. In essence, these methods augment the AE structure with a learned probabilistic model in the low dimensional (latent) space $\Real^d$ that is used to approximate the probability $P$ that generated the observed data $\gX$.
More relevant to our work, are recent works suggesting regularizers for deterministic autoencoders that together with ex-post density estimation in latent space forms a generative model. \cite{ghosh2020from} suggested to reduce the decoder degrees of freedom, either by regularizing the norm of the decoder weights or the norm of the decoder differential. Other regularizers of the differential of the decoder, aiming towards a deterministic variant of VAE, were recently suggested in \cite{kumar2020implicit,kumar2020regularized}. In contrast to our method, these methods do not regularize the encoder explicitly.

% The standard modeling defines the probability of a data point $\vx\in\gX$ using the law of total probability $P(\vx)=\int_{\Real^d} P(\vx|f(\vz))P(\vz)d\vz$, where $f$ denotes the decoder and the conditional probability $P(\vx | f(\vz))$ is usually taken to be a gaussian centered at the decoded value of $\vz$.

% \paragraph{Normalizing and injective flows}
% XXX
%\yl{google (murphy paper)
%kyle, injective flows} 

%\section{Notation}
%We denote by $[n]=\set{1,2,\ldots,n}$. The data space parameter is $\vx=(x_1,\ldots,x_D)\in\Real^D$, and the latent space %parameter is $\vz=(z_1,\ldots,z_d)\in\Real^d$. The differential matrix of differential map $f=(f^1,\ldots,f^D)^T:%\Real^d\too\Real^D$ is defined as $df(\vz)_{ij}=\frac{\partial f^i}{\partial z_j}(\vz)$.
\section{Isometric autoencoders}
\begin{wrapfigure}[6]{r}{0.26\textwidth}
  \begin{center}\vspace{-20pt}
    \includegraphics[width=0.24\textwidth]{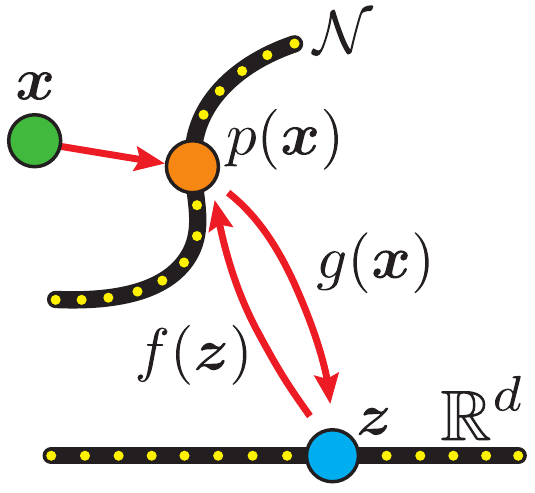}\vspace{-0pt}
  \end{center}
  \caption{I-AE.}\label{fig:setup}
\end{wrapfigure}
We consider high dimensional data points $\gX=\set{\vx_i}_{i=1}^n \subset \Real^D$ sampled from some probability distribution $P(\vx)$ in $\Real^D$ concentrated on or near some $d$ dimensional submanifold $\gM\subset \Real^D$, where $d<D$. 

Our goal is to compute \emph{isometric autoencoder} (I-AE) defined as follows. Let $g:\Real^D\too\Real^d$ denote the encoder, and $f:\Real^d\too\Real^D$ the decoder; $\gN$ is the learned manifold, \ie, the image of the decoder, $\gN=f(\Real^d)$.
I-AE is defined by the following requirements:

\begin{enumerate}[(i)]
\item The data $\gX$ is close to $\gN$.  \label{item:f_include_gX}
%\item $g$ is the inverse of $f$ when restricted to $\gN$. \label{item:g_inv_f} 
\item $f$ is an isometry. \label{item:f_iso}
\item $g$ is the pseudo-inverse of $f$. \label{item:g_piso}
\end{enumerate}

Figure \ref{fig:setup} is an illustration of I-AE. Let $\theta$ denote the parameters of $f$, and $\phi$ the parameters of $g$. We enforce the requirements (\ref{item:f_include_gX})-(\ref{item:g_piso}) by prescribing a loss function $L(\theta,\phi)$ and optimize it using standard stochastic gradient descent (SGD). We next break down the loss $L$ to its different components.

Condition (\ref{item:f_include_gX}) is promoted with the standard reconstruction loss in AE:
\begin{equation}\label{e:recon}
L_{\text{rec}}(\theta,\phi) = \frac{1}{n}\sum_{i=1}^n \norm{f(g(\vx_i))-\vx_i}^2,
\end{equation}
where $\norm{\cdot}$ is the 2-norm. 
%  
%

%Condition (\ref{item:g_inv_f}) is promoted using the loss: 
% \begin{equation}\label{e:loss_inv}
%     L_{\text{inv}}(\theta,\phi) = \E_{\vz} \norm{g(f(\vz))-\vz}^2,
% \end{equation}
% where $\vz\sim P_{\text{inv}}(\Real^d)$, and $P_{\text{inv}}(\vz)$ is some probability measure on $\Real^d$. 

Before handling conditions (\ref{item:f_iso}),(\ref{item:g_piso}) let us first define the notions of isometry and pseudo-inverse. A differentiable mapping $f$ between the euclidean spaces $\Real^d$ and $\Real^D$ is a local isometry if it has an orthogonal differential matrix $df(\vz)\in\Real^{D\times d}$, 
\begin{equation}\label{e:iso}
df(\vz)^T df(\vz) = \mI_d, 
\end{equation}
where $\mI_d\in\Real^{d\times d}$ is the identity matrix, and $df(\vz)_{ij}=\frac{\partial f^i}{\partial z_j}(\vz)$. A local isometry which is also a diffeomorphism is a global isometry. Restricting the decoder to isometry is beneficial for several reasons. First, Nash-Kuiper Embedding Theorem \cite{nash1956imbedding} asserts that non-expansive maps can be approximated arbitrary well with isometries if $D\geq d+1$ and hence promoting an isometry does not limit the expressive power of the decoder. Second, the low dimensional representation of the data computed with an isometric encoder preserves the geometric structure of the data. In particular volume, length, angles and probability densities are preserved between the low dimensional representation $\Real^d$, and the learned manifold $\gN$. Lastly, for a fixed manifold $\gN$ there is a huge space of possible decoders such that $\gN=f(\Real^d)$. For isometric $f$, this space is reduced considerably: Indeed, consider two isometries parameterizing $\gN$, \ie, $f_1,f_2:\Real^d\too \gN$. Then, since composition of isometries is an isometry we have that $f_2^{-1}\circ f_1:\Real^d\too\Real^d$ is a dimension-preserving isometry and hence a rigid motion. That is, all decoders of the same manifold are the same up to a rigid motion.

For the encoder the situation is different. Since $D>d$ the encoder $g$ cannot be an isometry in the standard sense. Therefore we ask $g$ to be the \emph{pseudo-inverse} of $f$. For that end we define the projection operator $\vp$ on a submanifold $\gN\subset \Real^D$ as $$\vp(\vx)=\argmin_{\vx'\in\gN}\norm{\vx-\vx'}.$$
\begin{definition}\label{def:piso}
We say the $g$ is the \emph{pseudo-inverse} of $f$ if $g$ can be written as $g=f^{-1}\circ \vp$, where $\vp$ is the projection on $\gN=f(\Real^d)$. \end{definition}
Consequently, if $g$ is the pseudo-inverse of an isometry $f$ then it extends the standard notion of isometry by projecting every point on a submanifold $\gN$ and then applying an isometry between the $d$-dimensional manifolds $\gN$ and $\Real^d$. See Figure \ref{fig:setup} for an illustration. 
%Note that low reconstruction loss in \eqref{e:recon} also promotes the encoder and decoder to be the inverse of each other when restricted to the $\gN$.

\textbf{First-order characterization.} To encourage $f,g$ to satisfy the (local) isometry and the pseudo-inverse properties (resp.) we will first provide a first-order (necessary) characterization using their differentials:
\begin{theorem}\label{thm:iso_piso}
Let $f$ be a decoder and $g$ an encoder satisfying conditions (\ref{item:f_iso}),(\ref{item:g_piso}). Then their differentials $\mA=df(\vz)\in\Real^{D\times d}$, $\mB=dg(f(\vz))\in\Real^{d\times D}$ satisfy 
\begin{align} \label{e:AtA_I}
    \mA^T \mA &= \mI_d  \\ \label{e:BBt_I}
    \mB\mB^T &= \mI_d \\ \label{e:M_eq_Ainv}
    \mB &= \mA^T
\end{align}
\end{theorem}
The theorem asserts that the differentials of the encoder and decoder are orthogonal (rectangular) matrices, and that the encoder is the pseudo-inverse of the differential of the decoder. Before proving this theorem, let us first use it to construct the relevant losses for promoting the isometry of $f$ and pseudo-inverse $g$. We need to promote conditions (\ref{e:AtA_I}), (\ref{e:BBt_I}), (\ref{e:M_eq_Ainv}). Since we want to avoid computing the full differentials $\mA=df(\vz)$, $\mB=dg(f(\vz))$, we will replace (\ref{e:AtA_I}) and (\ref{e:BBt_I}) with stochastic estimations based on the following lemma: denote the unit $d-1$-sphere by $\gS^{d-1}=\set{\vz\in\Real^d \vert \norm{\vz}=1}$.
\begin{lemma}\label{lem:orth}
Let $\mA\in\Real^{D\times d}$, where $d\leq D$. If $ \norm{\mA\vu} = 1$ for all $\vu\in \gS^{d-1}$, then $\mA$ is column-orthogonal, that is $\mA^T\mA=\mI_d$. 
\end{lemma}
Therefore, the isometry promoting loss, encouraging (\ref{e:AtA_I}), is defined by 
\begin{equation}\label{e:loss_iso}
L_{\text{iso}}(\theta)= \E_{\vz,\vu} \Big( \norm{ df(\vz)\vu } - 1\Big )^2,
\end{equation}
where $\vz\sim P_{\text{iso}}(\Real^d)$, and $P_{\text{iso}}(\Real^d)$ is a probability measure on $\Real^d$; $\vu\sim P(\gS^{d-1})$, and $P(\gS^{d-1})$ is the standard rotation invariant probability measure on the $d-1$-sphere $\gS^{d-1}$. 
The pseudo-inverse promoting loss, encouraging (\ref{e:BBt_I}) would be
\begin{equation}\label{e:loss_piso}
L_{\text{piso}}(\phi)= \E_{\vx,\vu} \Big( \norm{ \vu^Tdg(\vx) } - 1\Big )^2,
\end{equation}
where $\vx\sim P(\gM)$ and $\vu\sim P(\gS^{d-1})$. As-usual, the expectation with respect to $P(\gM)$ is 
computed empirically using the data samples $\gX$. 

Lastly, (\ref{e:M_eq_Ainv}) might seem challenging to enforce with neural networks, however the orthogonality of $\mA,\mB$ can be leveraged to replace this loss with a more tractable loss asking the encoder is merely the inverse of the decoder over its image:
\begin{lemma}\label{lem:B_is_A+}
Let $\mA\in\Real^{D\times d}$, and $\mB\in\Real^{d\times D}$. If $\mA^T\mA=\mI_d=\mB\mB^T$ and $\mB\mA=\mI_d$ then $\mB=\mA^+=\mA^T$.
\end{lemma}
Fortunately, this is already taken care of by the reconstruction loss: since low reconstruction loss in \eqref{e:recon} forces the encoder and the decoder to be the inverse of one another over the data manifold, \ie $g(f(\vz))=\vz$, it encourages  $\mB\mA=\mI_d$ and therefore, by Lemma \ref{lem:B_is_A+}, automatically encourages \eqref{e:M_eq_Ainv}. Note that invertability also implies bijectivity of the encoder/decoder restricted to the data manifold, pushing for global isometries (rather than local). 
%
%Indeed, differentiating both sides of the equation $\vz=g(f(\vz))$, and using the chain rule, we get $\mI_d = dg(f(\vz))df(\vz)=\mB\mA$. Given (\ref{e:AtA_I}), (\ref{e:BBt_I}), Lemma \ref{lem:B_is_A+} now implies that $\mB=\mA^+$, as desired. 
%
Summing all up, we define our loss for I-AE by 
\begin{equation}\label{eq:loss}
L(\theta,\phi) = L_{\text{rec}}(\theta,\phi) +  \lambda_{\text{iso}} \parr{L_{\text{iso}}(\theta) + L_{\text{piso}}(\phi)},
\end{equation}
where $\lambda_{\text{iso}}$ is a parameter controlling the isometry-reconstruction trade-off. 

\subsection{Details and proofs.} 
Let us prove Theorem \ref{thm:iso_piso} characterizing the relation of the differentials of isometries and pseudo-isometries, $\mA=df(\vz)\in\Real^{D\times d}$,  $\mB=dg(f(\vz))\in\Real^{d\times D}$. First, by definition of isometry (\eqref{e:iso}), $\mA^T\mA=\mI_d$.  
We denote by $T_\vx \gN$ the $d$-dimensional tangent space to $\gN$ at $\vx\in\gN$; accordingly, $T_{\vx}\gN^\perp$ denotes the normal tangent space.
\begin{lemma}\label{lem:dp}
The differential $d\vp(\vx)\in\Real^{D\times D}$ at $\vx\in\gN$ of the projection operator $\vp:\Real^D\too \gN$ is 
\begin{equation}
d\vp(\vx)\vu = \begin{cases} \vu & \vu \in T_{\vx}\gN \\ 0 & \vu\in  T_{\vx}\gN^\perp  \end{cases} \end{equation}
That is, $d\vp(\vx)$ is the orthogonal projection on the tangent space of $\gN$ at $\vx$.
\end{lemma} 
\begin{proof}
First, consider the squared distance function to $\gN$ defined by $\eta(\vx)=\frac{1}{2}\min_{\vx'\in\gN}\norm{\vx-\vx'}^2$. The envelope theorem implies that $\nabla \eta(\vx) = \vx-\vp(\vx)$. Differentiating both sides and rearranging we get $d\vp(\vx)=\mI_D - \nabla^2 \eta(\vx)$. As proved in \cite{ambrosio1994level} (Theorem 3.1), $\nabla^2 \eta(\vx)$ is the orthogonal projection on $T_{\vx}\gN^\perp$.
\end{proof}
%\matan{Thus dp(x) = AA^T, since AA^T is thw pojection on TxN). This implies the charcterization of dg.Indeed,}

Let $\vx=f(\vz)\in \gN$. Since $\vx\in\gN$ we have $\vp(\vx)=\vx$. Condition (\ref{item:g_piso}) asserts that $g(\vy)=f^{-1}(\vp(\vy))$; taking the derivative at $\vy=\vx$ we get $dg(\vx)=df^{-1}(\vx)d\vp(\vx)$. Lemma \ref{lem:dp} implies that $d\vp(\vx) = \mA \mA^T$, since $\mA\mA^T$ is the orthogonal projection on $\gT_\vx \gN$. Furthermore, $df^{-1}(\vx)$ restricted to $\mathrm{Im}(\mA)$ is $\mA^T$. Putting this together we get $\mB=dg(\vx)=\mA^T\mA\mA^T=\mA^T$. This implies that $\mB\mB^T=\mI_d$, and that $\mB=\mA^+=\mA^T$. This concludes the proof of Theorem \ref{thm:iso_piso}. \qed

%Thus $d\vp(\vx) = \mA \mA^T$, since $\mA\mA^T$ is the linear projection on $\gT_\vx \gN$. This leads us to the characterization of $\mB$. Indeed, let $\vx=f(\vz)\in\gN$. Then $\vp(\vx)=\vx$, and $g(\vx)=g(\vp(\vx))$ as $g$ is a pseudo-isometry. Using the chain rule implies $\mB = dg(\vp(\vx))d\vp(\vx) = \mB \mA \mA^T$. 
%Then, $\mB\mB^T=\mB \mA \mA^T\mB^T=\mB \mA (\mB \mA)^T = \mI_d$ where the last equality follows from the fact that $\mB\mA$ is an orthogonal matrix as the restriction of $g$ to $\gN$ is an isometry. We saw above that $\mA^T\mA=\mI_d$ and $\mB\mA=\mI_d$, therefore lemma \ref{lem:B_is_A+} implies that $\mB=\mA^+$.

%This lemma can be used to understand the form of the differential $dg$ of a pseudo-isometry over the manifold $\gN$. Indeed, let $\vx=f(\vz)\in\gN$, then $\vp(\vx)=\vx$. Indeed, since a pseudo-isometry satisfies $g(\vx)=g(\vp(\vx))$ (see definition \ref{def:piso}) using the chain rule again and lemma \ref{lem:dp} we have $\mB = dg(\vx) = dg(\vp(\vx))d\vp(\vx) = \mB \mA \mA^T$, since $\mA\mA^T$ is the projection on $\gT_\vx \gN$.
%Furthermore, since the restriction of $g$ to $\gN$ is an isometry, we have that $\mB\mA$ is an orthogonal matrix. This means that $\mB\mA\mA^T$ is orthogonal and from the previous equality we get $\mB\mB^T=\mI_d$. We saw above that $\mA^T\mA=\mI_d$ and $\mB\mA=\mI_d$, therefore lemma \ref{lem:B_is_A+} implies that $\mB=\mA^+$. 

\begin{proof}[Proof of Lemma \ref{lem:orth}.]
Writing the SVD of $\mA=\mU\Sigma\mV^T$, where $\Sigma=\diag(\sigma_1,\ldots,\sigma_d)$ are the singular values of $\mA$, we get that $\sum_{i=1}^d \sigma_i^2 \vv_i^2 = 1$ for all $\vv\in \gS^{d-1}$. Plugging $\vv=\ve_j$, $j\in[d]$ (the standard basis) we get that all $\sigma_i=1$ for $i\in[d]$ and $\mA=\mU\mV^T$ is orthogonal as claimed. 
\end{proof}

\begin{proof}[Proof of Lemma \ref{lem:B_is_A+}.]
Let $\mU=[\mA,\mV]$, $\mV\in\Real^{D\times(D-d)}$, be a completion of $\mA$ to an orthogonal matrix in $\Real^{D\times D}$. Now, $\mI_d = \mB\mU \mU^T \mB^T = \mI_d + \mB\mV\mV^T\mB^T$, and since $\mB\mV\mV^T\mB^T\succeq 0$ this means that $\mB\mV=0$, that is $\mB$ takes to null the orthogonal space to the column space of $\mA$. A direct computation shows that $\mB\mU=\mA^T\mU$ which in turn implies $\mB=\mA^T=\mA^+$.
\end{proof}

\paragraph{Implementation.}
Implementing the losses in \eqref{e:loss_iso} and  \eqref{e:loss_piso} requires making a choice for the probability densities and approximating the expectations. We take $P_{\text{iso}}(\Real^d)$ to be either uniform or gaussian fit to the latent codes $g(\gX)$; and $P(\gM)$ is approximated as the uniform distribution on $\gX$, as mentioned above. The expectations are estimated using Monte-Carlo sampling. That is, at each iteration we draw samples $\hat{\vx}\in \gX$, $\hat{\vz}\sim P_{\text{iso}}(\Real^d)$, $\hat{\vu}\sim P(\gS^{d-1})$ and use the approximations
\begin{align*}
L_{\text{iso}}(\theta) &\approx\big(\norm{df(\hat{\vz})\hat{\vu}}-1\big)^2\\
L_{\text{piso}}(\phi) &\approx \big(\norm{\hat{\vu}^T dg(\hat{\vx})}-1\big)^2
\end{align*}
The right differential multiplication $df(\hat{\vz})\hat{\vu}$ and left differential multiplication $\hat{\vu}^Tdg(\hat{\vx})$ are computed using forward and backward mode automatic differentiation (resp.). Their derivatives with respect to the networks' parameters $\theta,\phi$ are computed by another backward mode automatic differentiation. %\yl{is it correct?}

%\yl{we want $df^T=dg$ this cannot be done by weights tying etc. if  $dg, df$ are orthogonal and inverse then it is true. this is like PCA. }

%\paragraph{Relation to principle component analysis}

%\paragraph{Relation to manifold learning methods.}

%XXX
%in manifold learning - computing distance in high dim is expensive.
%also, our loss is local but enforce global constraint.  
 
%\yl{continuity of loss. ie prove small epsilon loss implies epsilon isometry....}

\section{Experiments}

\subsection{Evaluation}
We start by evaluating the effectiveness of our suggested I-AE regularizer, addressing the following questions: (i) does our suggested loss $L\left(\theta,\phi\right)$ in \eqref{eq:loss} drives I-AE training to converge to an isometry? (ii) What is the effect of the $L_{\text{piso}}$ term? In particular, does it encourage better manifold approximations as conjectured?
To that end, we examined the I-AE training on data points $\gX$ sampled uniformly from 3D surfaces with known global parameterizations. Figure \ref{fig:evaluation} shows qualitative comparison of the learned embeddings for various AE regularization techniques: Vanilla autoencoder (AE); Contractive autoencoder (CAE) \citep{Rifai2011ContractiveAE}; Contractive autoencoder with decoder weights tied to the encoder weights (TCAE) \citep{rifai2011learning}; Gradient penalty on the decoder (RAE-GP) \citep{ghosh2020from}; and Denoising autoencoder with gaussian noise (DAE) \citep{vincent2010stacked}. For fairness in evaluation, all methods were trained using the same training hyper-parameters. See Appendix for the complete experiment details including mathematical formulation of the different AE regularizers. In addition, we compared versus popular classic manifold learning techniques: U-MAP \citep{mcinnes2018umap}, t-SNE \citep{maaten2008visualizing} and LLE.  \citep{roweis2000nonlinear}. The results demonstrate that I-AE is able to learn an isometric embedding, showing some of the advantages in our method: sampling density and distances between input points is preserved in the learned low dimensional space.

\begin{wraptable}[6]{r}{0.51\columnwidth}
\centering\vspace{-5pt}
\resizebox{0.5\textwidth}{!}{    
\begin{tabular}{lrrrrrr} 
                & I-AE                & CAE & TCAE & RAE-GP & DAE & AE \\
    \midrule
    S Shape    & \textbf{0.03}      & 0.36 & 0.26 & 1.22 & 2.53 & 1.85   \\
    Swiss Roll    & \textbf{0.02}      & 1.00 & 0.38 & 1.75 & 1.80 & 1.63   \\
    Open Sphere   & \textbf{0.07} & 0.21 & 0.21 & 0.50 & 1.09 & 1.29  \\
    \end{tabular} 
    }
    \caption{Std of $\left\{l_{ij}\right\}$.}
    \label{tab:std}
\end{wraptable}
In addition, for the AE methods, we quantitatively evaluate how close is the learnt decoder to an isometry. For this purpose, we triangulate a grid of planar points $\left\{\vz_i\right\}\subset \Real^2$. We denote by $\left\{e_{ij}\right\}$ the triangles edges incident to grid points $\vz_i$ and $\vz_j$. Then, we measured the edge lengths ratio, $l_{ij} = {\norm{f\left(\vz_i\right) - f\left(\vz_i\right)}}/{\norm{e_{ij}}}$ 
expected to be $\approx 1$ for all edges $e_{ij}$ in an isometry.
In Table \ref{tab:std} we log the standard deviation (Std) of $\left\{l_{ij}\right\}$ for I-AE compared to other regularized AEs. For a fair comparison, we scaled $\vz_i$ so the mean of $l_{ij}$ is $1$ in all experiments. As can be seen in the table, the distribution of $\left\{l_{ij}\right\}$ for I-AE is significantly more concentrated than the different AE baselines. 

%\subsection{Importance of the pseudo-isometric loss.}
\begin{wrapfigure}[9]{r}{0.4\textwidth}
  \begin{center}\vspace{-22pt}
    \includegraphics[width=0.4\textwidth]{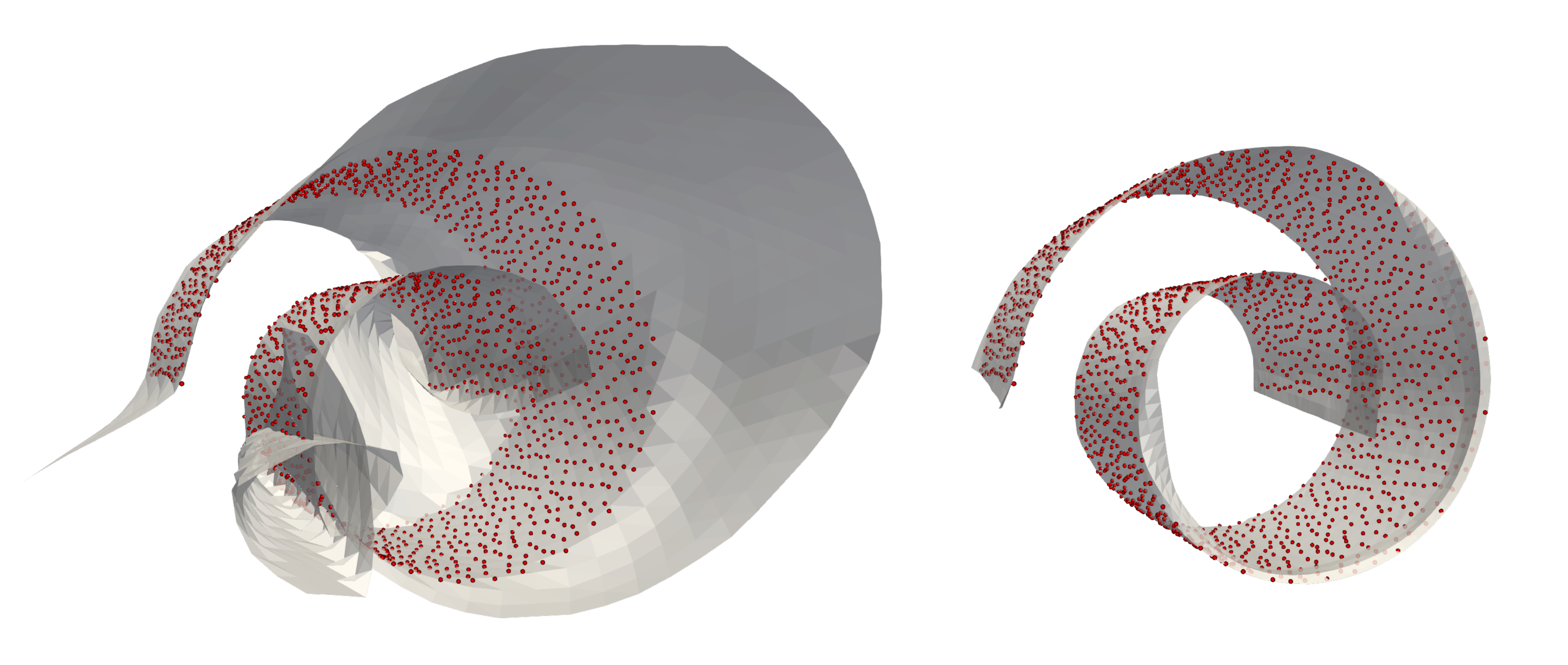}
  \end{center}
  \vspace{-15pt}
  \caption{Decoder surfaces without $L_{\text{piso}}$ (left) and with (right).}\label{fig:no_piso}
  \vspace{-20pt}
\end{wrapfigure}
Finally, although $L_{\text{iso}}$ is already responsible for learning an isometric decoder, the pseudo-inverse encoder (enforced by the loss $L_{\text{piso}}$) helps it converge to simpler solutions. We ran AE training with and without the $L_{\text{piso}}$ term. 
Figure \ref{fig:no_piso} shows in gray the learnt decoder surface, $\gN$, without $L_{\text{piso}}$ (left), containing extra (unnatural) surface parts compared to the learnt surface with $L_{\text{piso}}$ (right). 
In both cases we expect (and achieve) a decoder approximating an isometry that passes through the input data points. Nevertheless, the pseudo-inverse loss restricts some of the degrees of freedom of the encoder which in turn leads to a simpler solution.

\begin{figure}[t]
    \centering
    \setlength\tabcolsep{0.0pt}
    \begin{tabular}{c|cccccc|ccc}
        3D Data& \textbf{I-AE} &
        CAE& TCAE &
        RAE-GP&
        \shortstack{DEA}&
        AE&
        U-MAP&
        t-SNE&
        LLE\\
        \hline
        \includegraphics[width=0.1\textwidth]{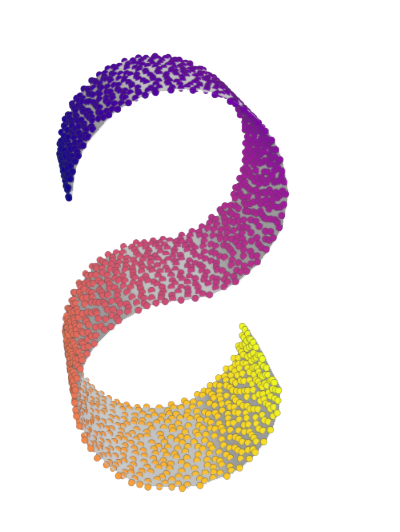}&
        \includegraphics[width=0.1\textwidth]{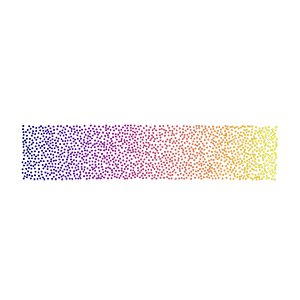}&
        \includegraphics[width=0.1\textwidth]{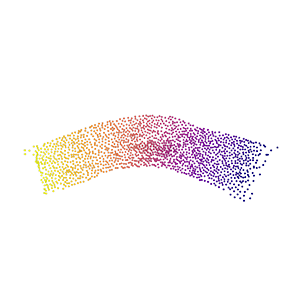}&
        \includegraphics[width=0.1\textwidth]{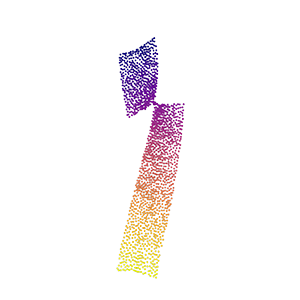}&
        \includegraphics[width=0.1\textwidth]{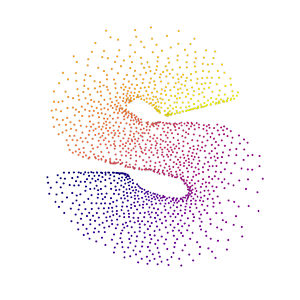}&
        \includegraphics[width=0.1\textwidth]{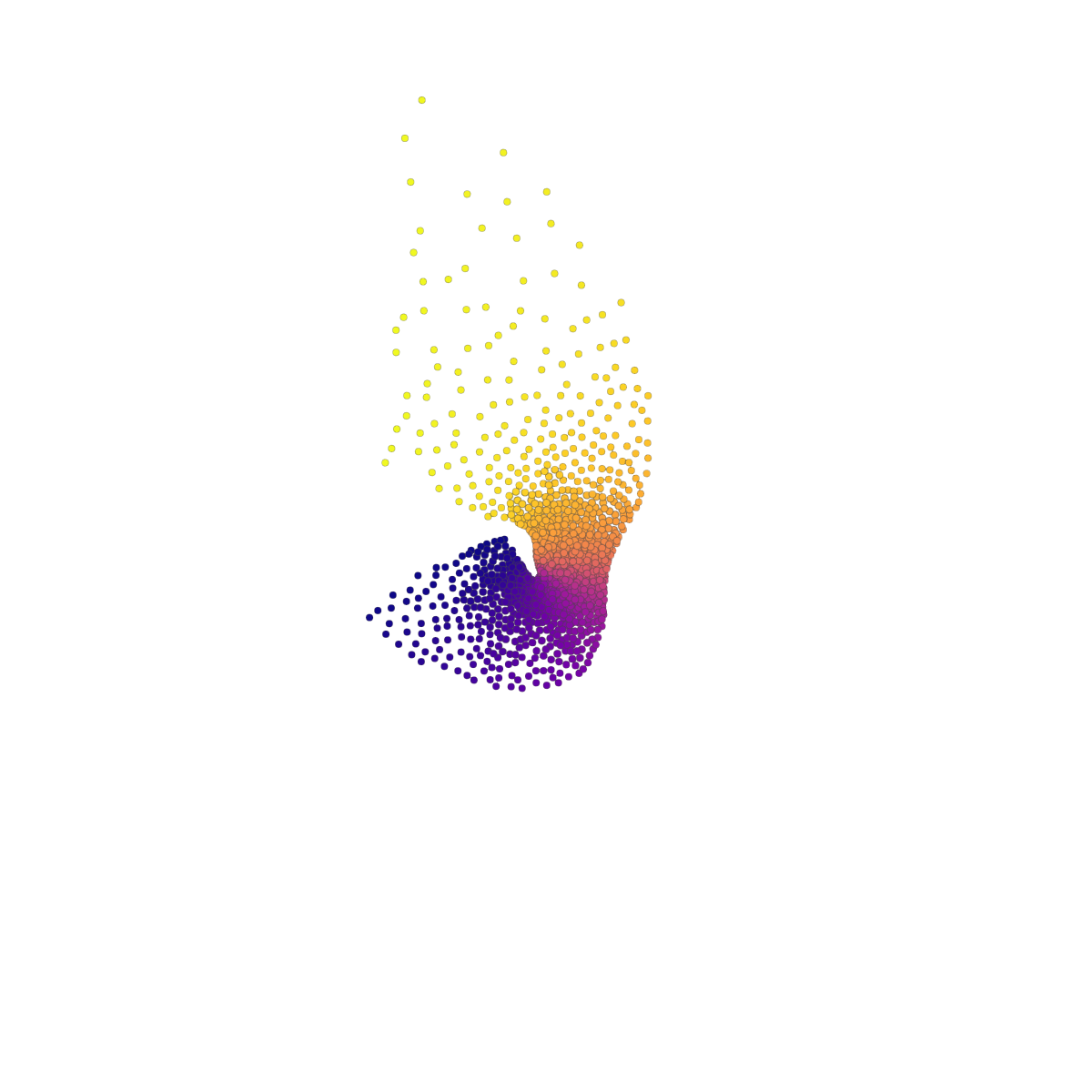}&
        \includegraphics[width=0.1\textwidth]{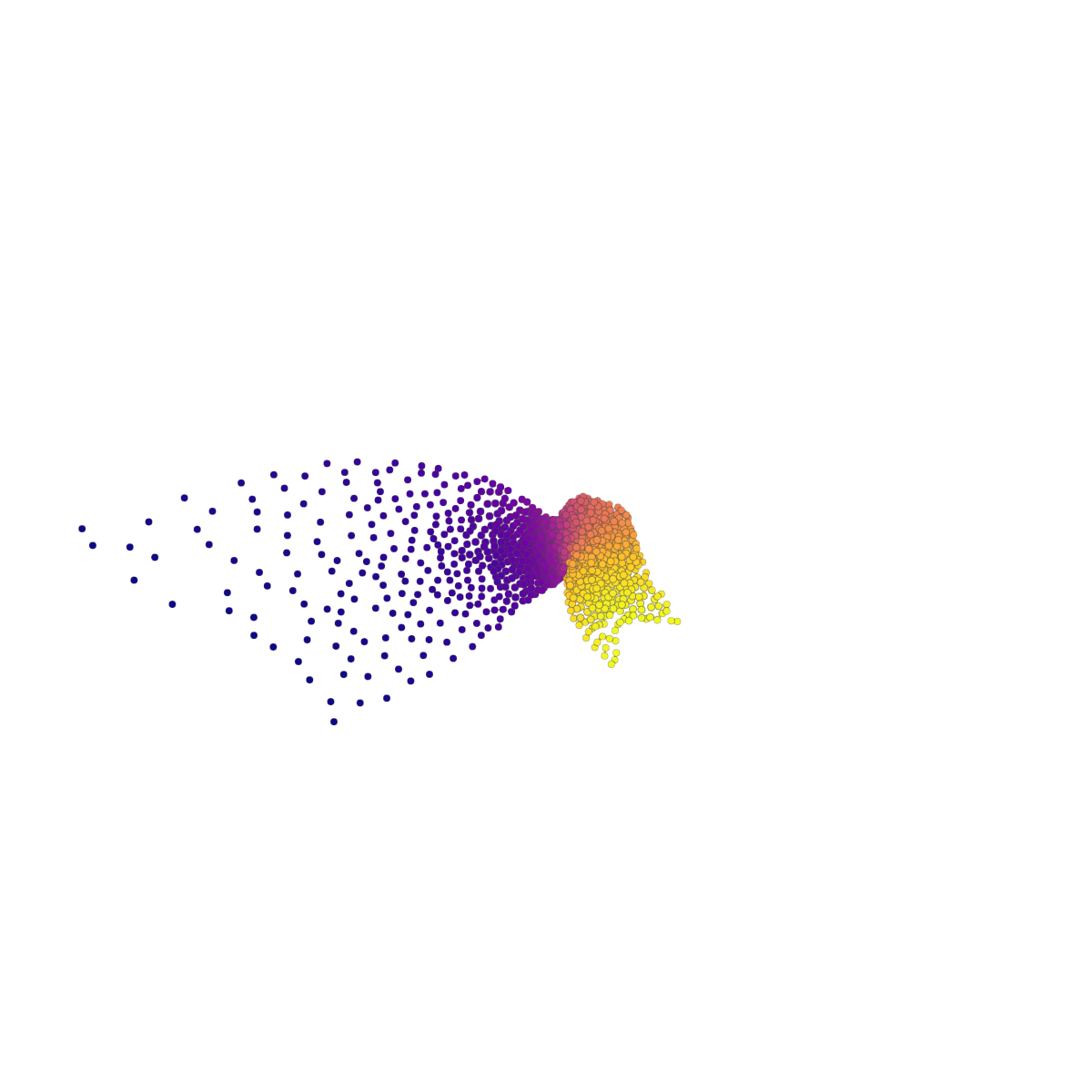}&
        \includegraphics[width=0.1\textwidth]{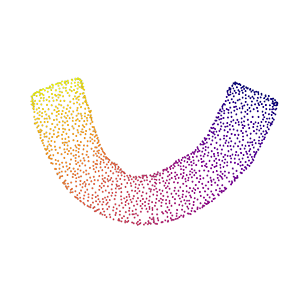}&
        \includegraphics[width=0.1\textwidth]{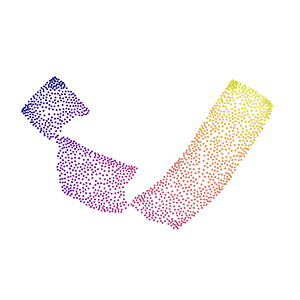}&
        \includegraphics[width=0.1\textwidth]{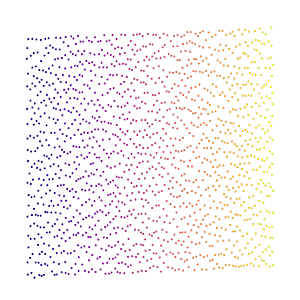}
        \\
        \includegraphics[width=0.1\textwidth]{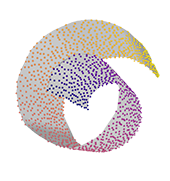}&
        \includegraphics[width=0.1\textwidth]{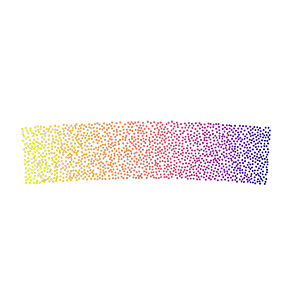}&
        \includegraphics[width=0.1\textwidth]{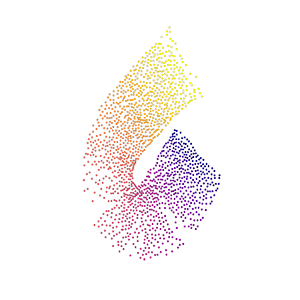}&
        \includegraphics[width=0.1\textwidth]{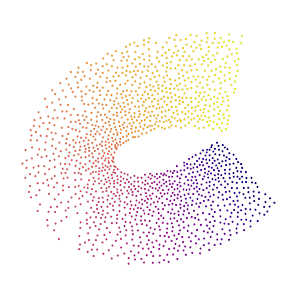}&
        \includegraphics[width=0.1\textwidth]{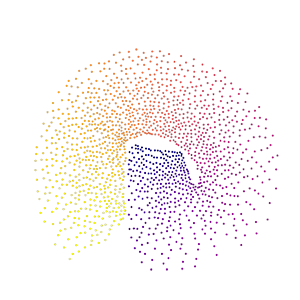}&
        \includegraphics[width=0.1\textwidth]{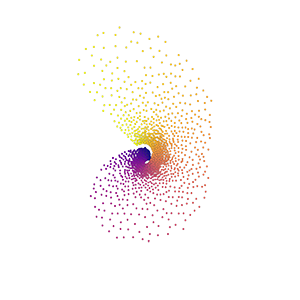}&
        \includegraphics[width=0.1\textwidth]{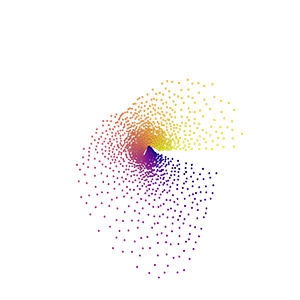}&
        \includegraphics[width=0.1\textwidth]{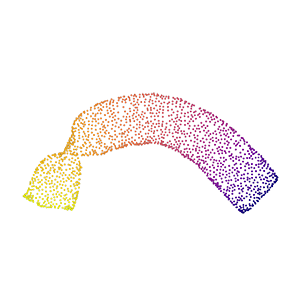}&
        \includegraphics[width=0.1\textwidth]{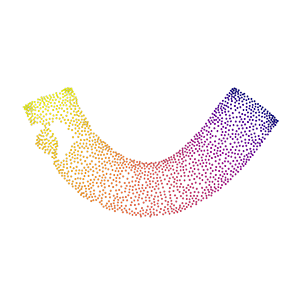}&
        \includegraphics[width=0.1\textwidth]{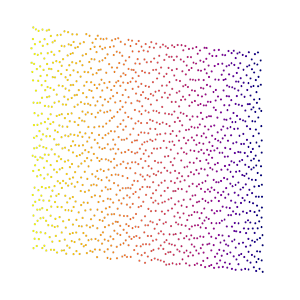}
        \\
        \includegraphics[width=0.1\textwidth]{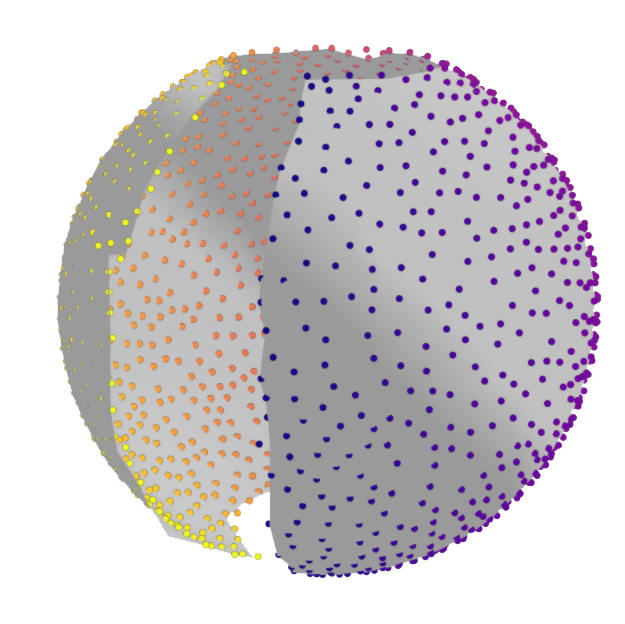}&
        \includegraphics[width=0.1\textwidth]{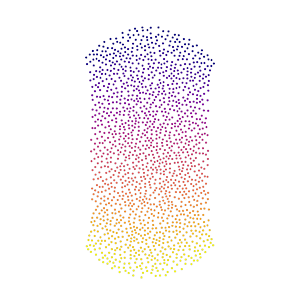}&
        \includegraphics[width=0.1\textwidth]{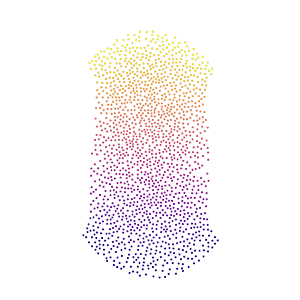}&
        \includegraphics[width=0.1\textwidth]{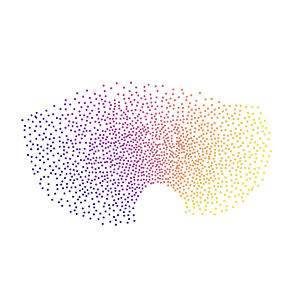}&
        \includegraphics[width=0.1\textwidth]{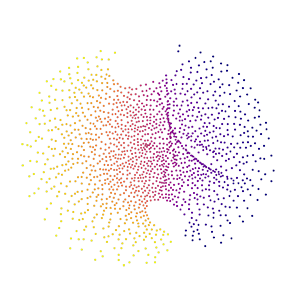}&
        \includegraphics[width=0.1\textwidth]{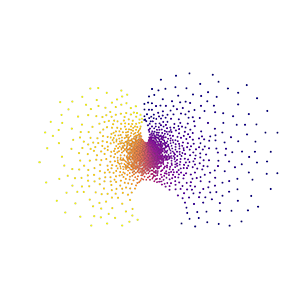}&
        \includegraphics[width=0.1\textwidth]{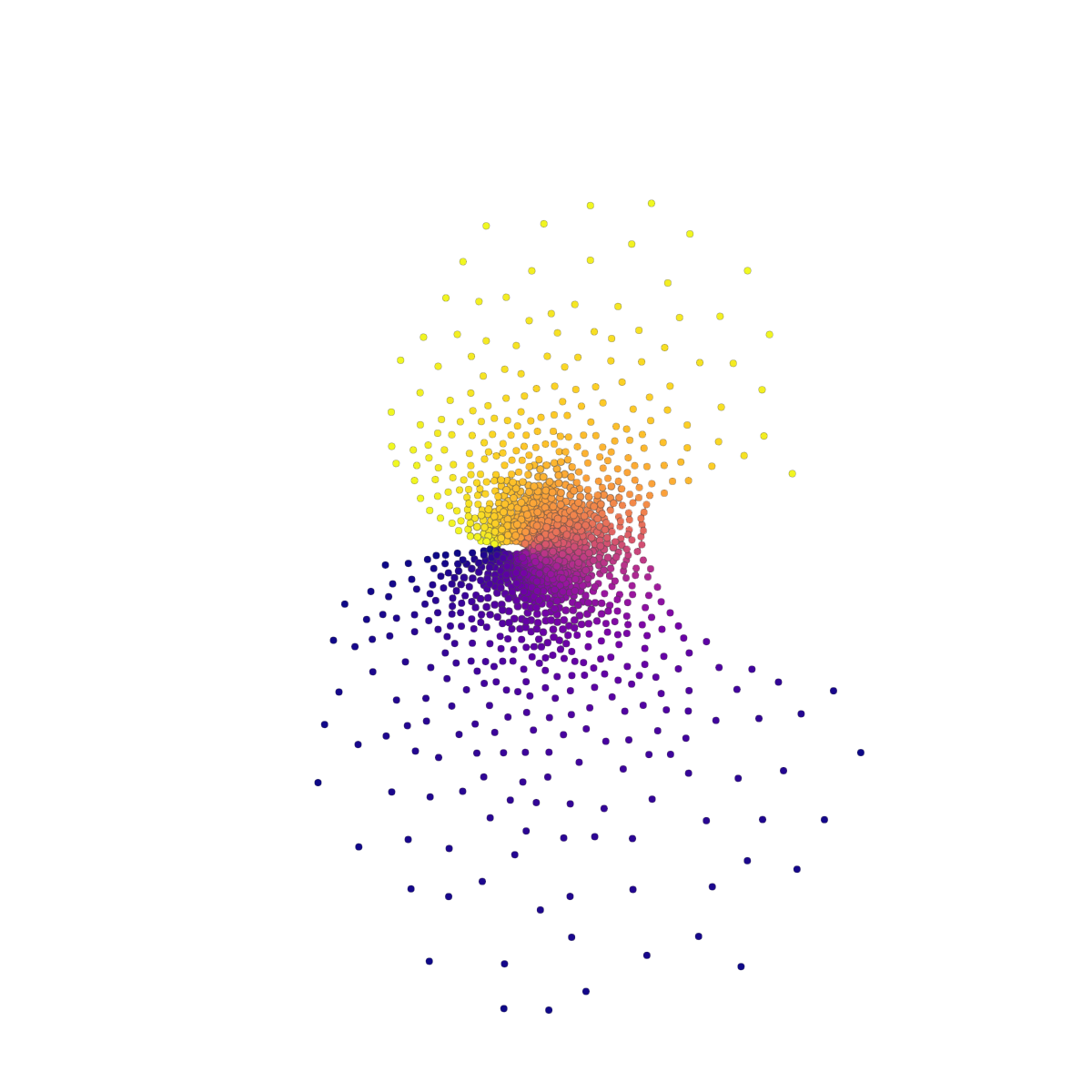}&
        \includegraphics[width=0.1\textwidth]{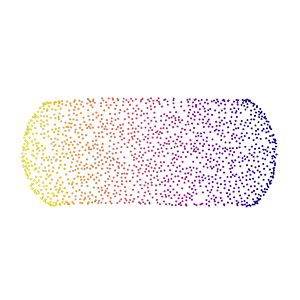}&
        \includegraphics[width=0.1\textwidth]{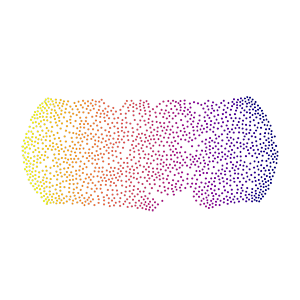}&
        \includegraphics[width=0.1\textwidth]{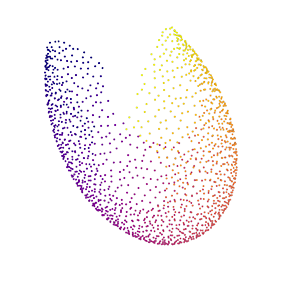}
    \end{tabular}
    
    \caption{Evaluation of $3D\too 2D$ embeddings. }
    \label{fig:evaluation}
\end{figure}
\subsection{Data visualization}\label{ss:visulaization}

\begin{figure}[th!]
    \centering
    \setlength\tabcolsep{0.3pt}
    \begin{tabular}{ccccc|cc}
        \textbf{I-AE} &
        CAE &
        RAE-GP&
        DAE&
        AE&
        U-MAP&
        t-SNE \\
        
        \hline
        \includegraphics[width=0.14\textwidth]{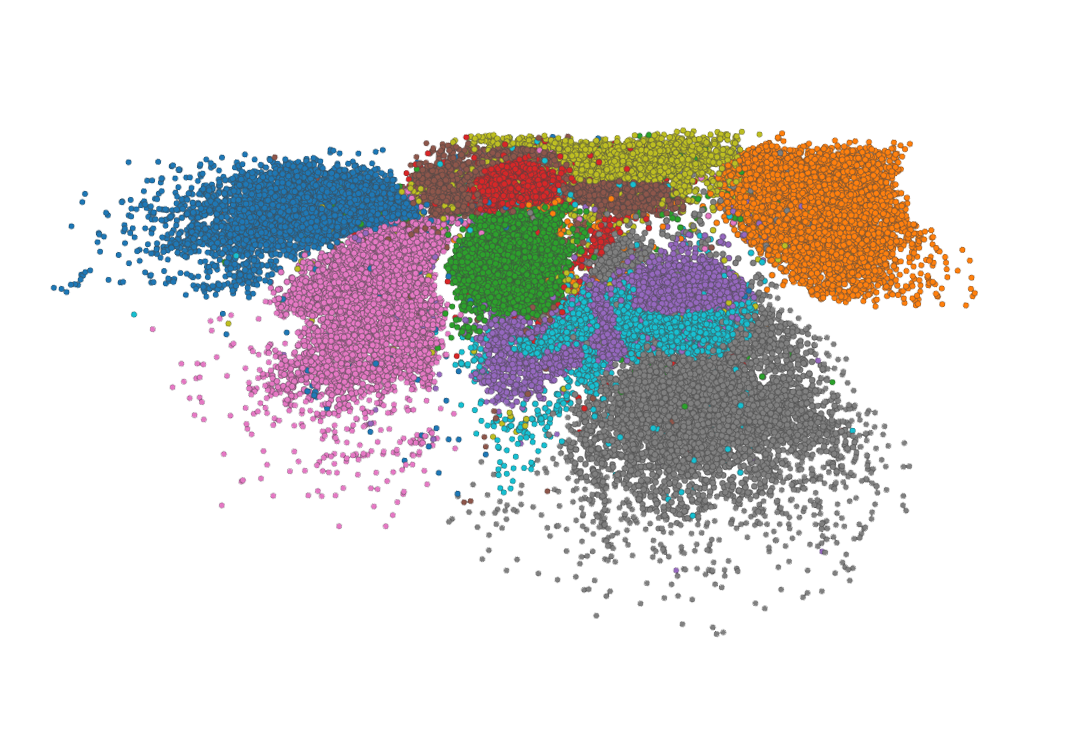}&
        \includegraphics[width=0.14\textwidth]{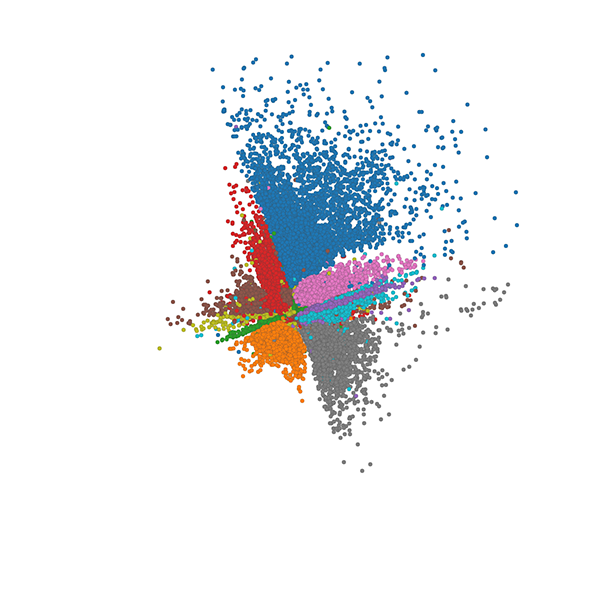}&
        \includegraphics[width=0.14\textwidth]{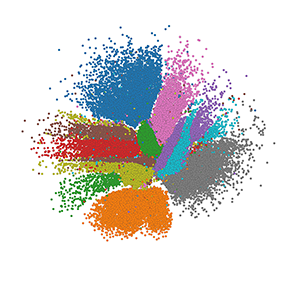}&
        \includegraphics[width=0.14\textwidth]{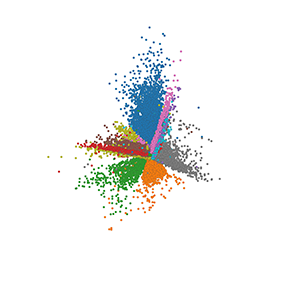}&
        \includegraphics[width=0.14\textwidth]{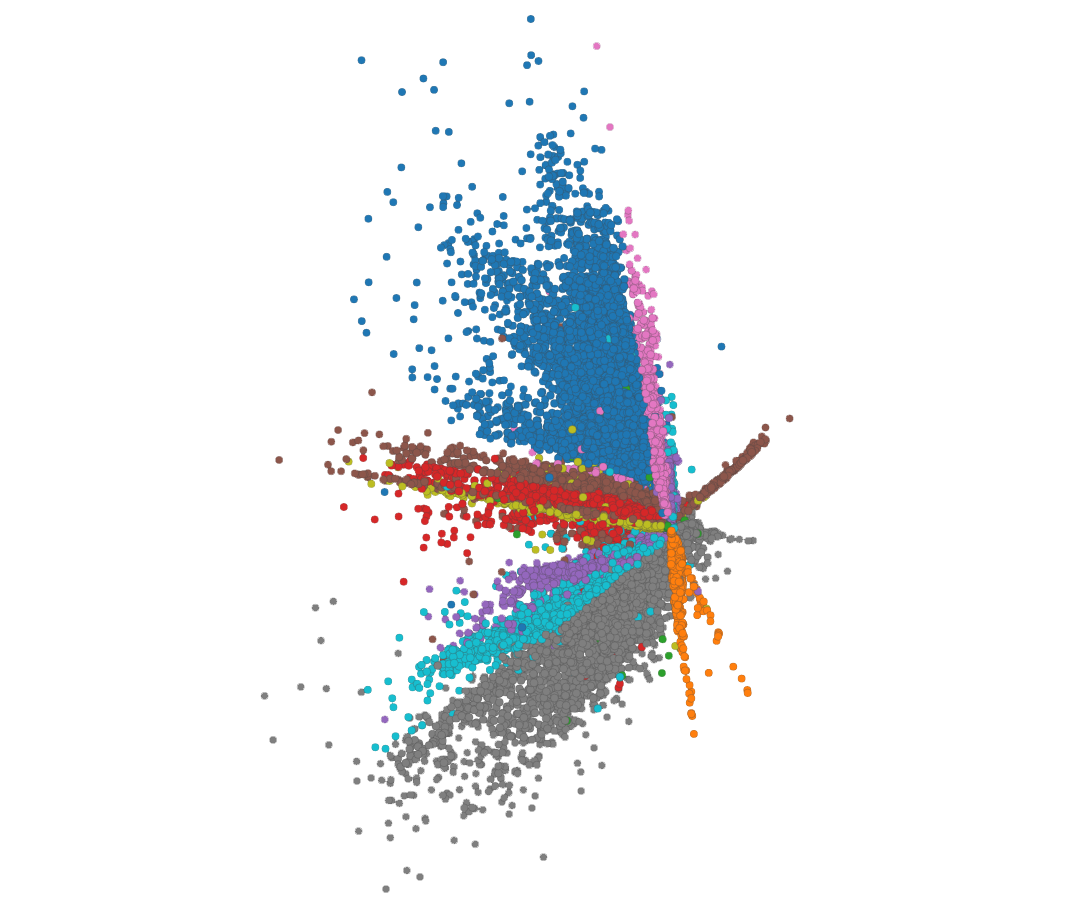}&
        \includegraphics[width=0.14\textwidth]{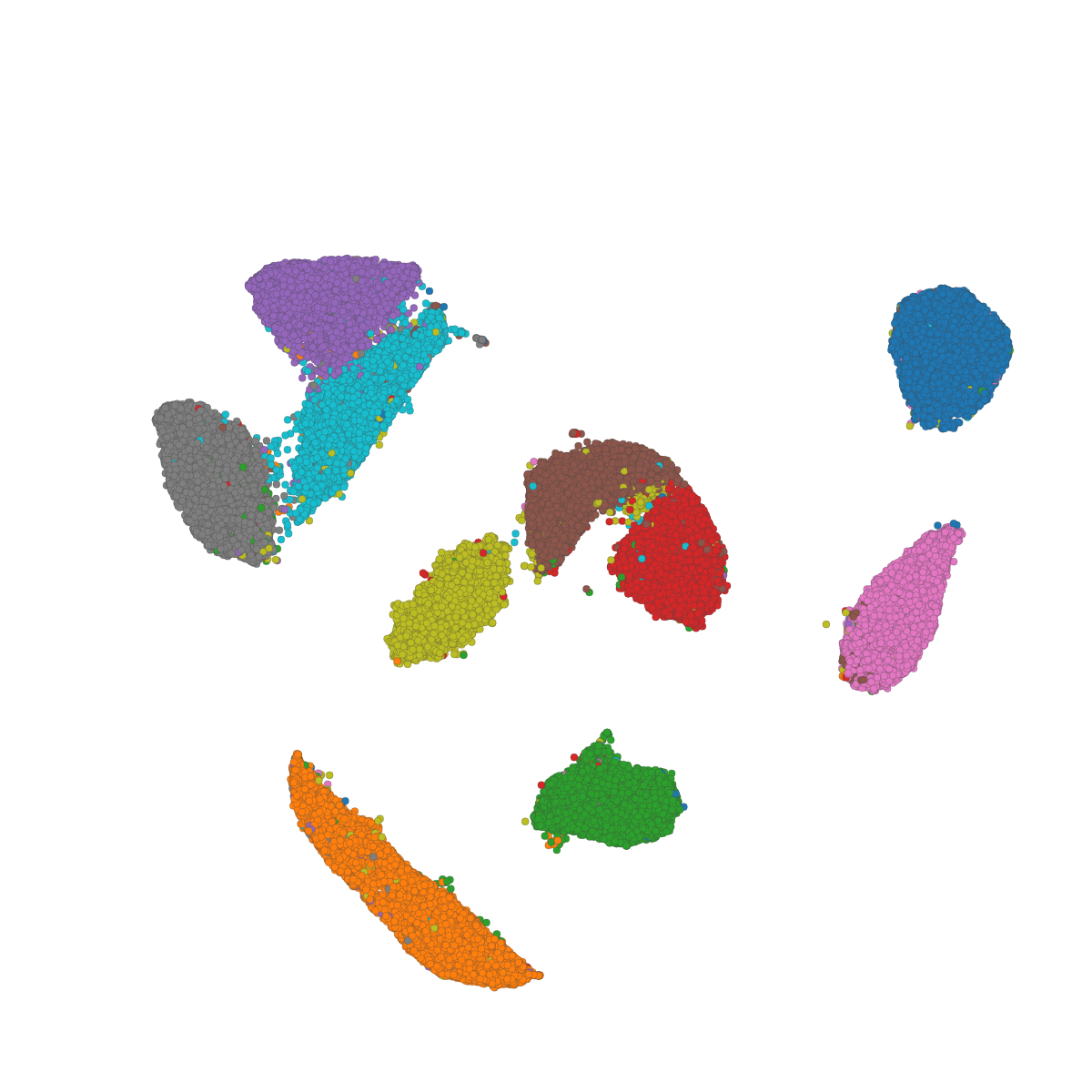}&
        \includegraphics[width=0.14\textwidth]{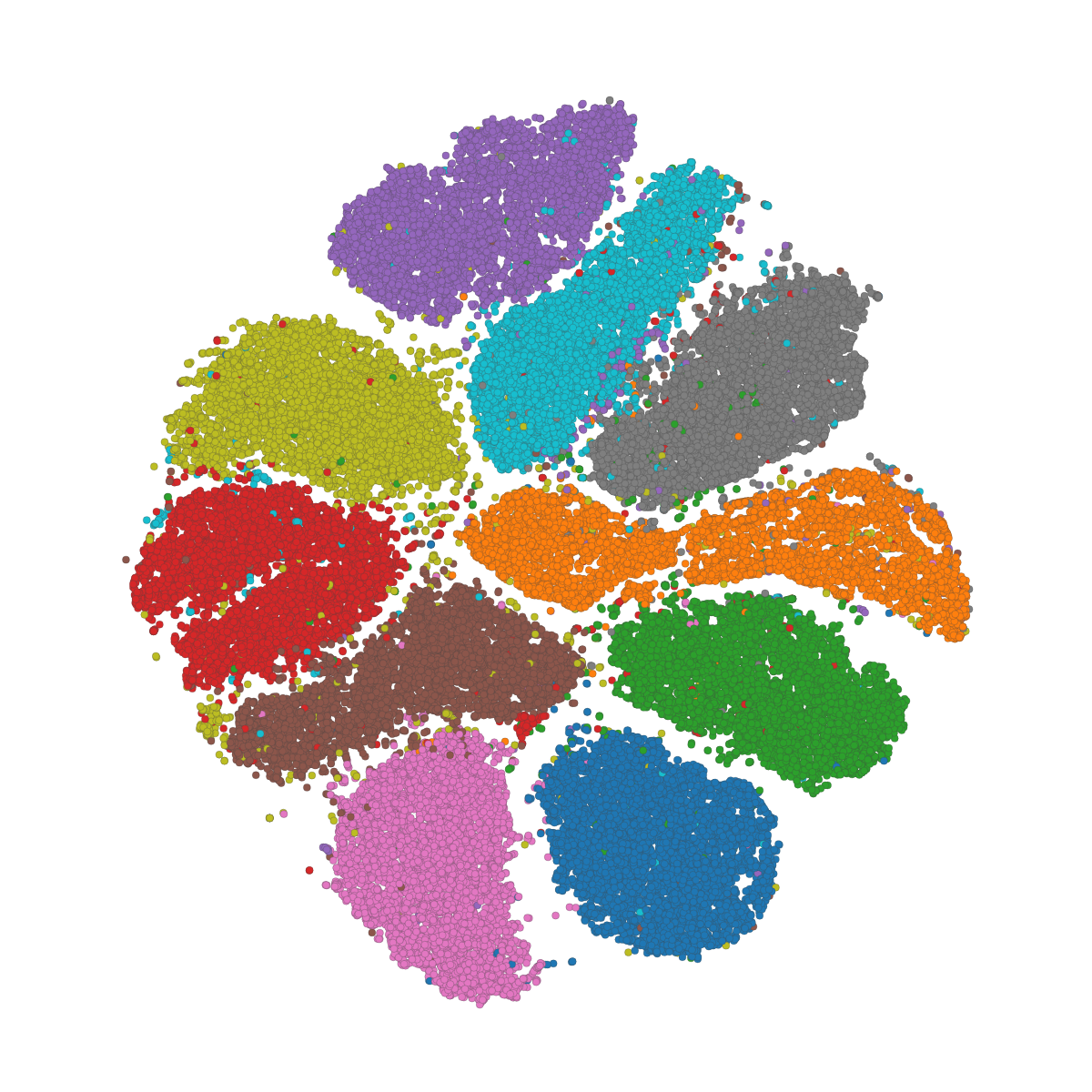}
        \\
        \includegraphics[width=0.14\textwidth]{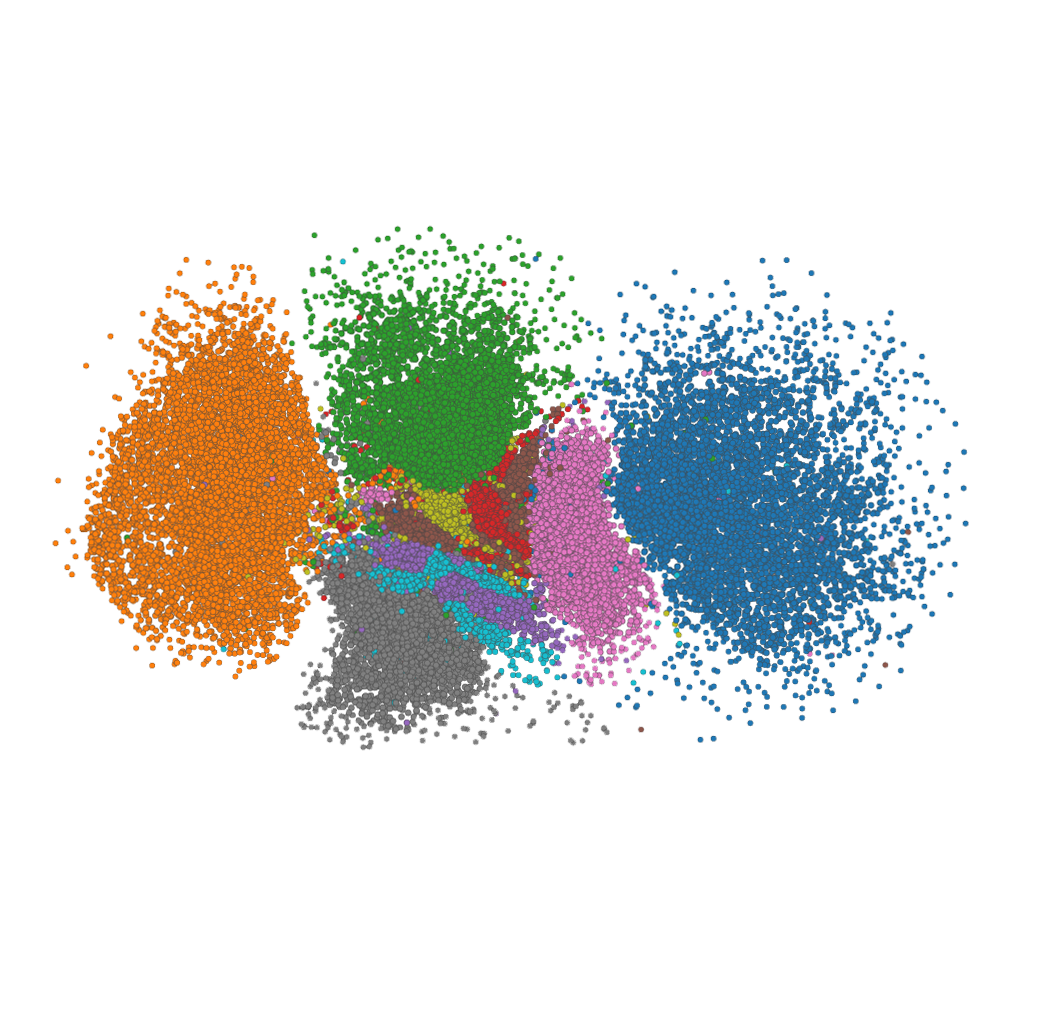}&
        \includegraphics[width=0.14\textwidth]{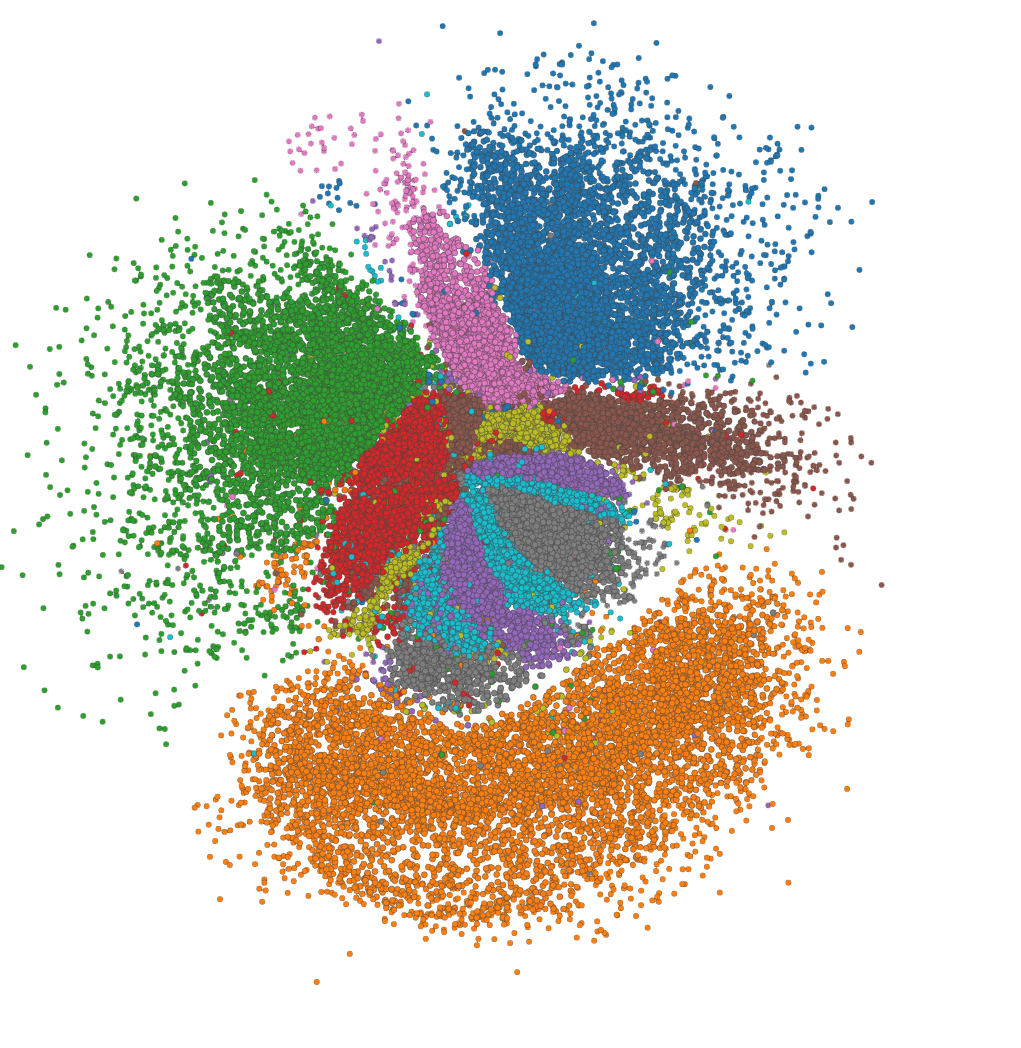}&
        \includegraphics[width=0.14\textwidth]{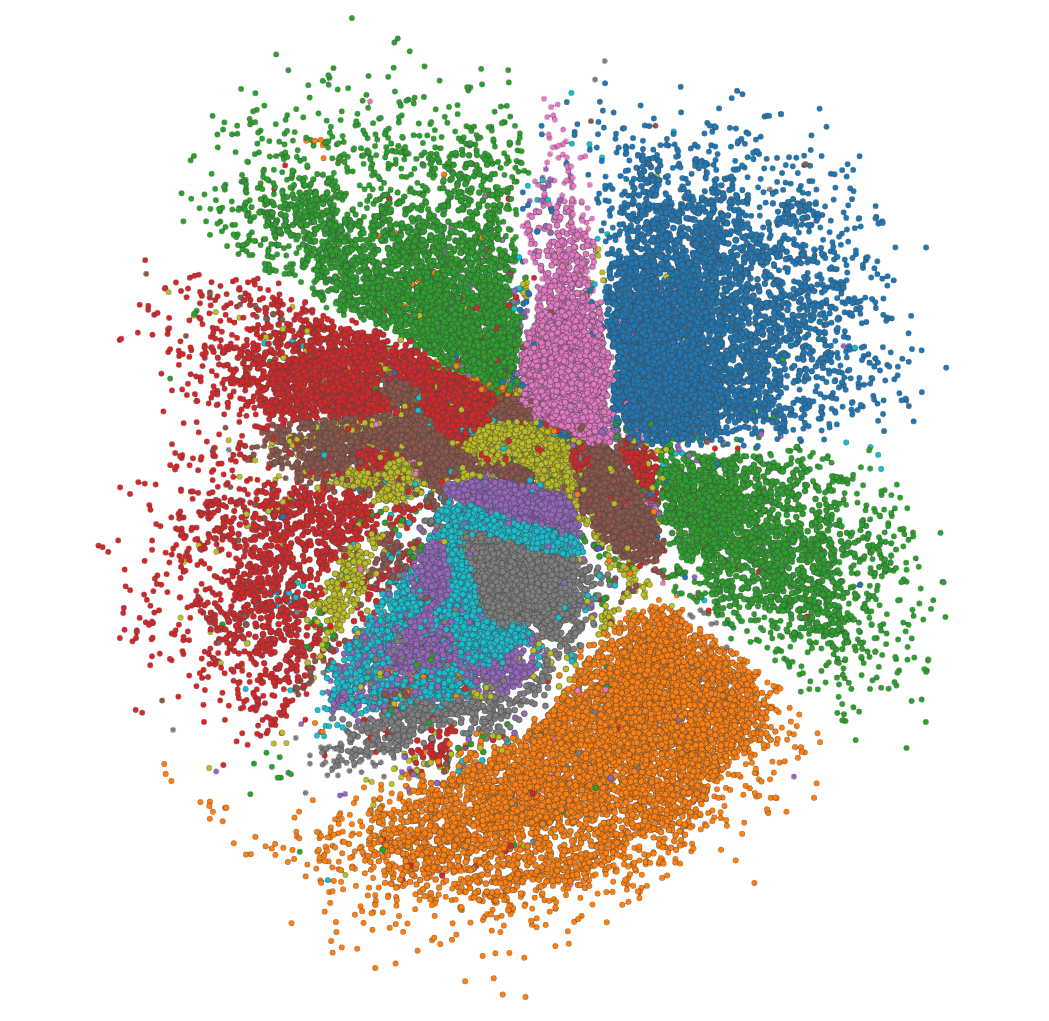}&
        \includegraphics[width=0.14\textwidth]{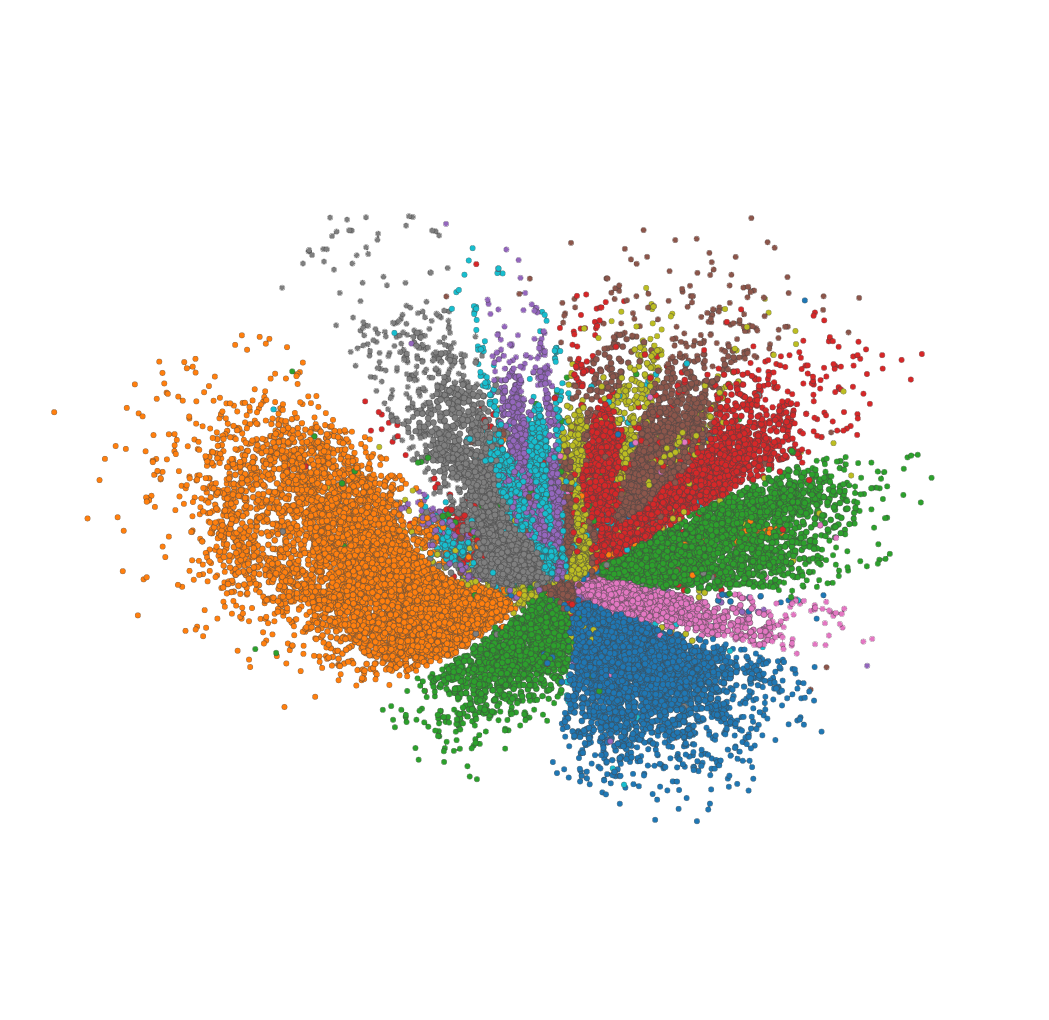}&
        \includegraphics[width=0.14\textwidth]{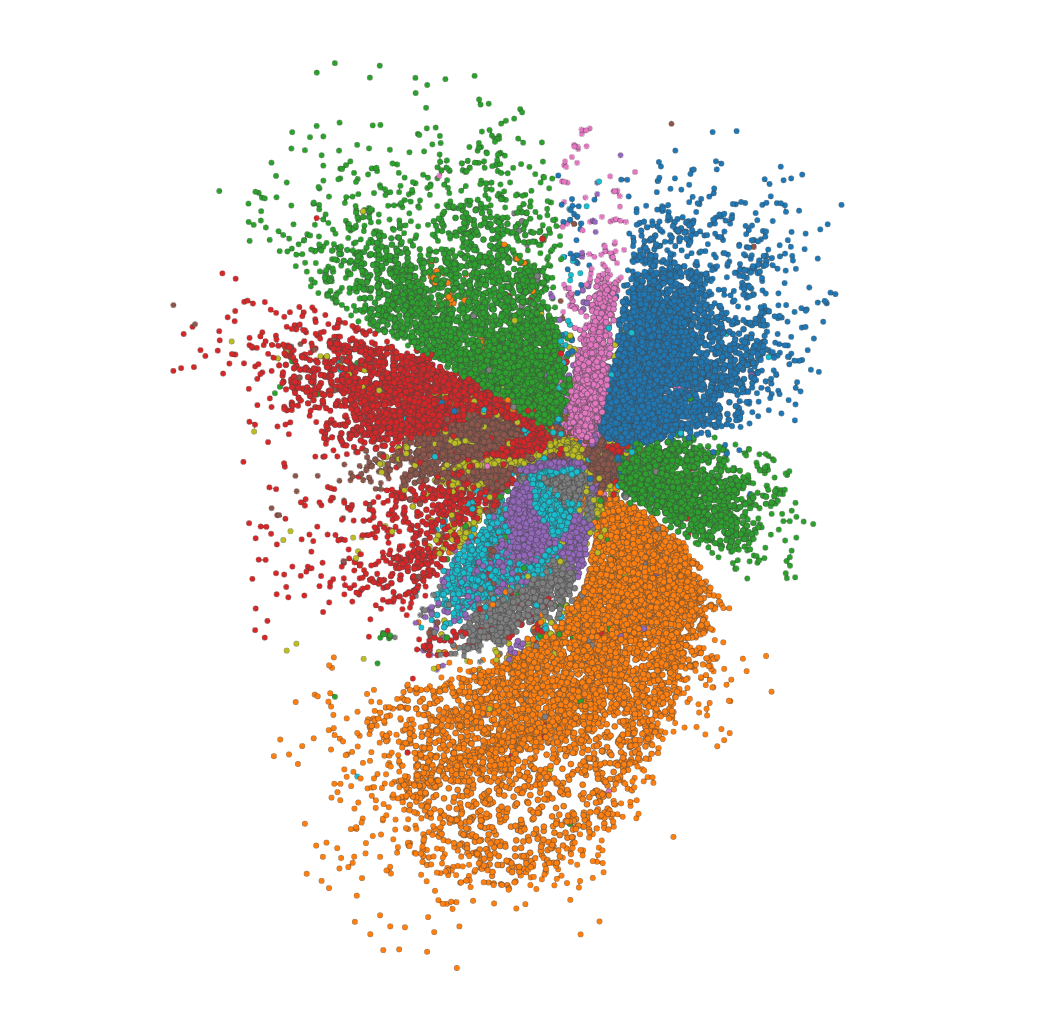} & & \\
        \hline
        
        \includegraphics[width=0.14\textwidth]{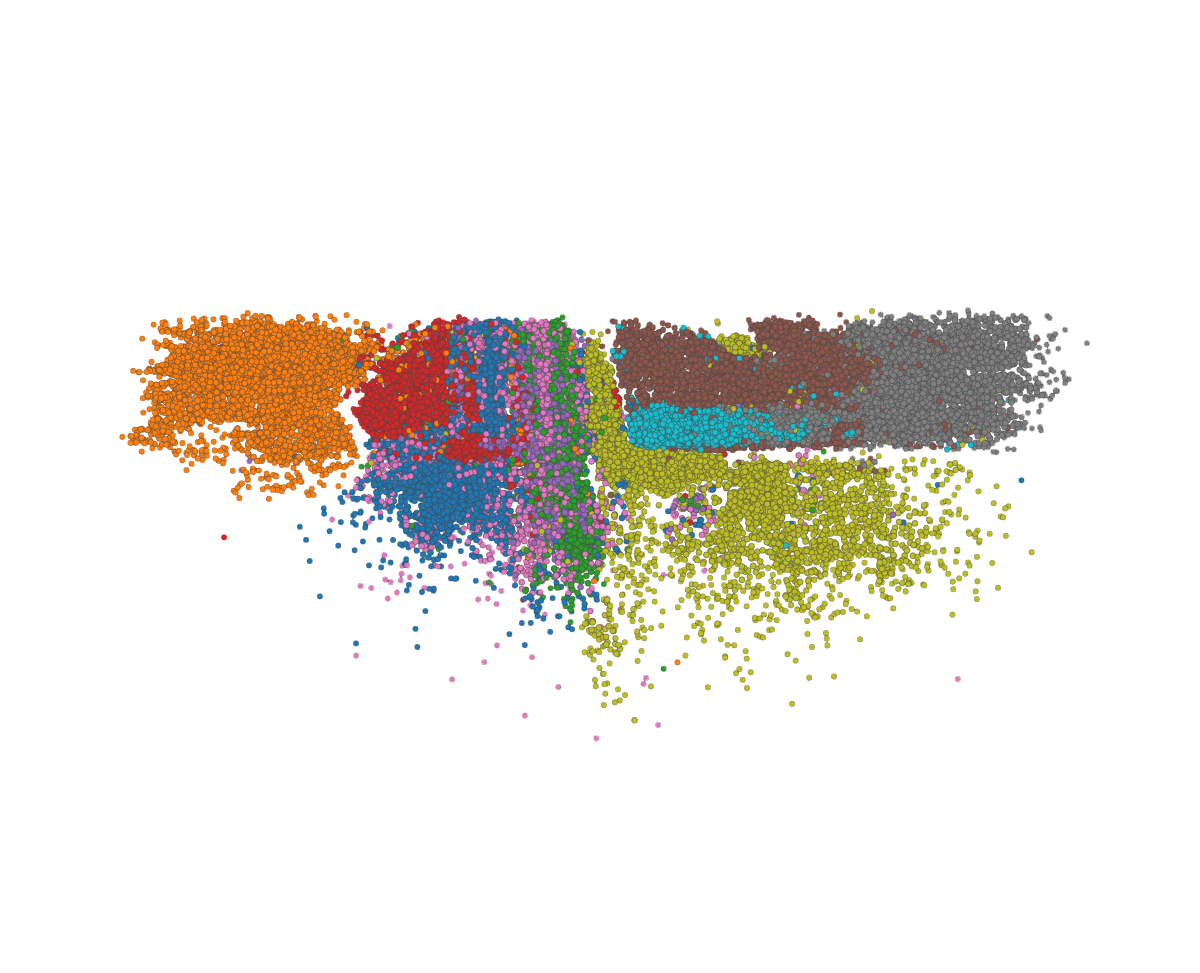}&
        \includegraphics[width=0.14\textwidth]{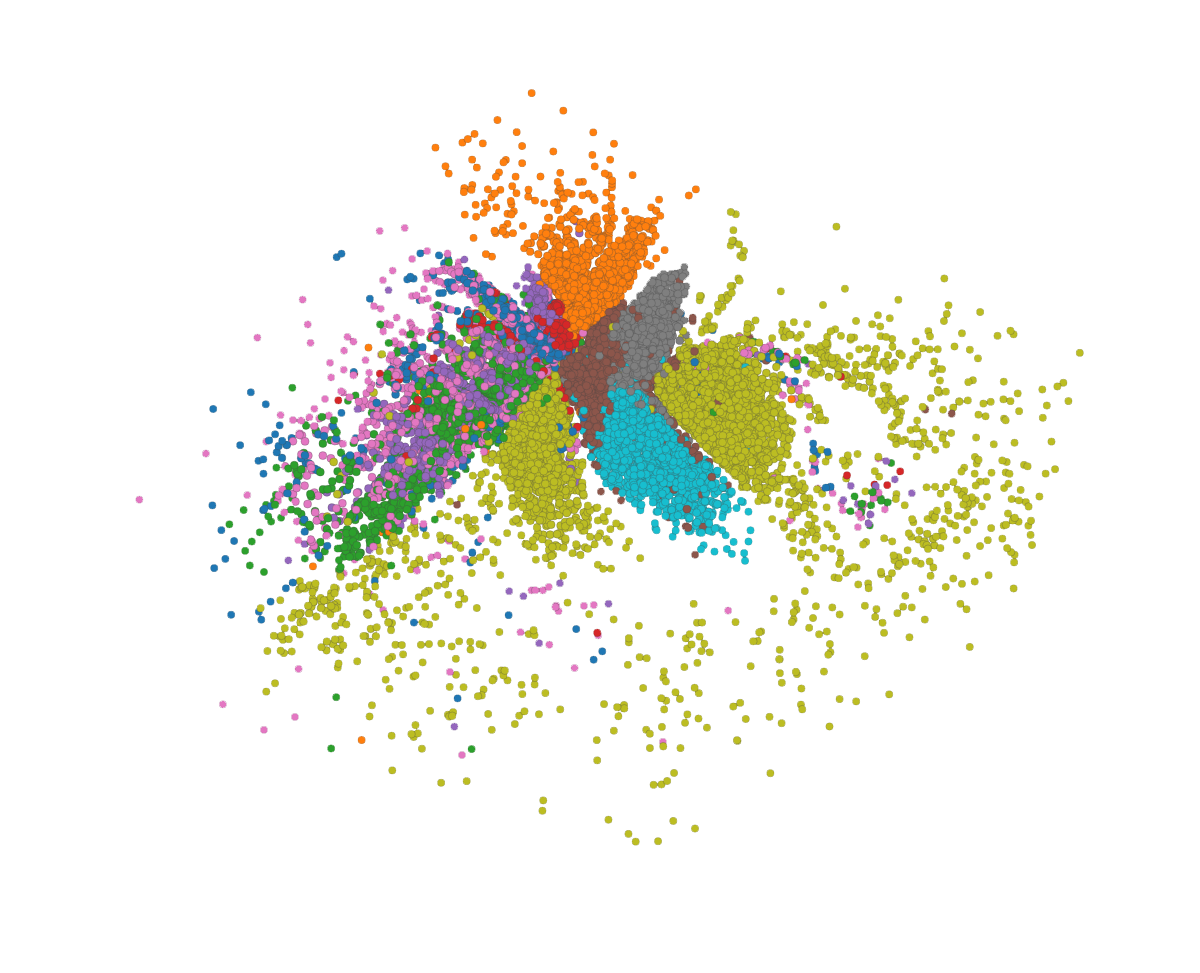}&
        \includegraphics[width=0.14\textwidth]{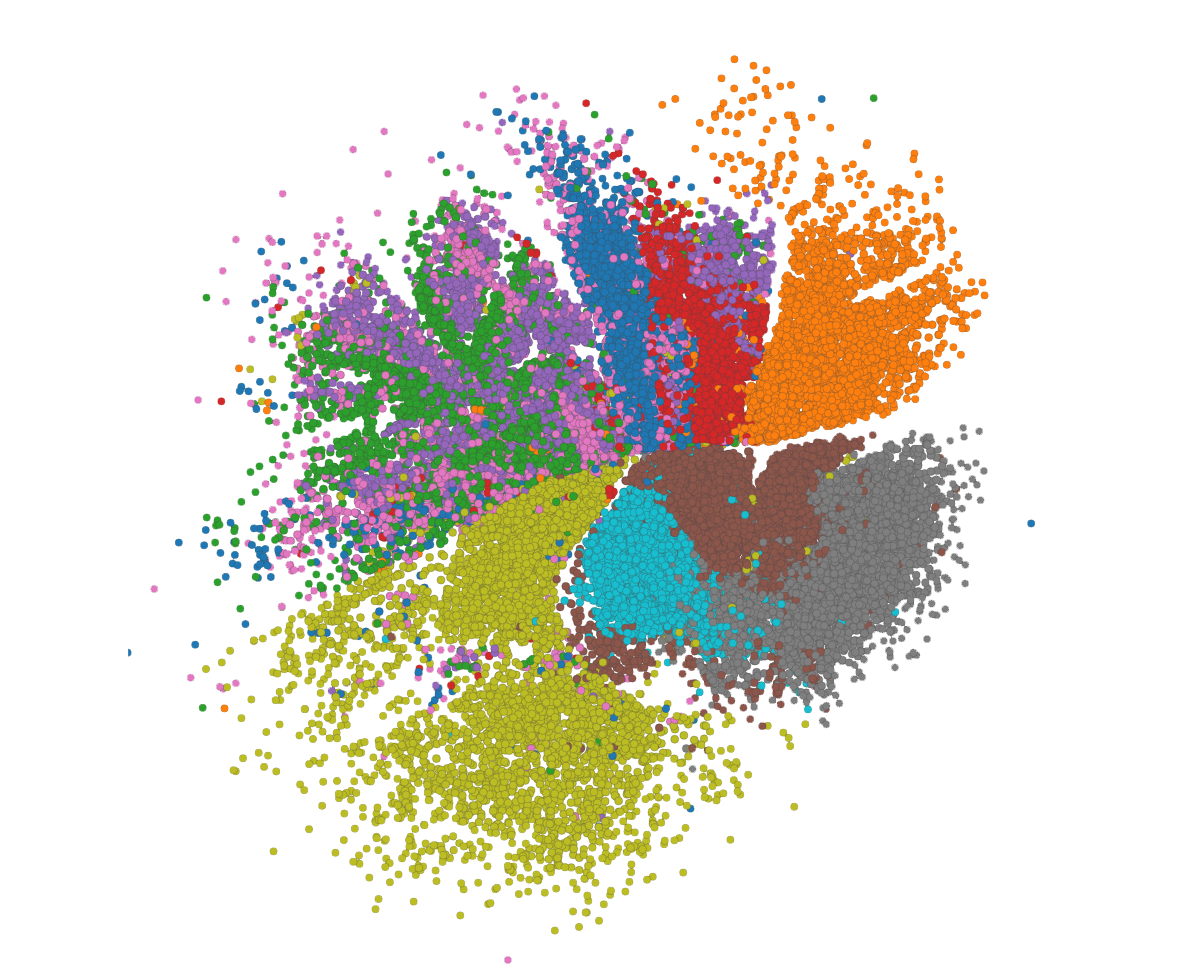}&
        \includegraphics[width=0.14\textwidth]{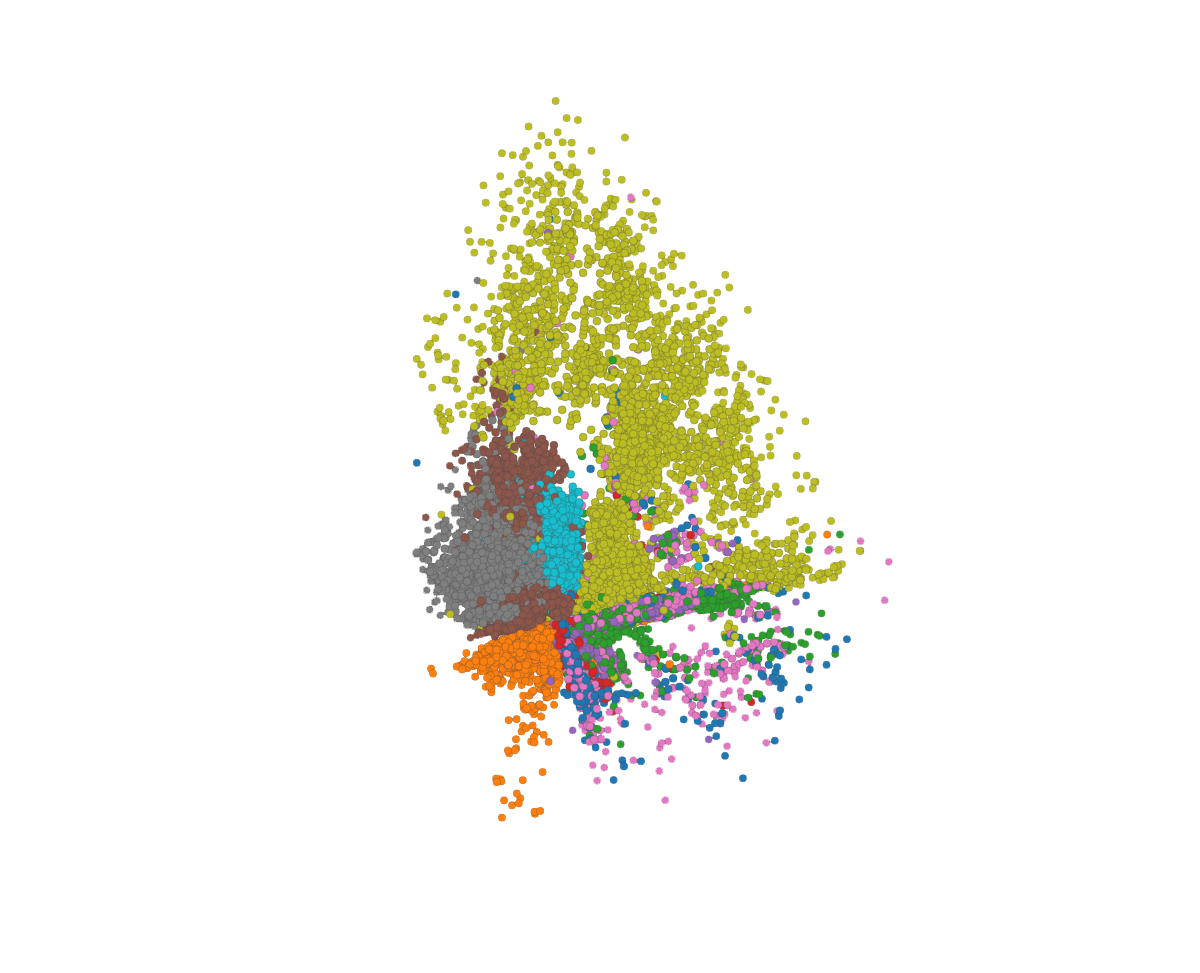}&
        \includegraphics[width=0.14\textwidth]{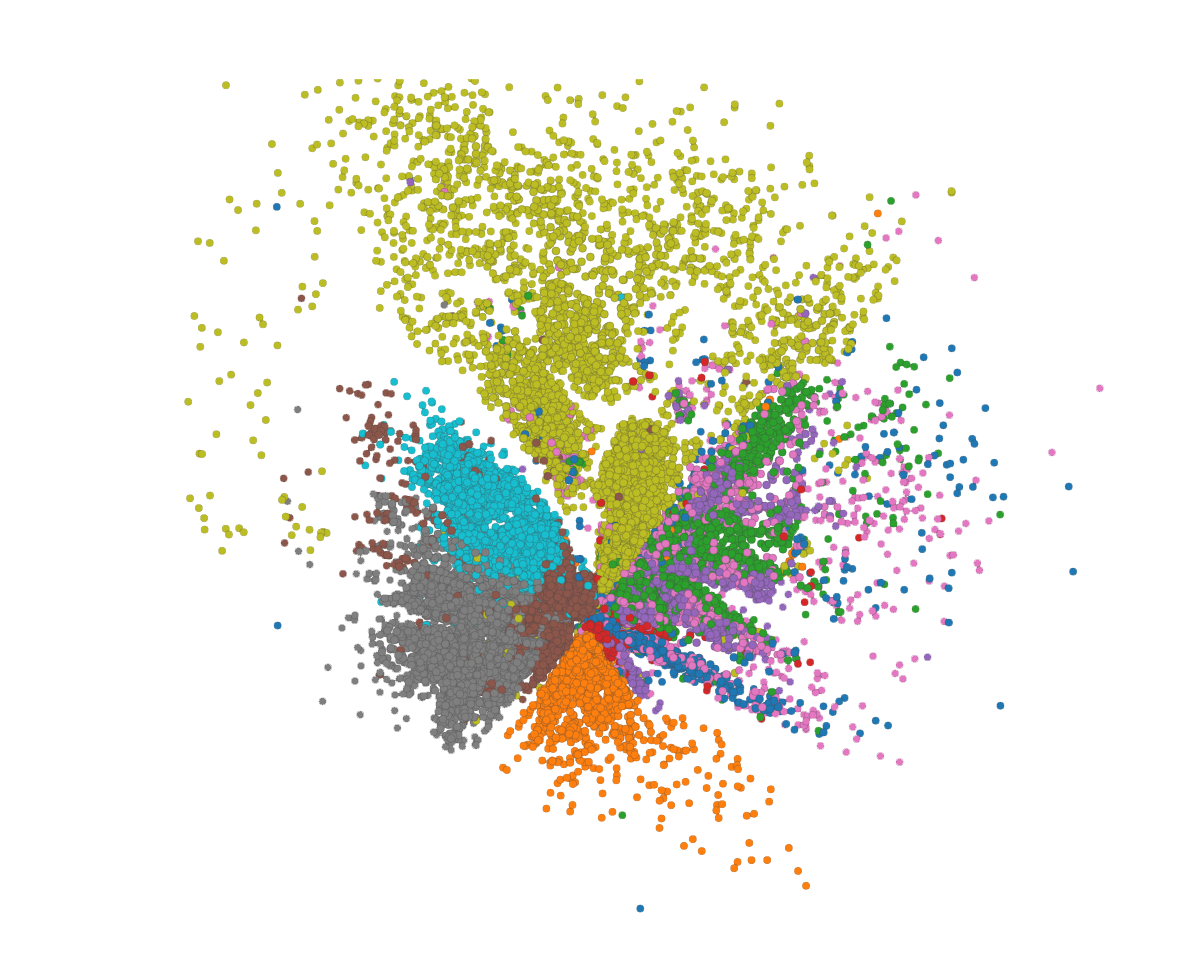}&
        \includegraphics[width=0.14\textwidth]{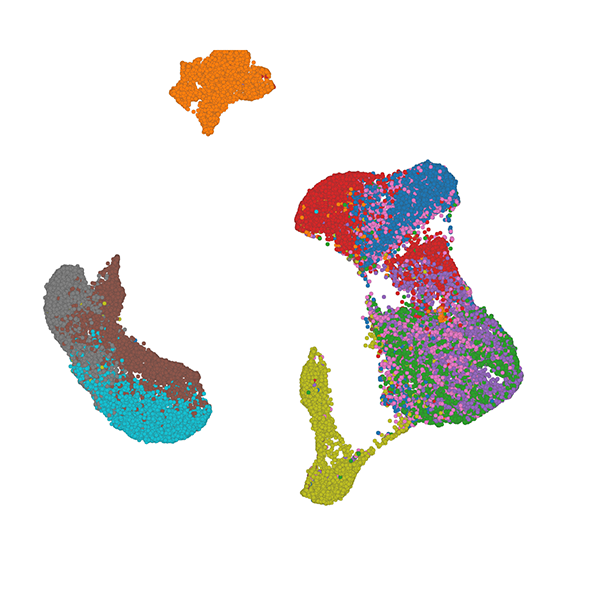}&
        \includegraphics[width=0.14\textwidth]{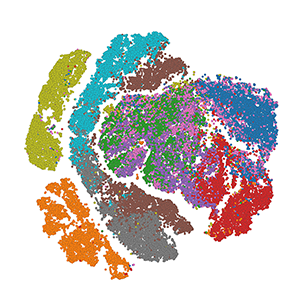}
        \\
        
        \includegraphics[width=0.14\textwidth]{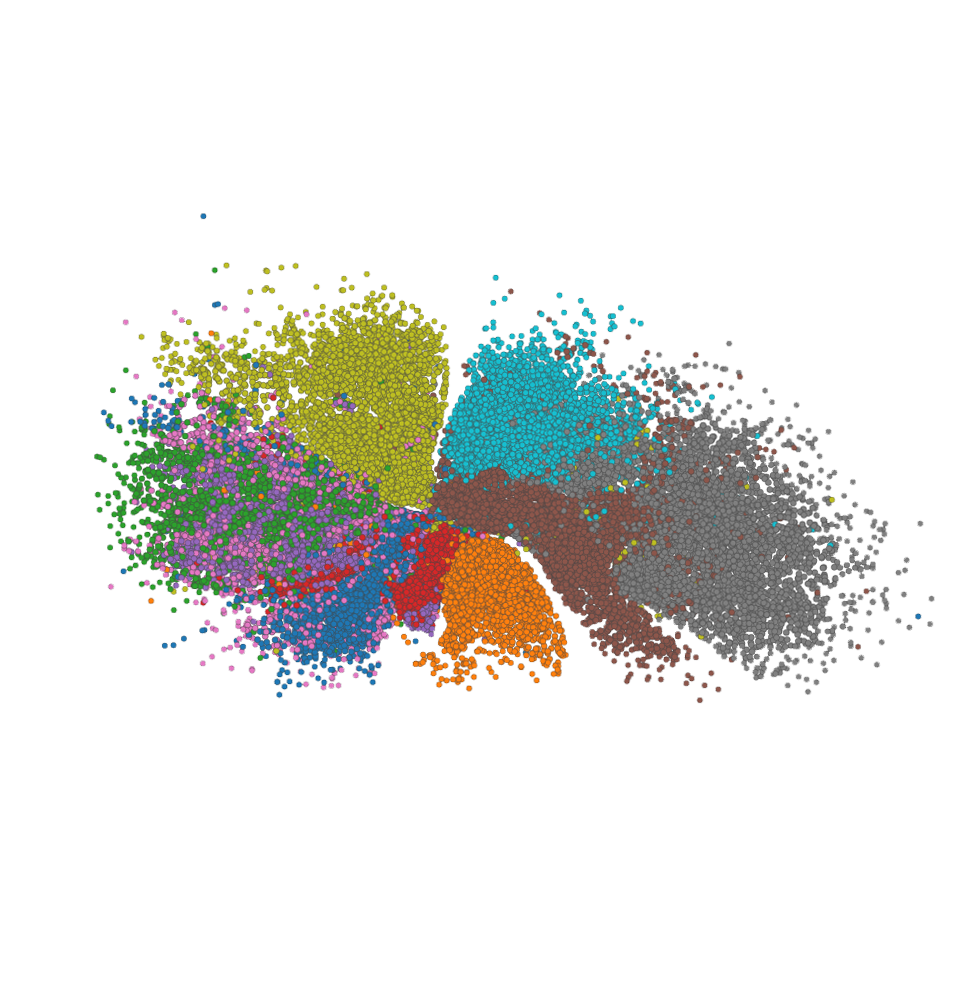}&
        \includegraphics[width=0.14\textwidth]{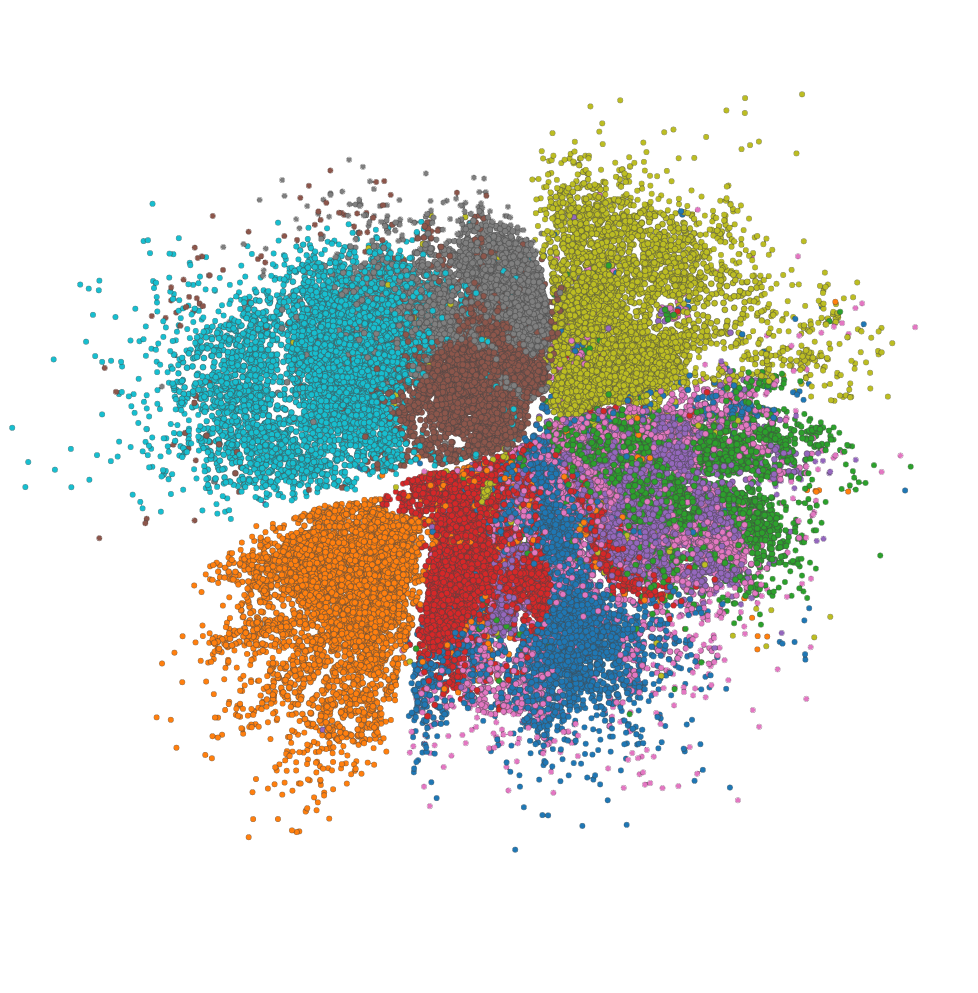}&
        \includegraphics[width=0.14\textwidth]{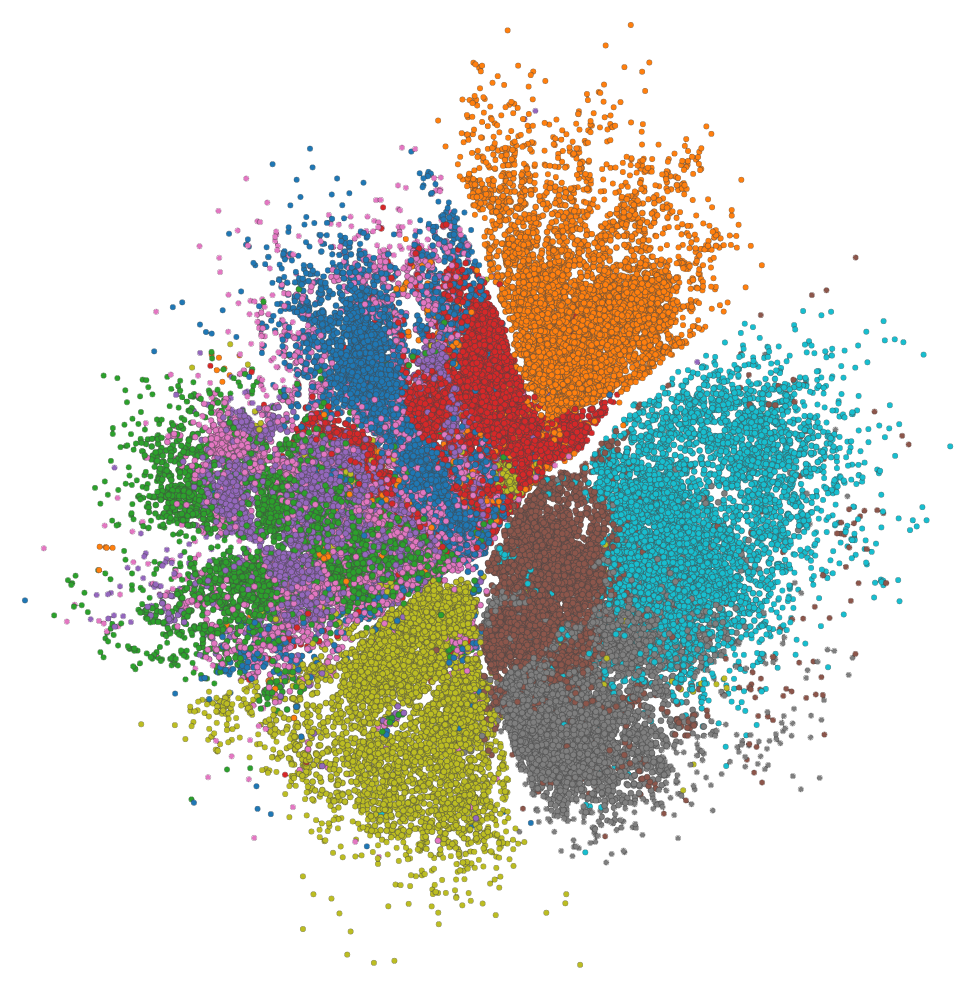}&
        \includegraphics[width=0.14\textwidth]{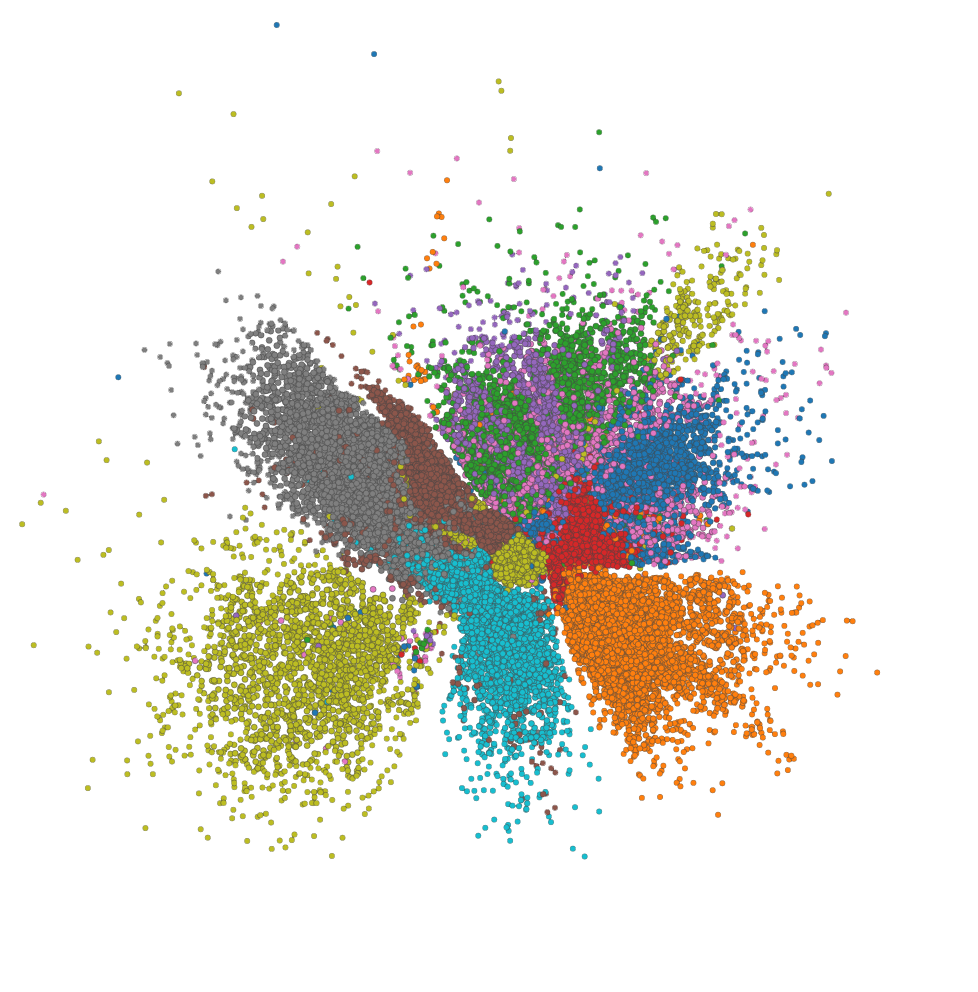}&
        \includegraphics[width=0.14\textwidth]{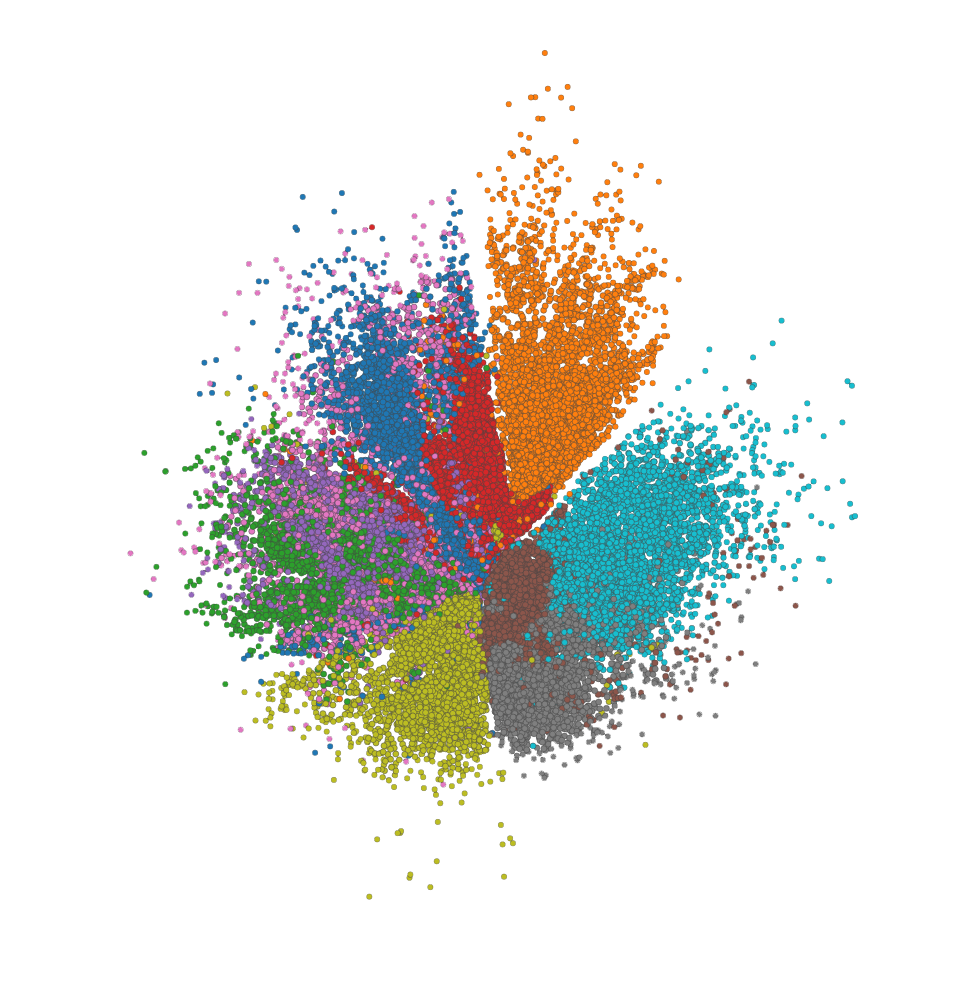}& &
        \\
        \hline
        \includegraphics[width=0.14\textwidth]{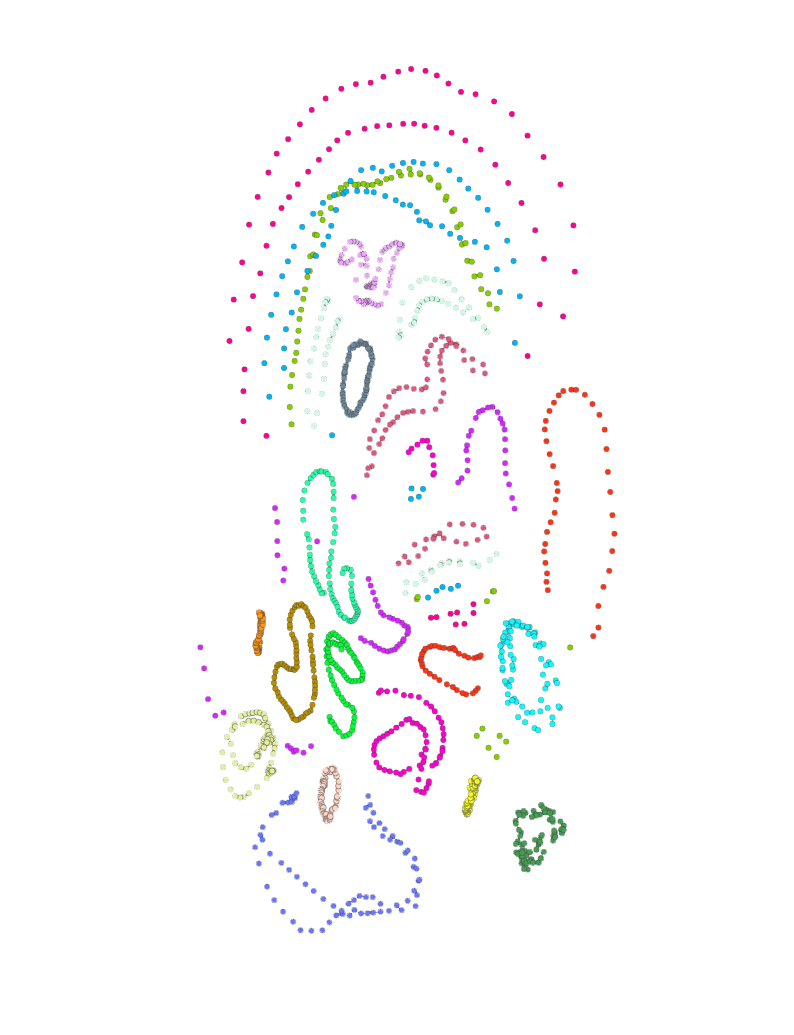}&
        \includegraphics[width=0.14\textwidth]{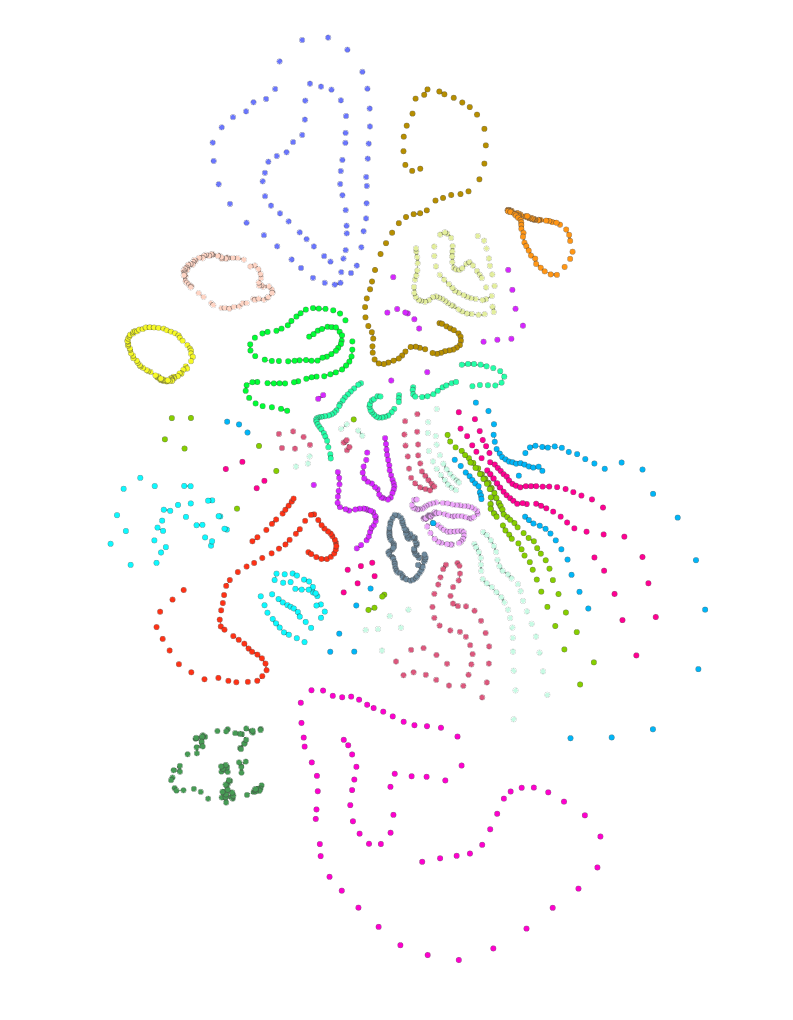}&
        \includegraphics[width=0.14\textwidth]{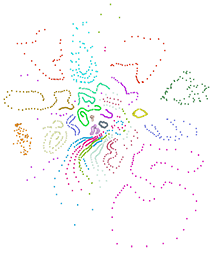}&
        \includegraphics[width=0.14\textwidth]{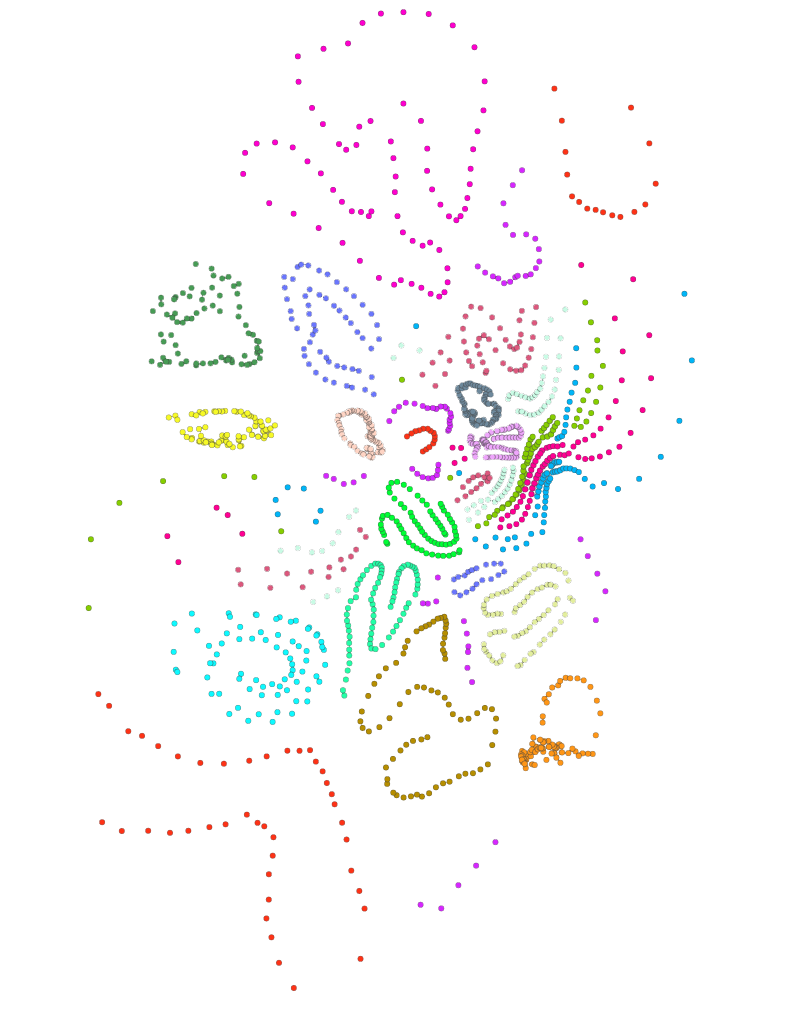}&
        \includegraphics[width=0.14\textwidth]{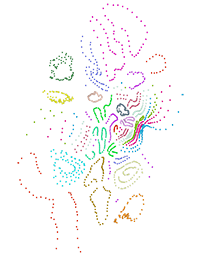}&
        \includegraphics[width=0.14\textwidth]{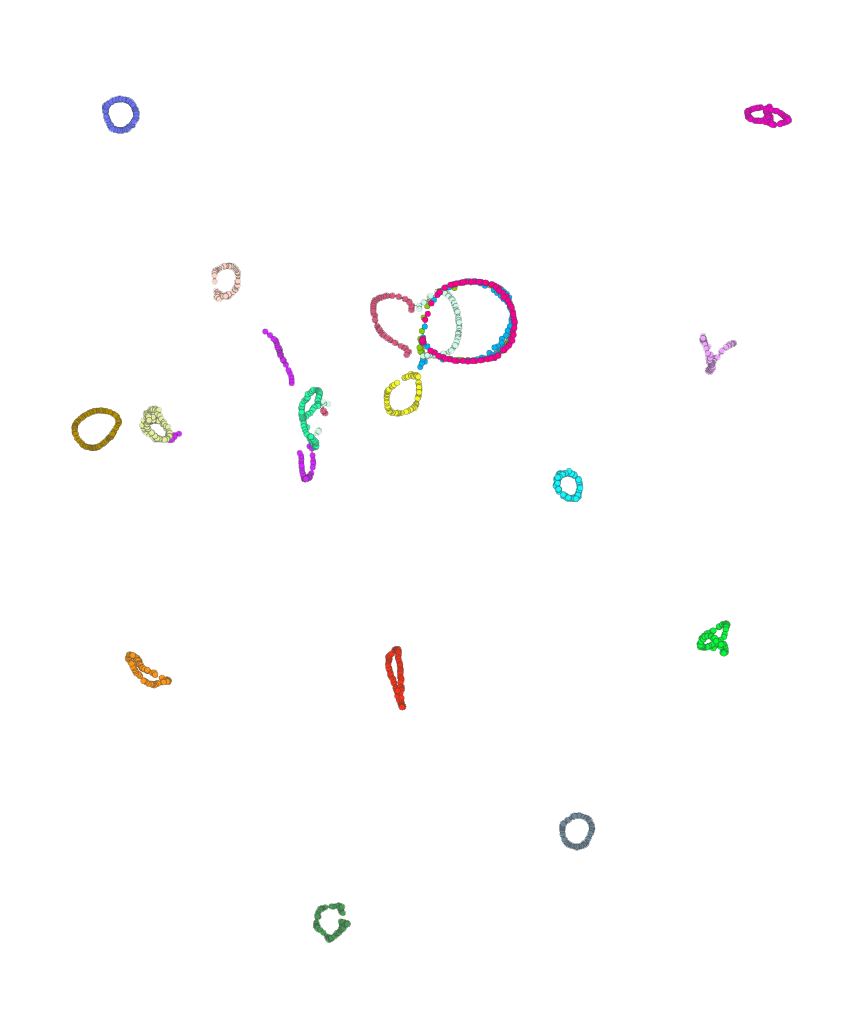}&
        \includegraphics[width=0.14\textwidth]{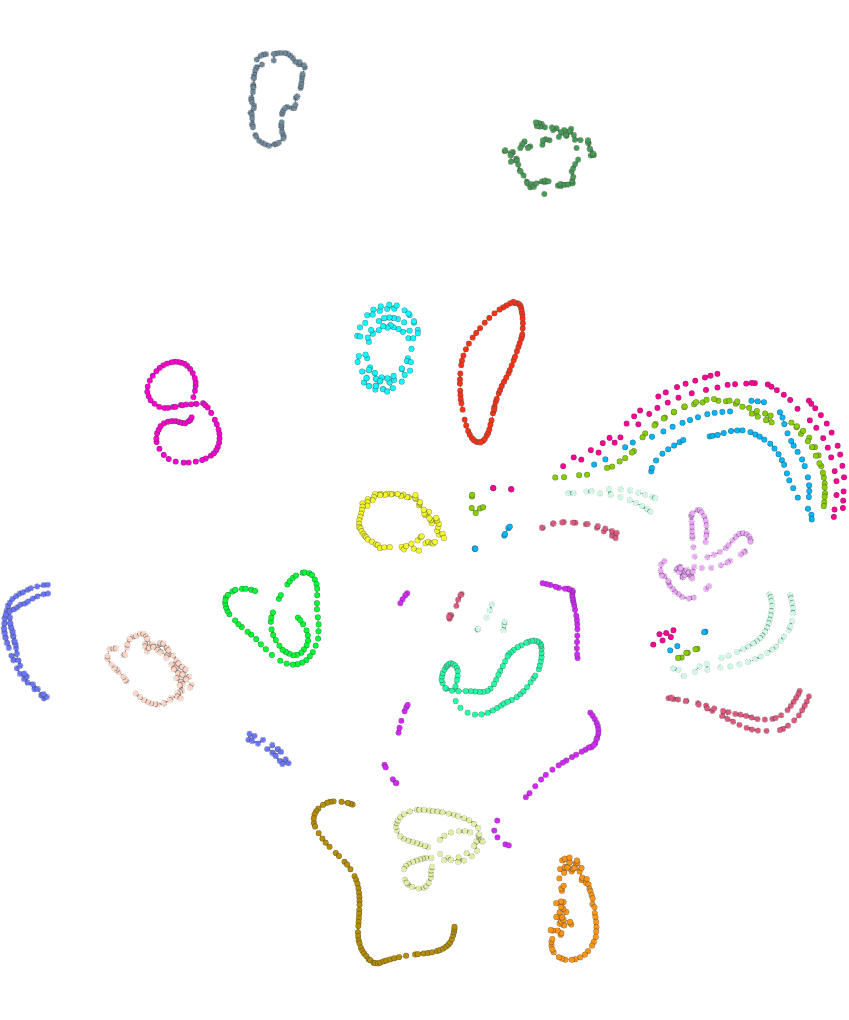}
        \\
        \includegraphics[width=0.14\textwidth]{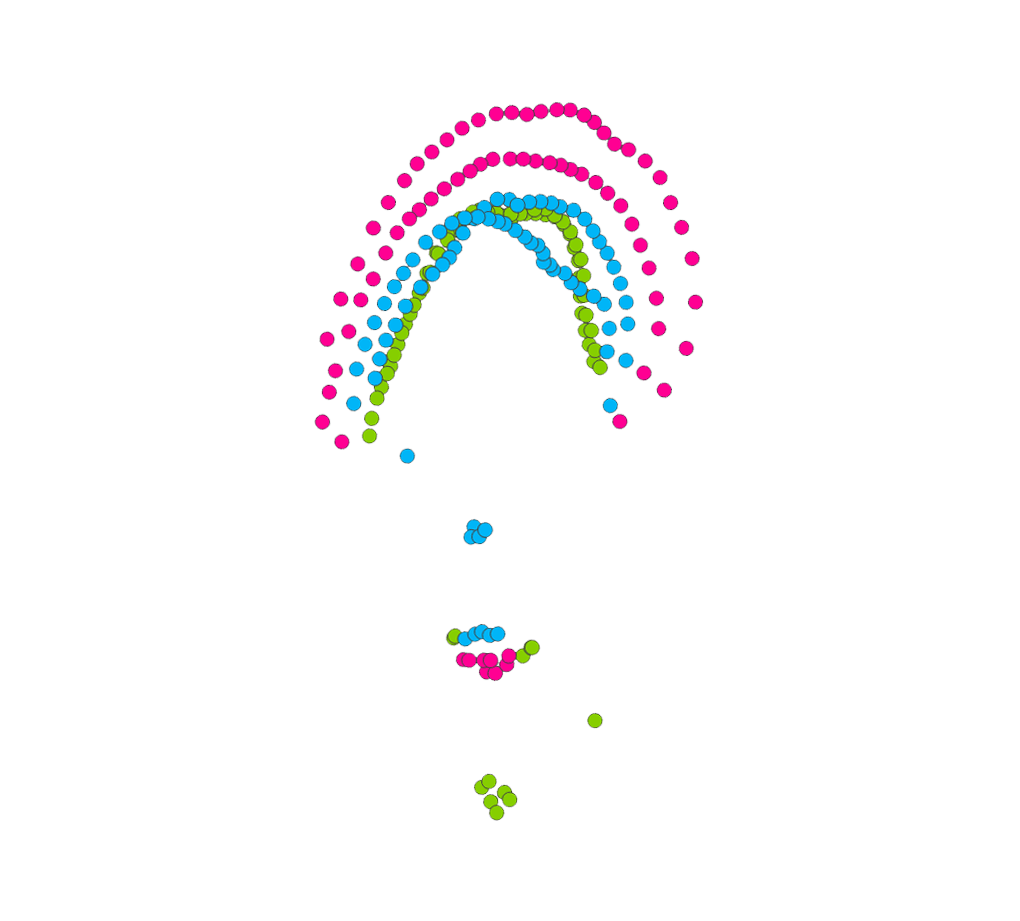}&
        \includegraphics[width=0.14\textwidth]{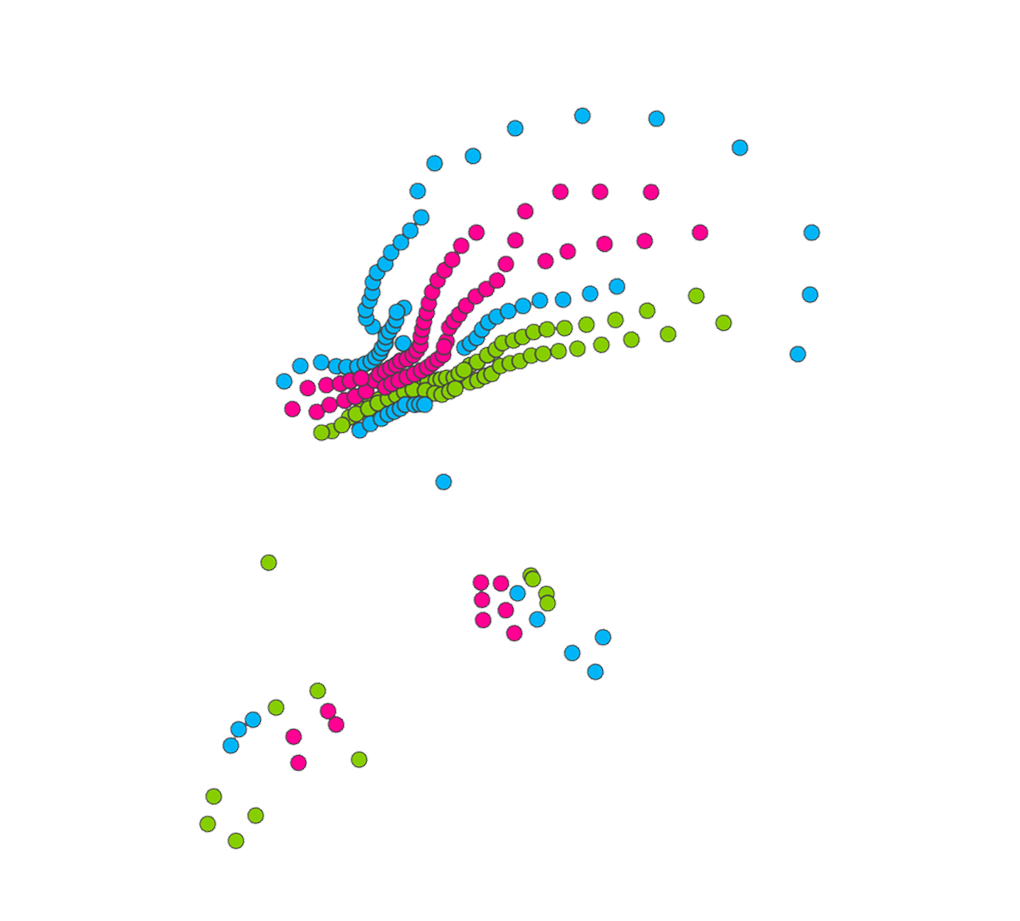}&
        \includegraphics[width=0.14\textwidth]{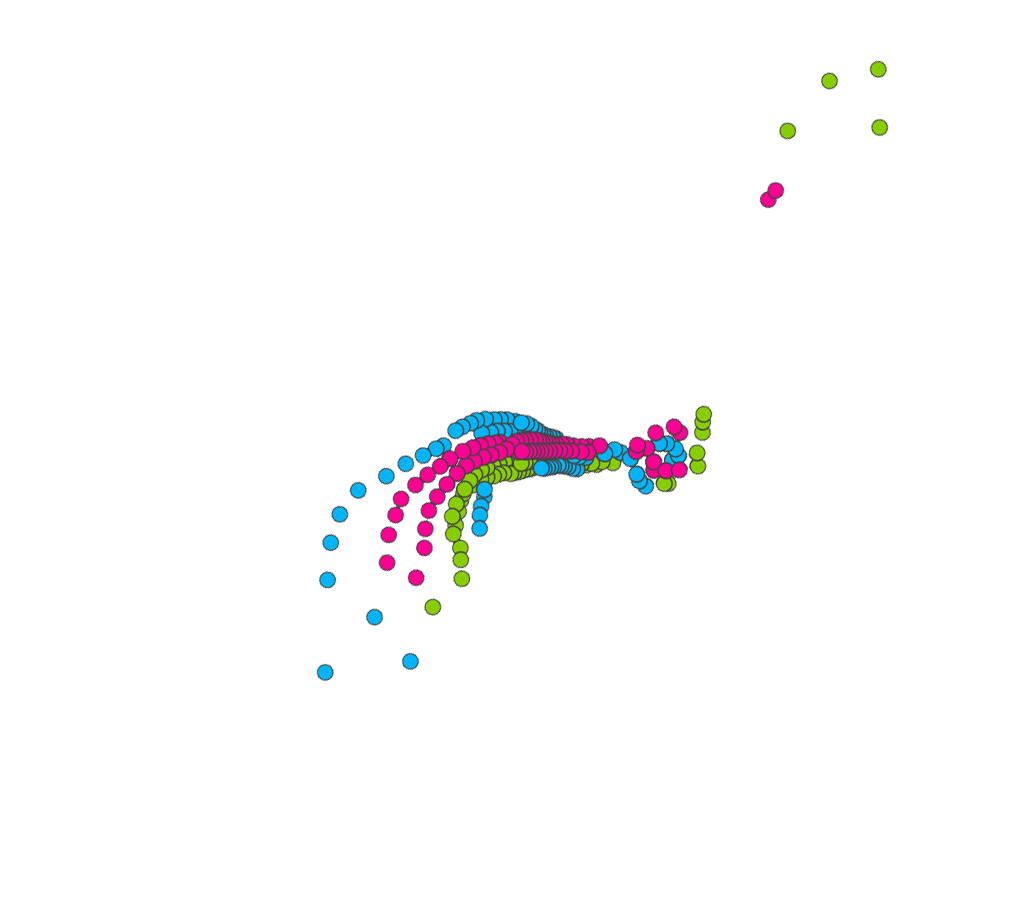}&
        \includegraphics[width=0.14\textwidth]{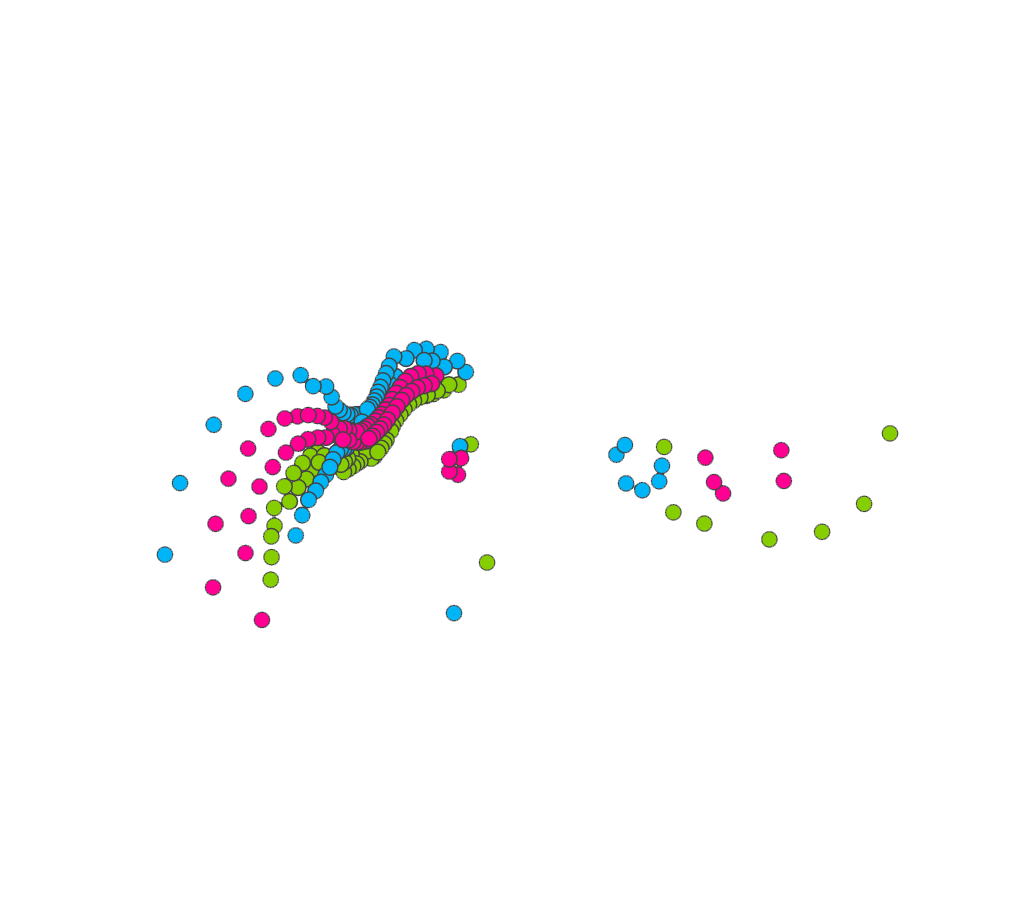}&
        \includegraphics[width=0.14\textwidth]{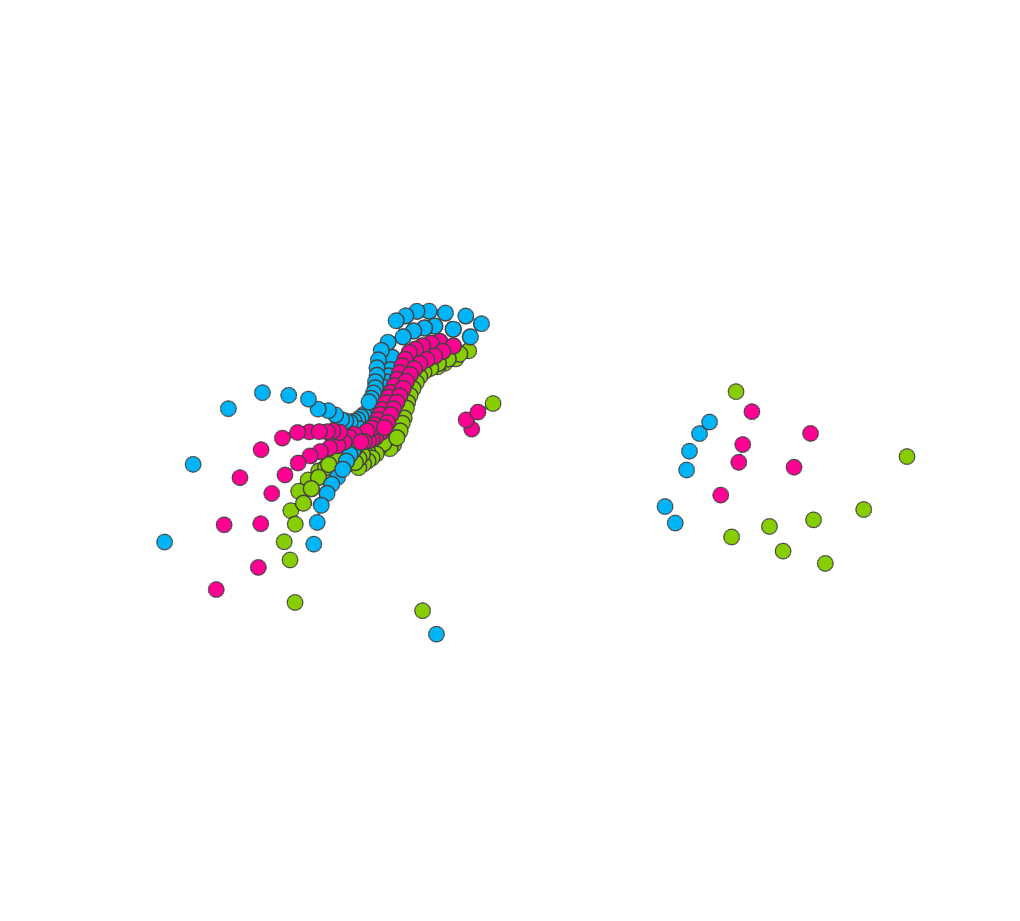}&
        \includegraphics[width=0.14\textwidth]{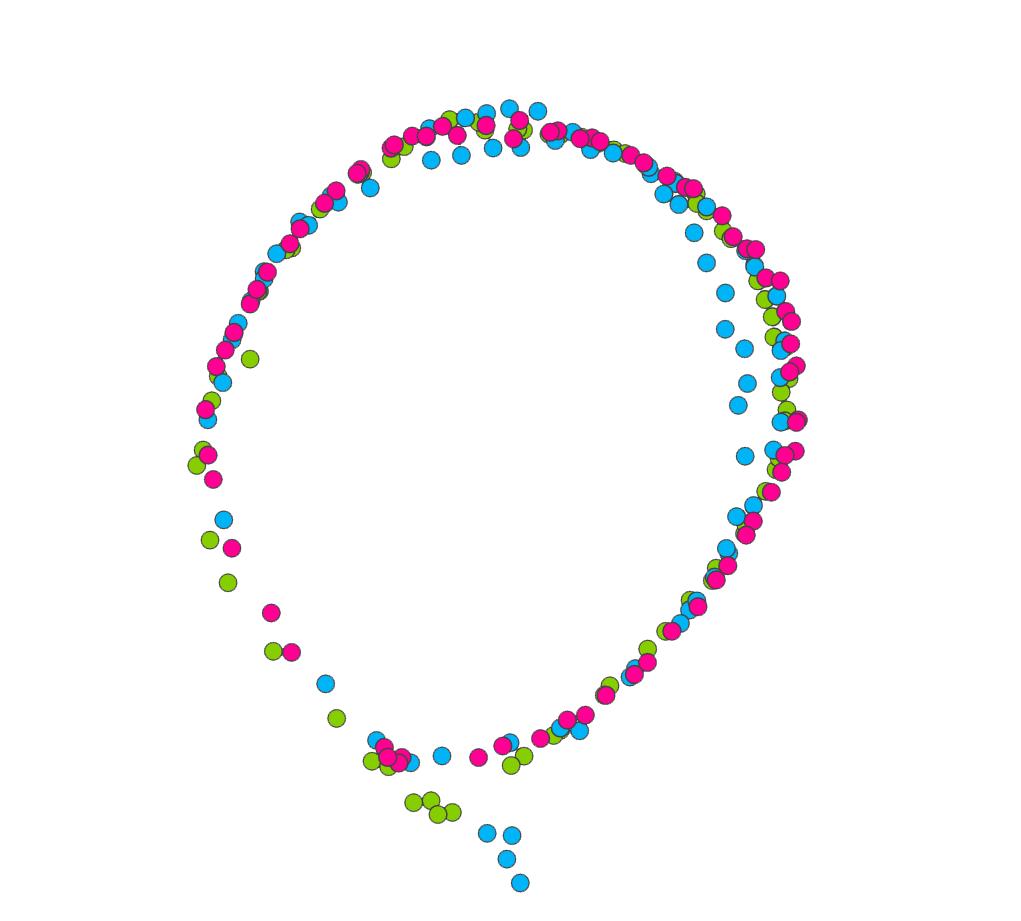}&
        \includegraphics[width=0.14\textwidth]{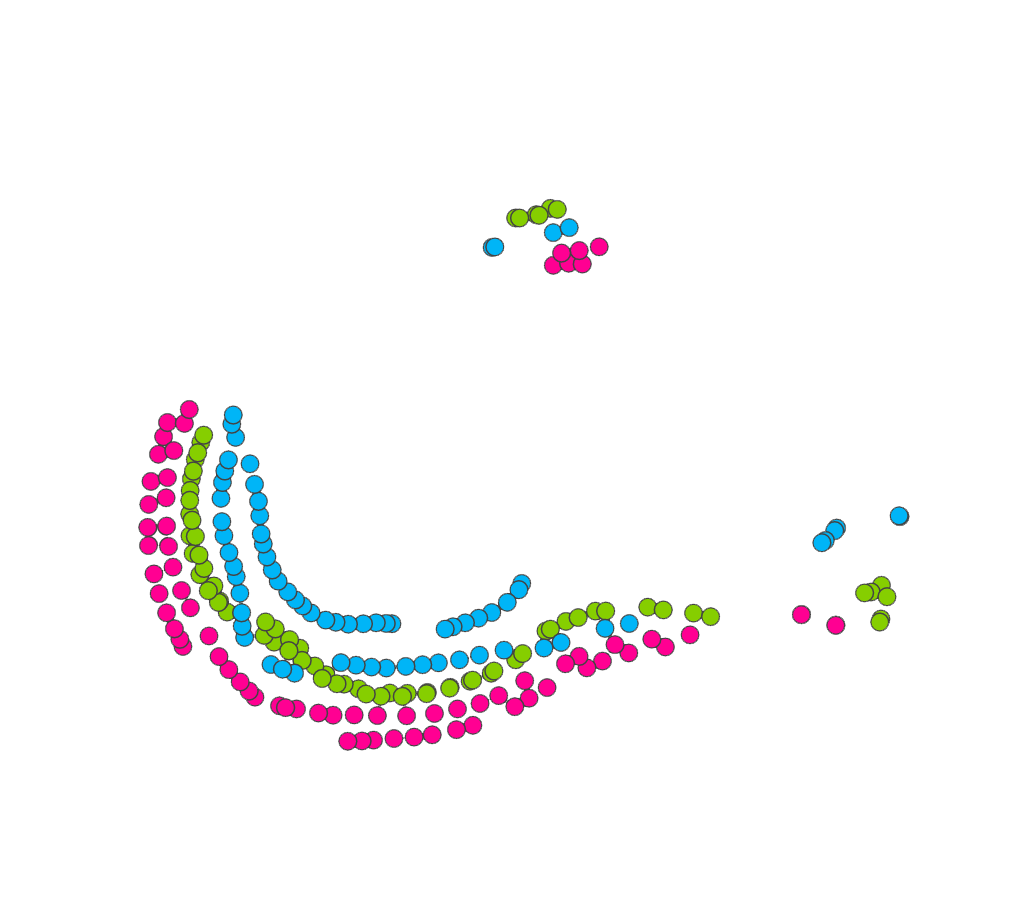}

    \end{tabular}
    
    \caption{Results of data visualization experiment. Different colors indicate different ground turth labels/classes.
    Top shows MNIST: FC architecture of the encoder/decoder (top row), and CNN (bottom row);
    Middle shows FMNIST: FC (top row), and CNN (bottom row); 
    Bottom shows COIL20 with CNN architecture, where zoom-ins of 3 classes are shown in the bottom row.}
    \label{fig:visualization}
\end{figure}

In this experiment we evaluate our method in the task of high dimension data visualization, \ie, reducing high dimensional data into two dimensions to be visually interpreted by a human observer or some downstream application. Usually the data is not assumed to lie on a manifold with such a low dimension, and it is therefore impossible to preserve all of its geometric properties. A common artifact when squeezing higher dimensional data into the plane is crowding \citep{maaten2008visualizing}, that is planar embedded points are crowded around the origin.

We evaluate our method on three standard datasets of images: MNIST \citep{lecun1998mnist} (60k hand-written digits), Fashion-MNIST (60k Zalando’s article images) \citep{xiao2017/online} and COIL20 \citep{nene1996columbia} (20 different images of object rotated with 72 even rotations).
For baselines we take: Vanilla AE; CAE; GP-RAE; DAE; U-MAP; and t-SNE.
We use the same architecture for all auto-encoder methods on each dataset. MNIST and FMNIST we evaluated in two scenarios: (i) Both encoder and decoder are fully-connected (MLP) networks; and (ii) Both encoder and decoder are Convolutional Neural Network (CNN). For COIL20 dataset both encoder and decoder are Convolutional Neural Network.  Full implementation details and hyper-parameters values can be found in the Appendix.

The results are presented in figure \ref{fig:visualization}; where each embedded point $\vz$ is colored by its ground-truth class/label.
We make several observation. First, in all the datasets our method is more resilient to crowding compared to the baseline AEs, and provide a more even spread. U-MAP and t-SNE produce better separated clusters. However, this separation can come at a cost: See the COIL20 result (third row) and blow-ups of three of the classes (bottom row). In this dataset we expect evenly spaced points that correspond to the even rotations of the objects in the images. Note (in the blow-ups) that U-MAP maps the three classes on top of each other (non-injectivity of the "encoder"), t-SNE is somewhat better but does not preserve well the distance between pairs of data points (we expect them to be more or less equidistant in this dataset). In I-AE the rings are better separated and points are more equidistant; the baseline AEs tend to densify the points near the origin.   
Lastly, considering the inter and intra-class variations for the MNIST and FMNIST datasets, we are not sure that isometric embeddings are expected to produce strongly separated clusters as in U-MAP and t-SNE (\eg, think about similar digits of different classes and dissimilar digits of the same class with distances measured in euclidean norm).  

%are not sure how faithful the strong separation
%However this separation appears a bit artificial; for instance in MNIST don't expect a full separation between class since some digit can resemble each other. 

%In addition we can see that though in a first glance the result of u-map on the COIL20 dataset is well spread, in a close glance as seen on the last row of \ref{fig:visualization} we see that it embed three different classes on the same ring, while our method do much better job of separation.

%In the COIL20 dataset we get circular clusters shape in the low dimension; due to the nature of the dataset. Compared to the other auto-encoders our clusters tends to be more circular. In the last row the red scatter shows how our method, and u-map returns almost constant distances between neighboring points within the class; while in other methods the distance of neighboring points varies. Both of the circular phenomena, and the constant distance phenomena could be explained the isometric nature of our method; the inputs are rotated in a constant manner and each constant rotation translated to a constant distance between latent points.

\subsection{Hyper-Parameters Sensitivity}
To evaluate the affect of $ \lambda_{\text{iso}}$ on the output we compared the visualizations and optimized loss values of MNIST and FMNIST, trained with same CNN architecture as in Section \ref{ss:visulaization} with $ \lambda_{\text{iso}}\in\{0,0.01,0.025,0.05,0.075,0.1,0.25, 0.5, 0.75, 0.1\} $. Figure \ref{fig:hyperparameters} shows the different visualization results as well as $ L_{\text{rec}},L_{\text{iso}},L_{\text{piso}} $ as a function of $\lambda_{\text{iso}}$. As can be seen in both datasets the visualizations and losses are stable for $\lambda_{\text{iso}}$ values between $0.01$ and $0.5$, where a significant change to the embedding is noticeable at $0.75$. The trends in the loss values are also rather stable; $ L_{\text{iso}}$ and $ L_{\text{piso}}$ start very high in the regular AE, \ie,  $\lambda_{\text{iso}}=0$, and quickly stabilize. As for $L_{\text{rec}}$ on FMNIST we see a stable increase while in MNIST it also starts with a steady increase until it reaches $0.75$ and then it starts to be rockier, which is also noticeable in the visualizations.

\begin{figure}
    \centering
    \setlength\tabcolsep{0.3pt}
     \begin{tabular}{c c c c c c c c c c}
          $\lambda_{\text{iso}}=0 $ & $0.01$ & $0.025$ & $ 0.05 $ & $0.075 $ & $0.1$ & $0.25$ & $ 0.5 $ & $ 0.75 $ & $ 1 $ \\
          \hline
          \includegraphics[width=0.1\textwidth]{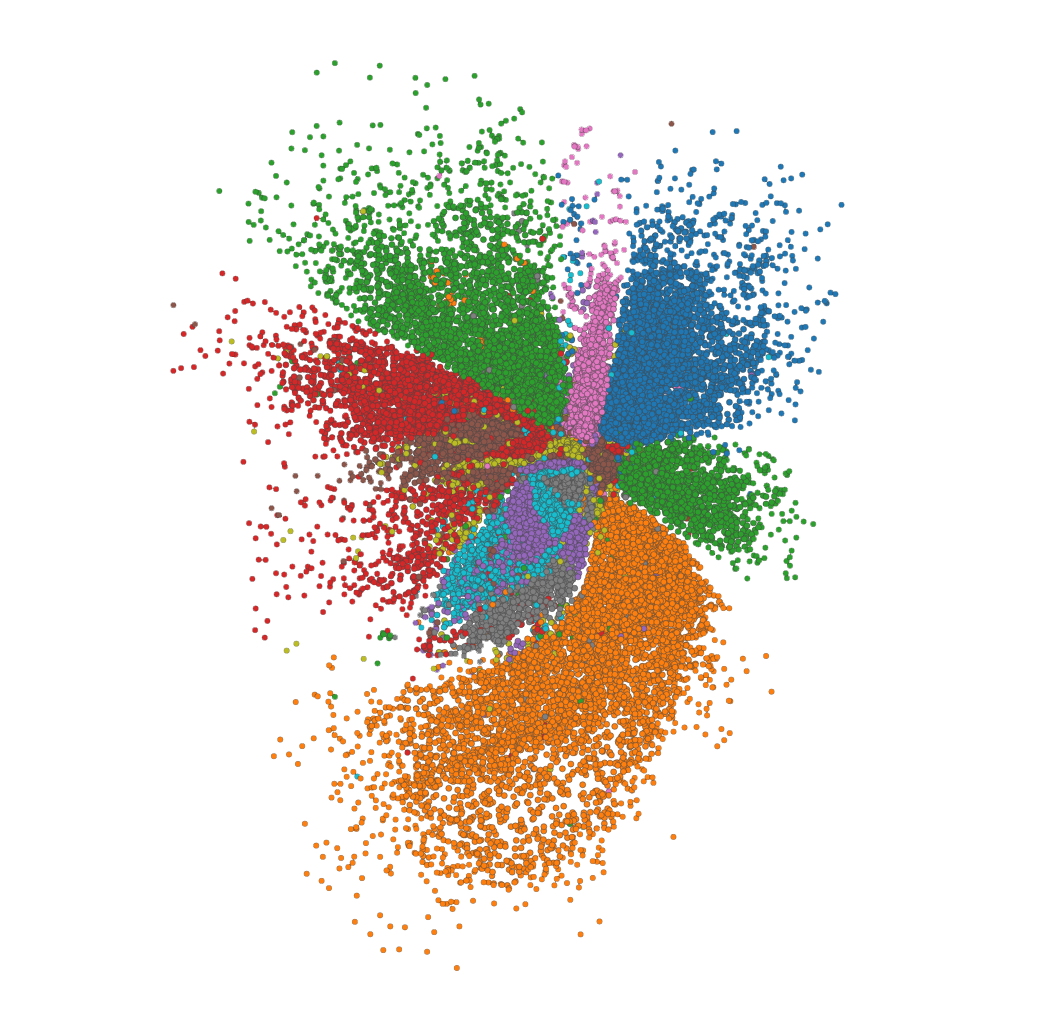}&
        \includegraphics[width=0.1\textwidth]{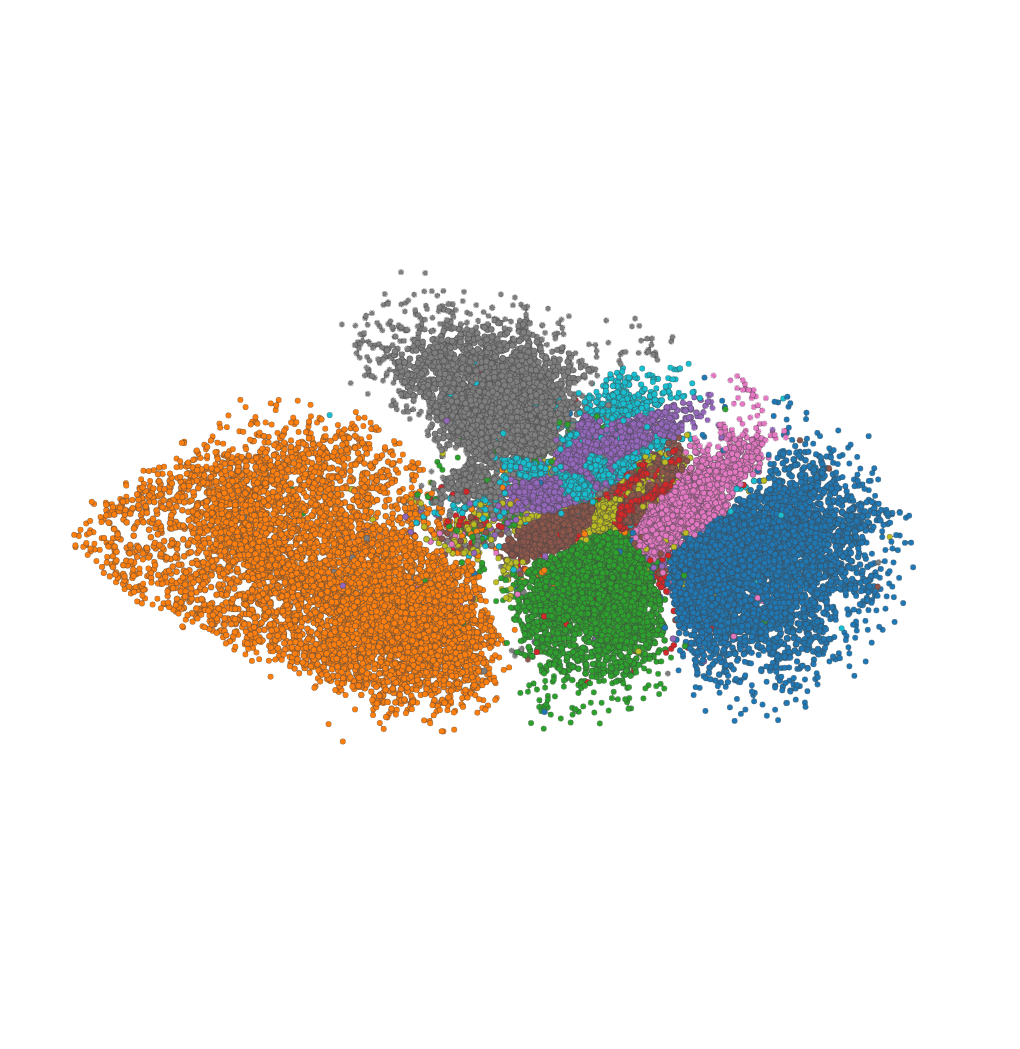}&
        \includegraphics[width=0.1\textwidth]{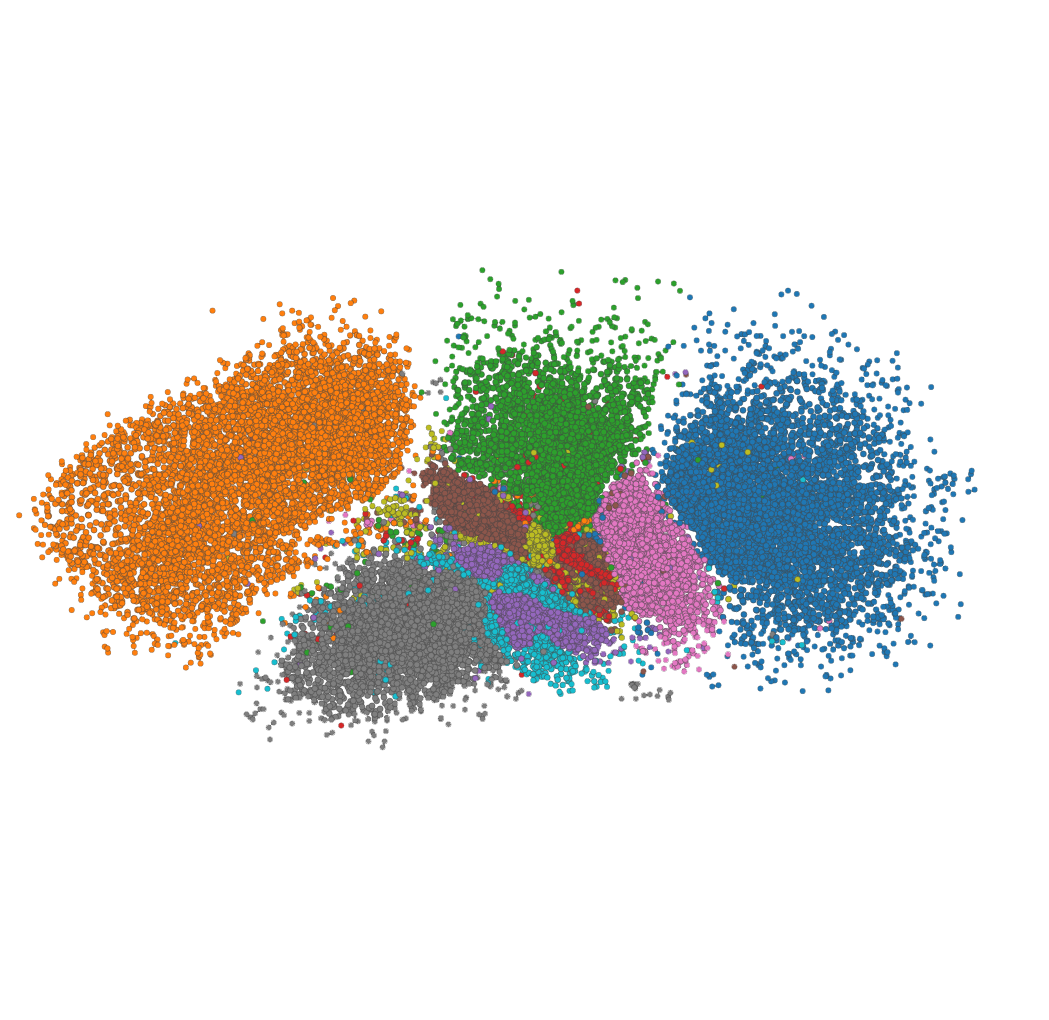}&
        \includegraphics[width=0.1\textwidth]{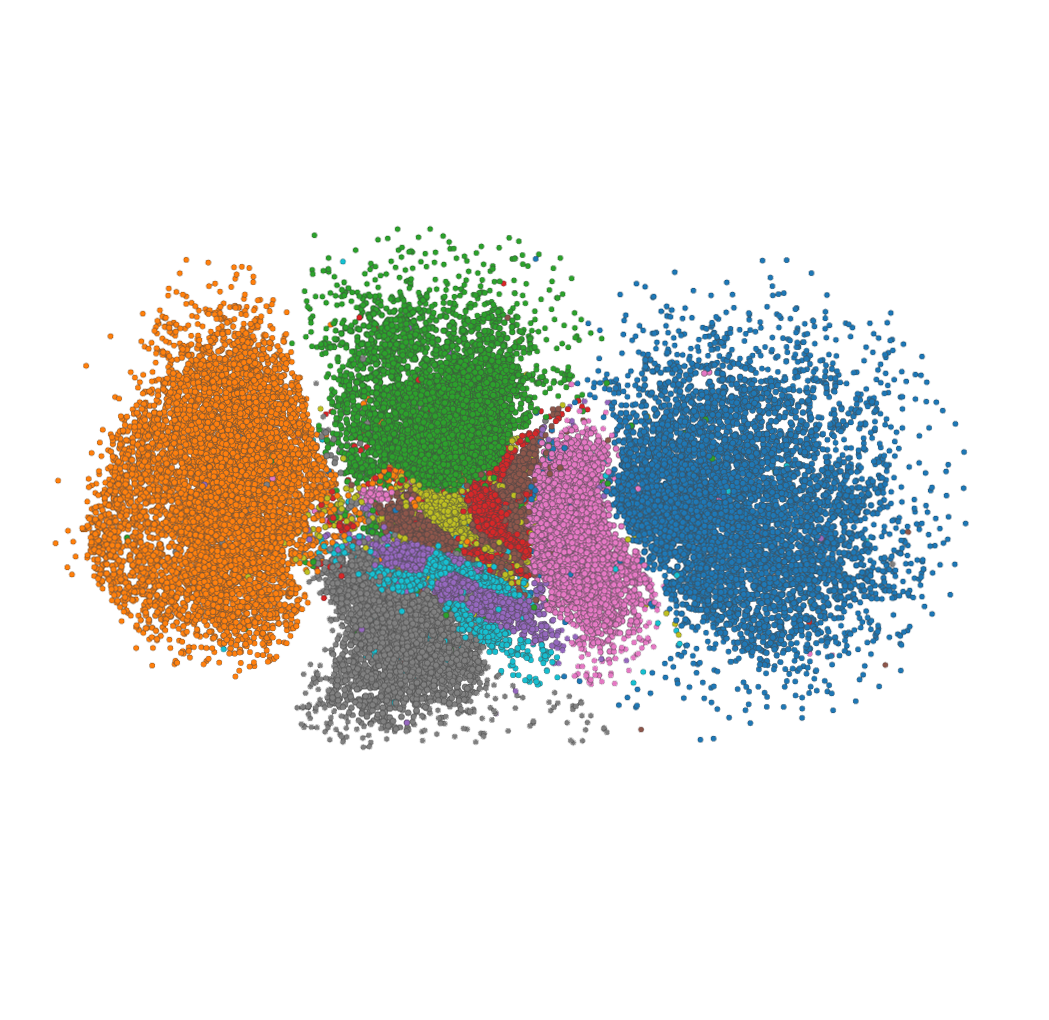}&
        \includegraphics[width=0.1\textwidth]{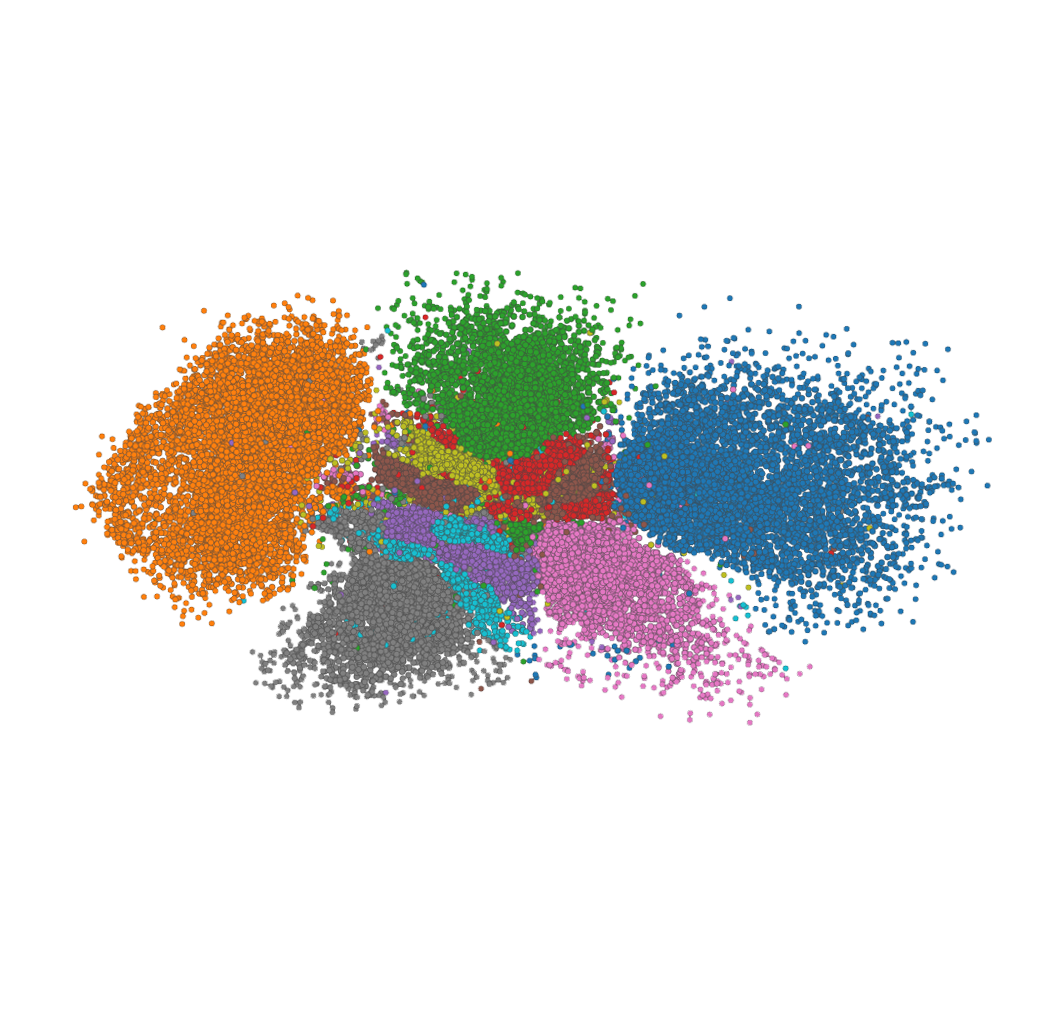}&
        \includegraphics[width=0.1\textwidth]{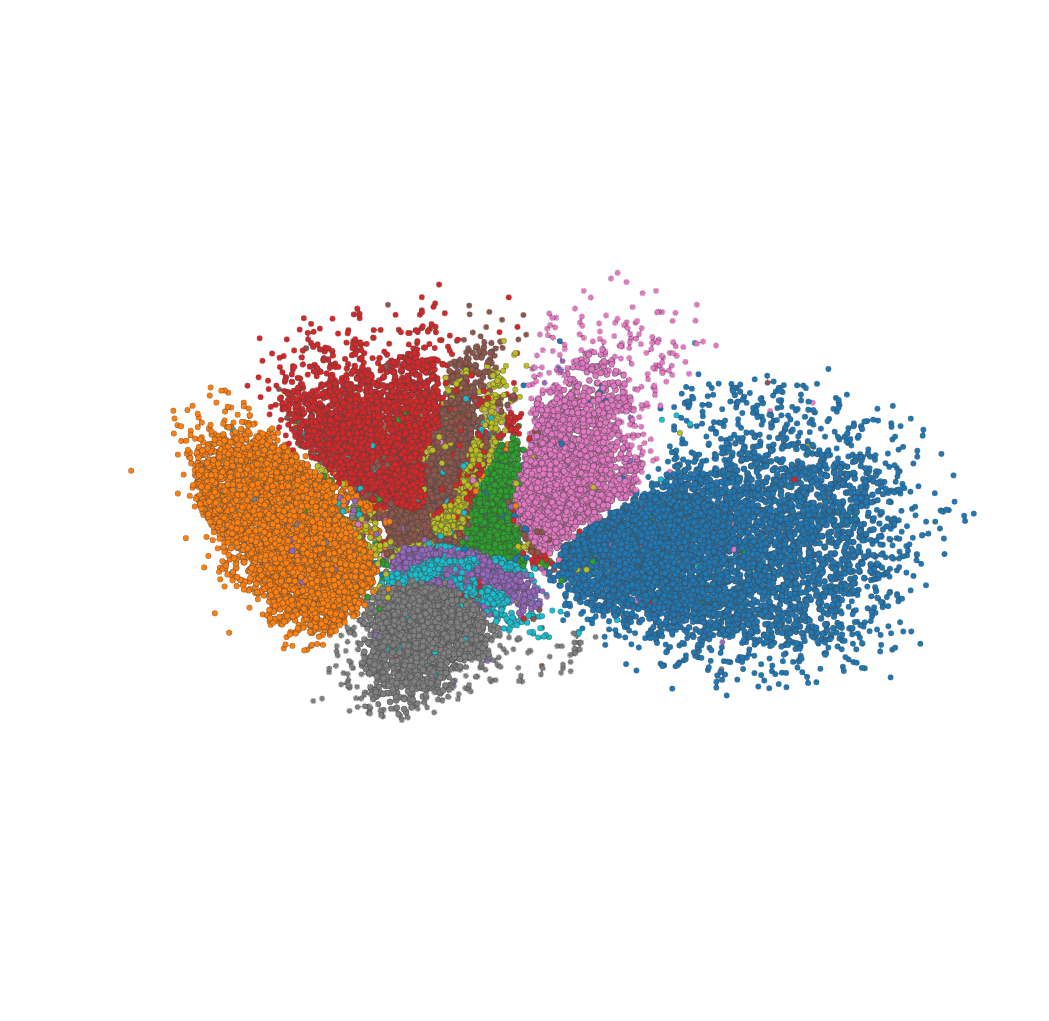}&
        \includegraphics[width=0.1\textwidth]{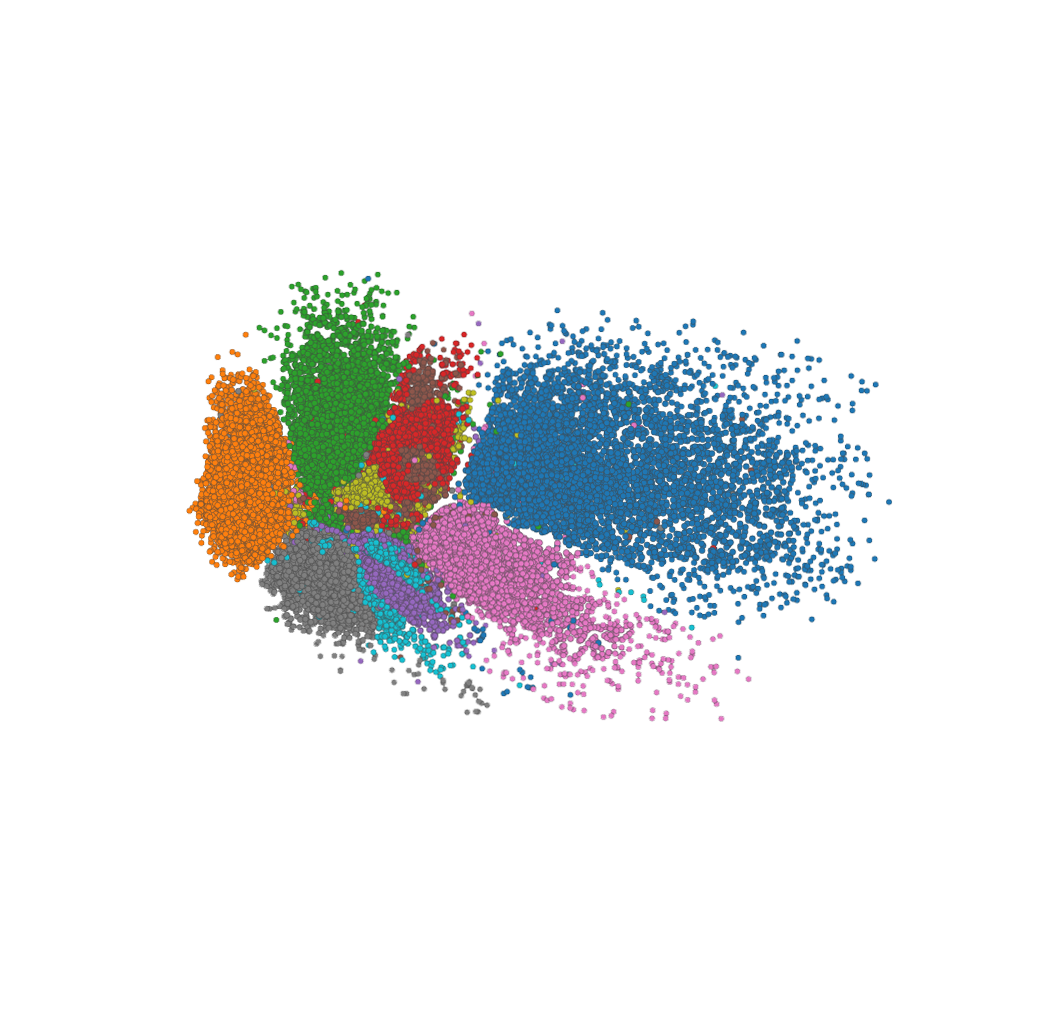}&
        \includegraphics[width=0.1\textwidth]{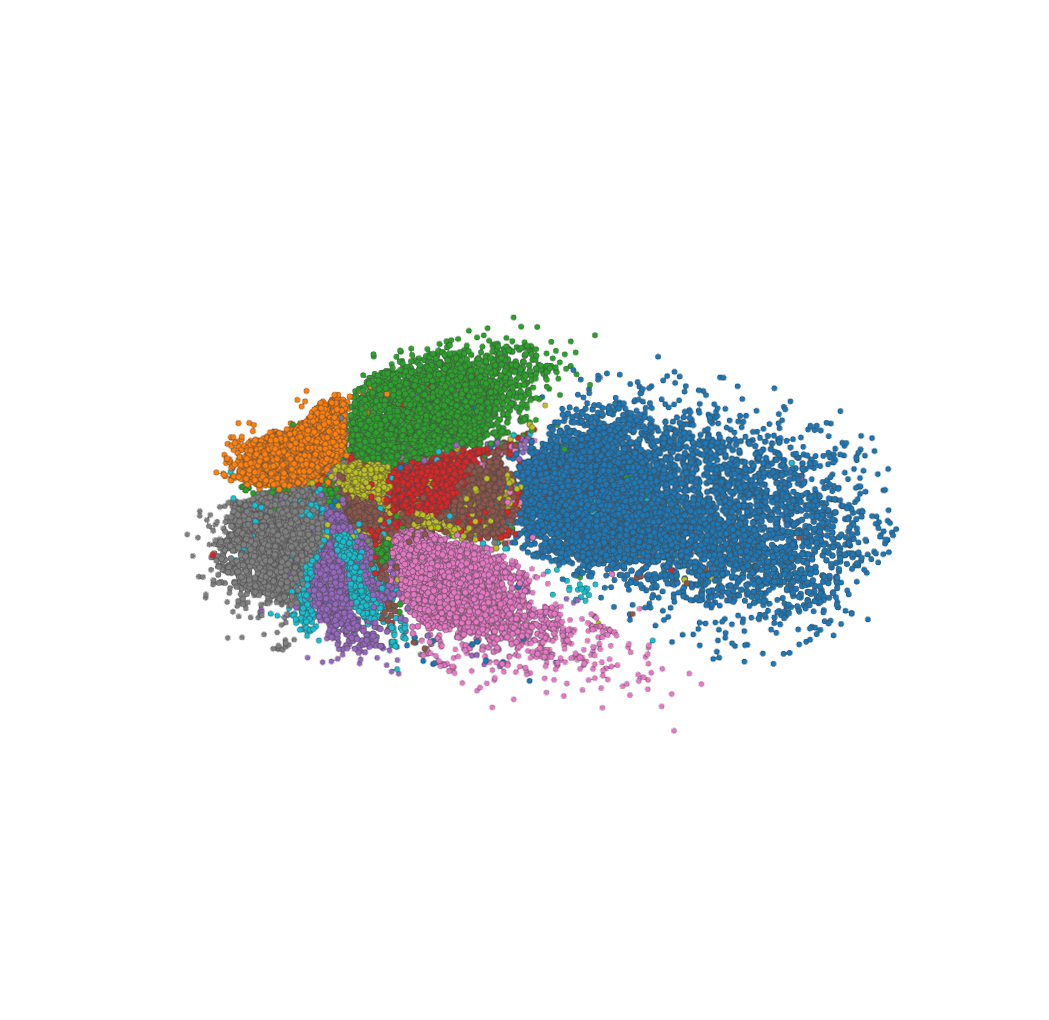}&
        \includegraphics[width=0.1\textwidth]{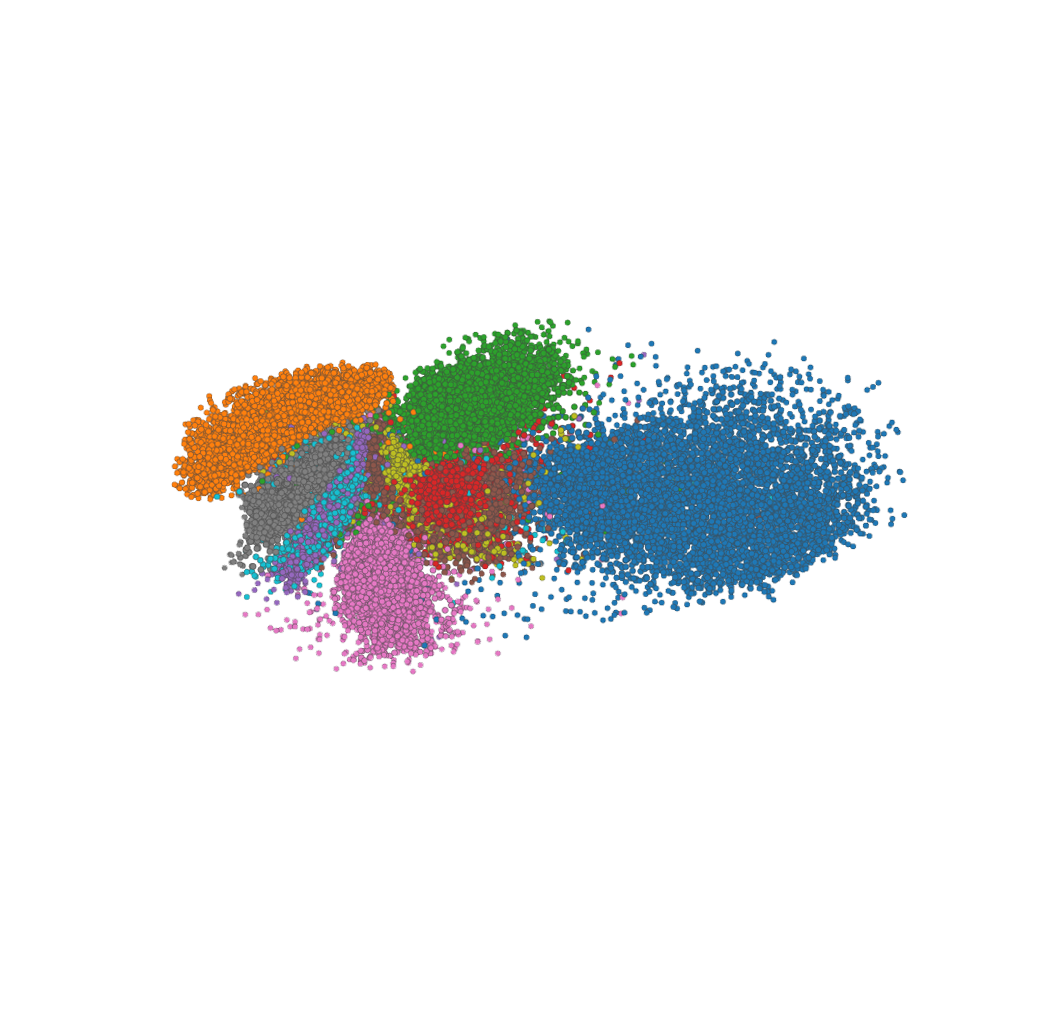}&
        \includegraphics[width=0.1\textwidth]{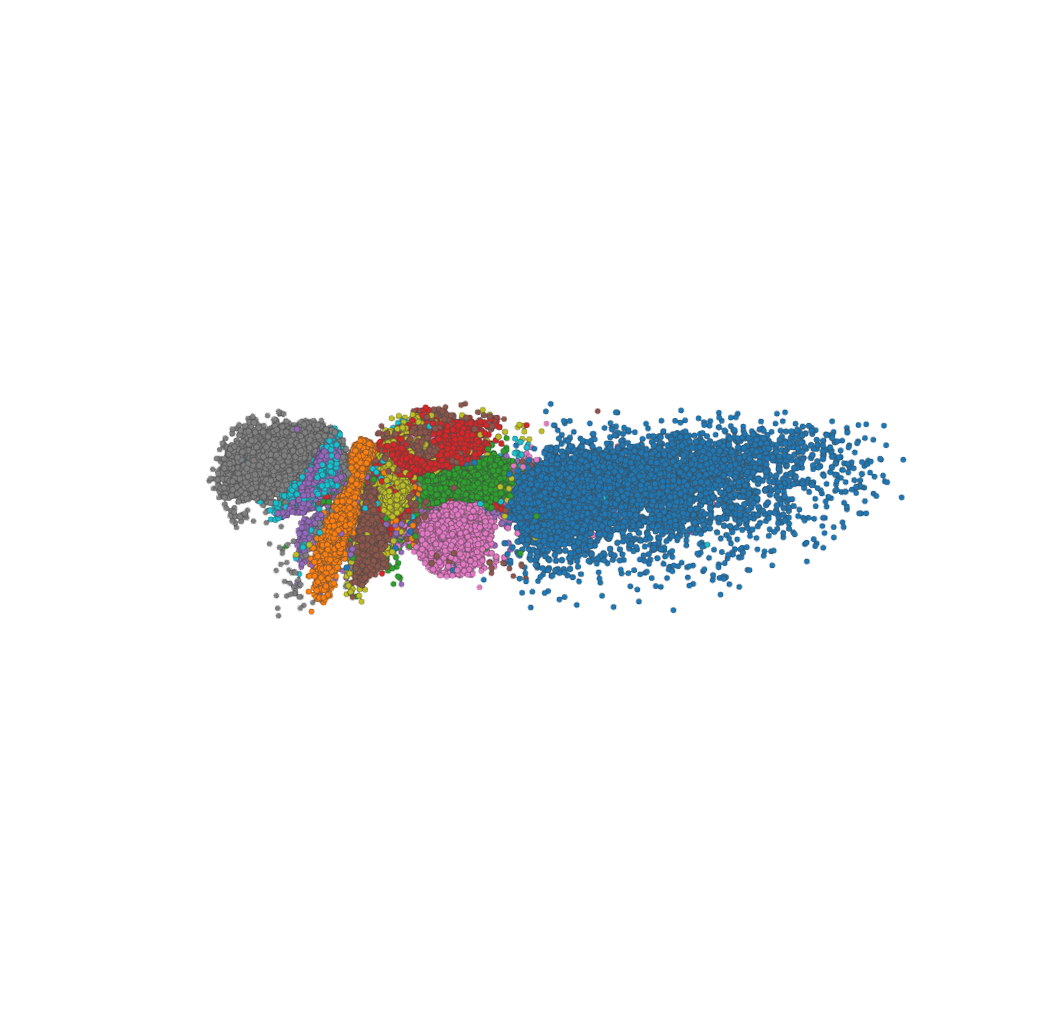}\\
        %\rotatebox{90}{Visualization} &
        \includegraphics[width=0.1\textwidth]{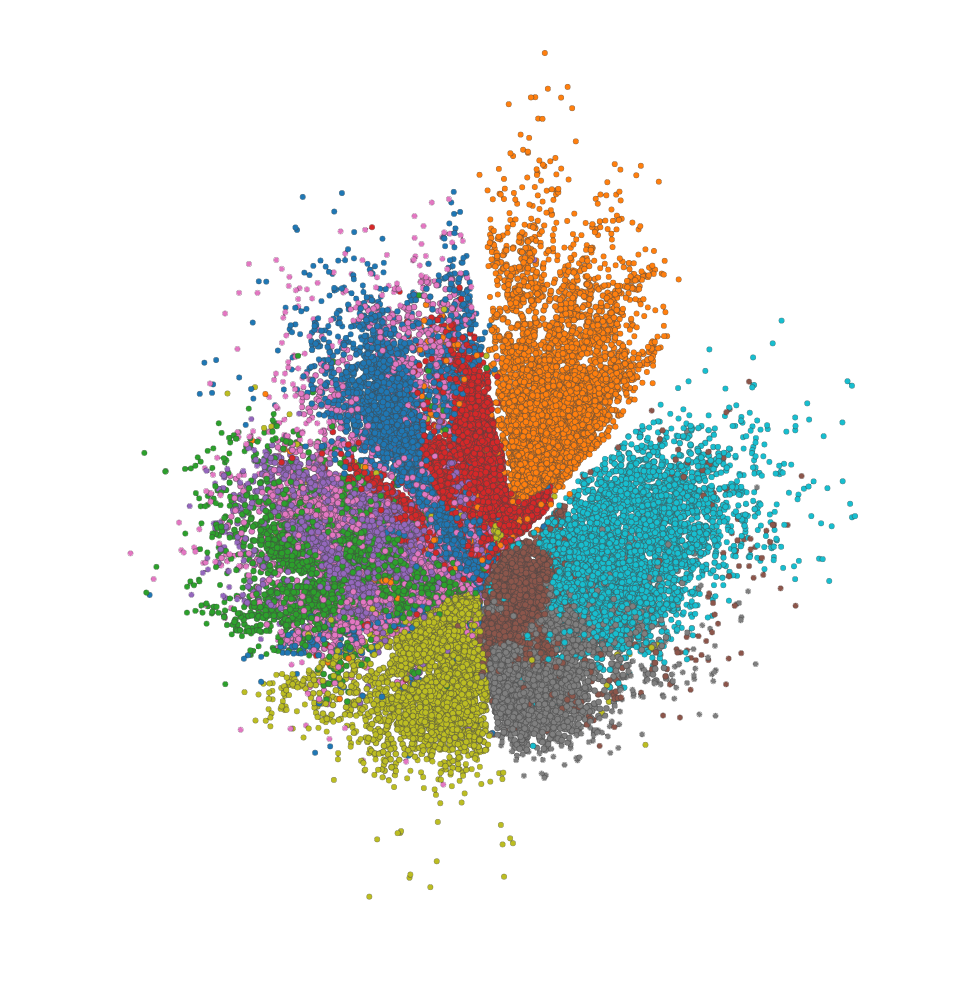}&
        \includegraphics[width=0.1\textwidth]{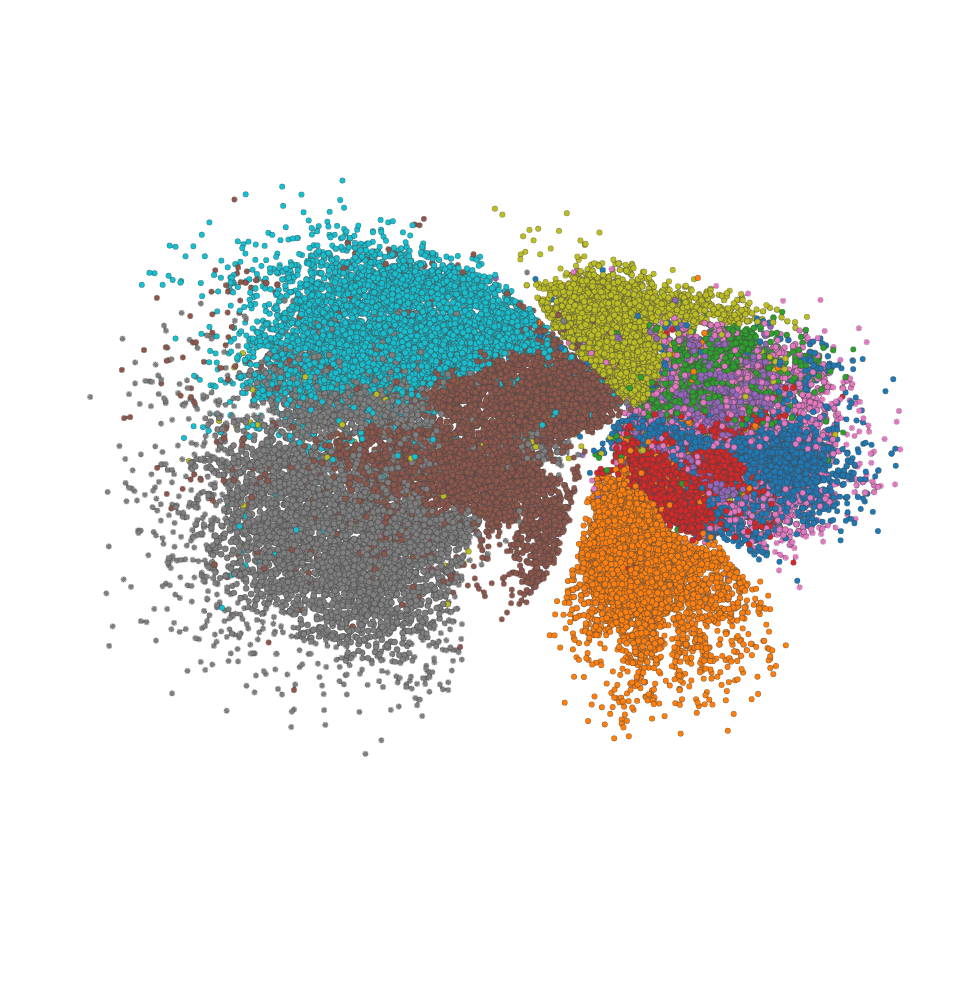}&
        \includegraphics[width=0.1\textwidth]{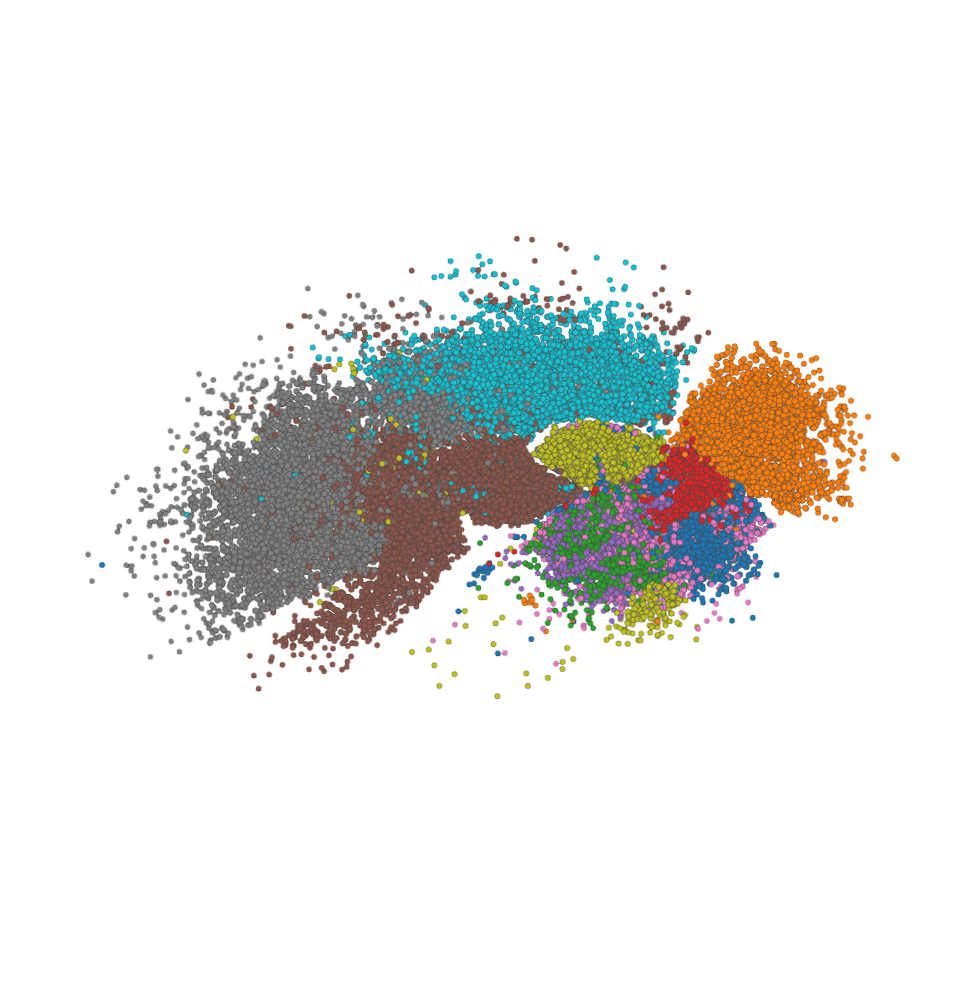}&
        \includegraphics[width=0.1\textwidth]{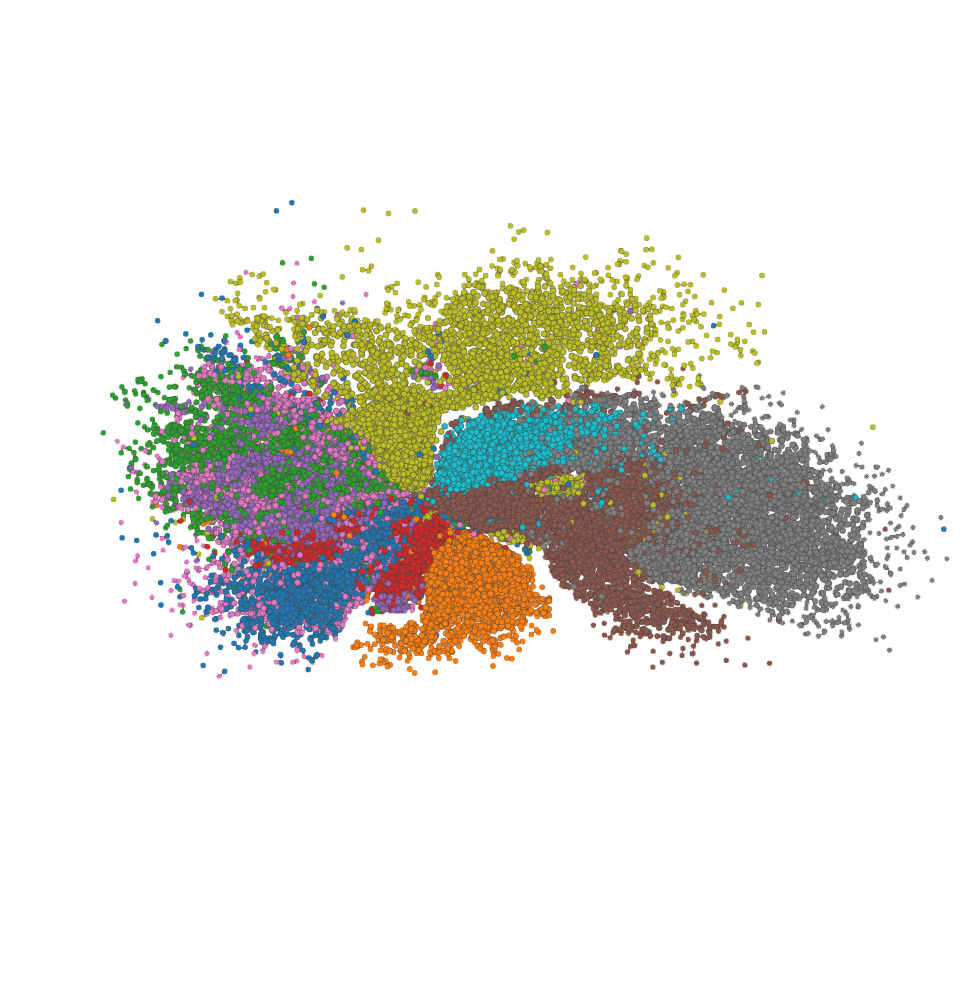}&
        \includegraphics[width=0.1\textwidth]{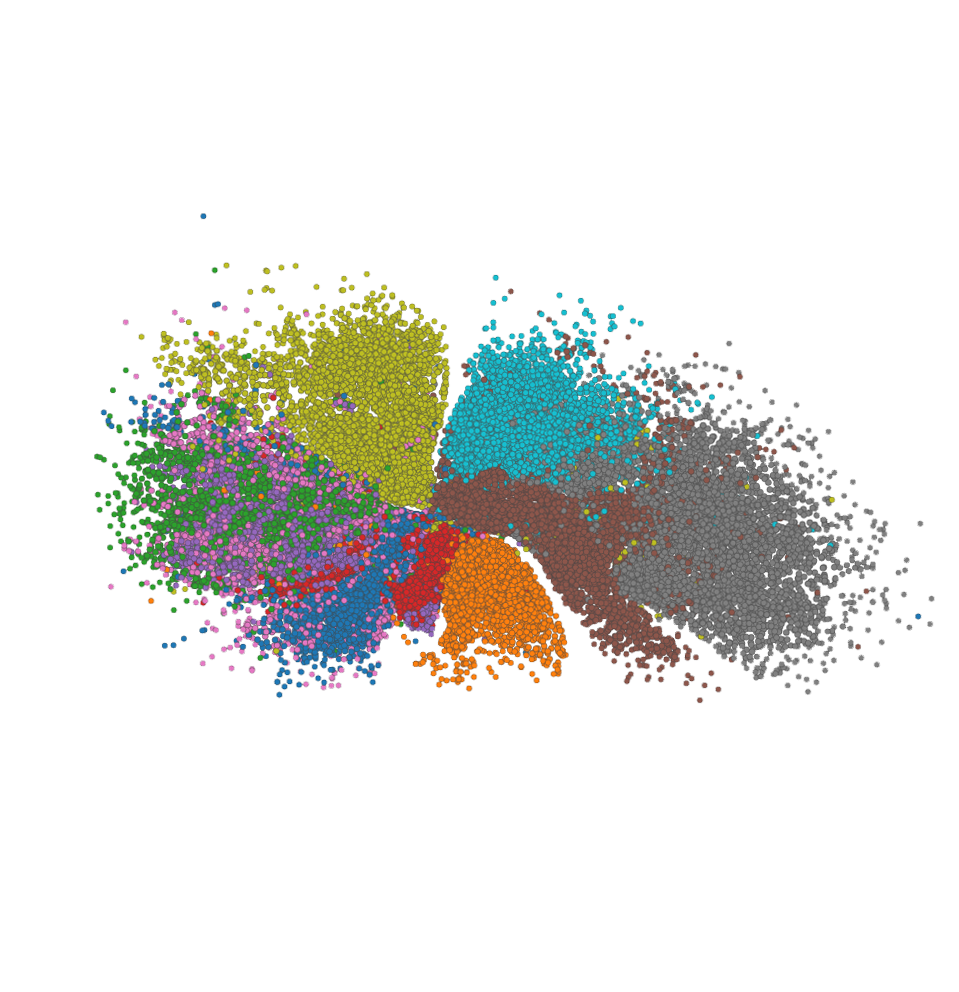}&
        \includegraphics[width=0.1\textwidth]{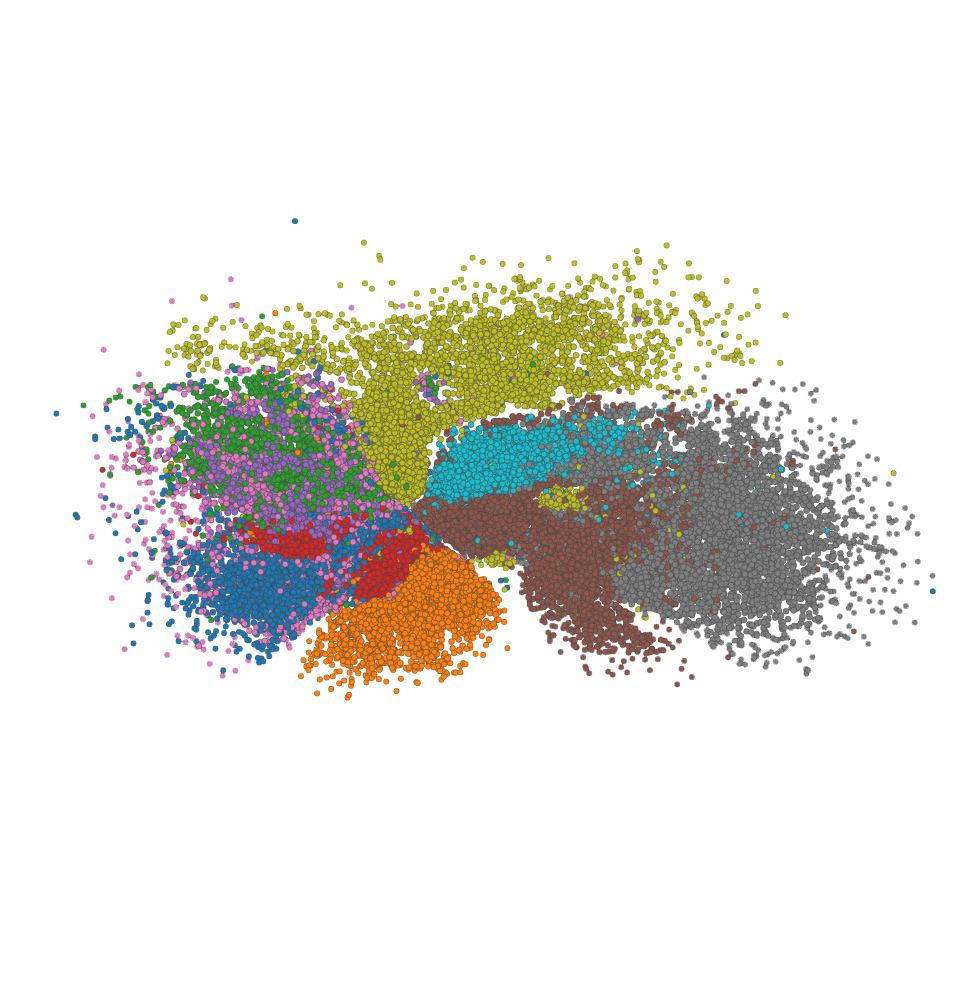}&
        \includegraphics[width=0.1\textwidth]{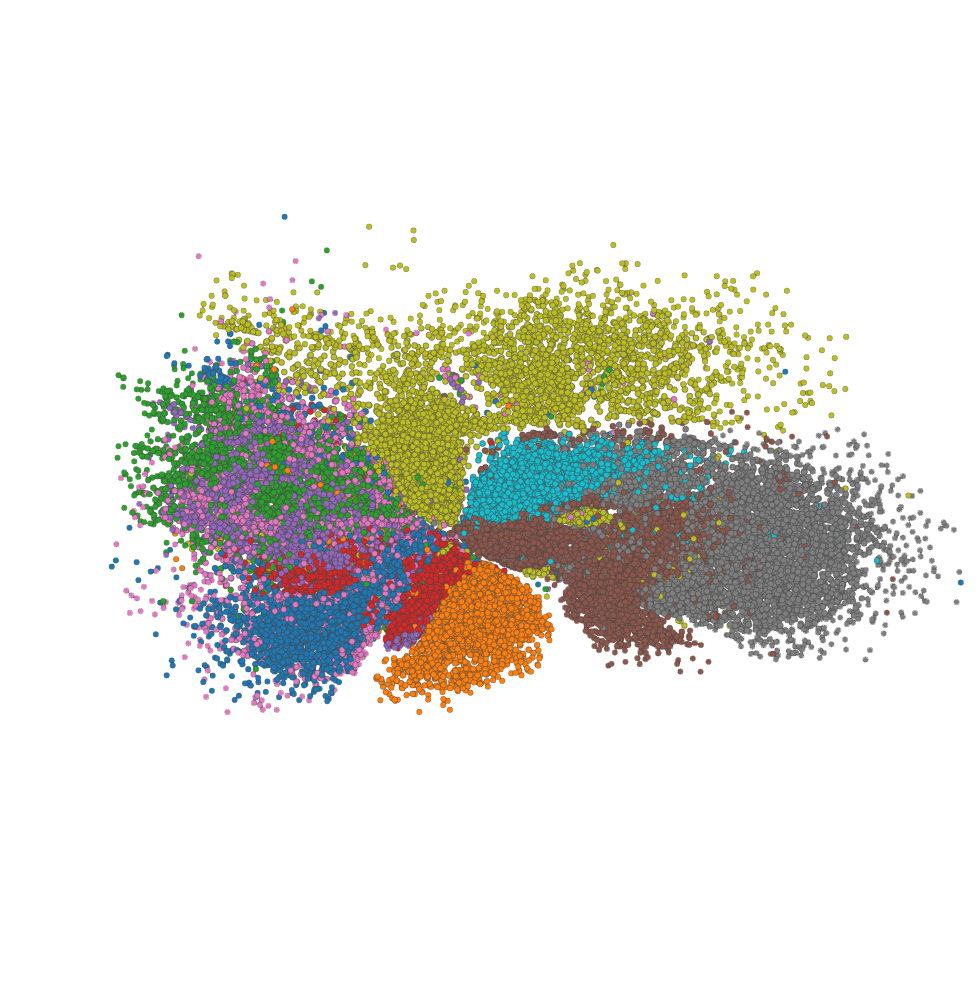}&
        \includegraphics[width=0.1\textwidth]{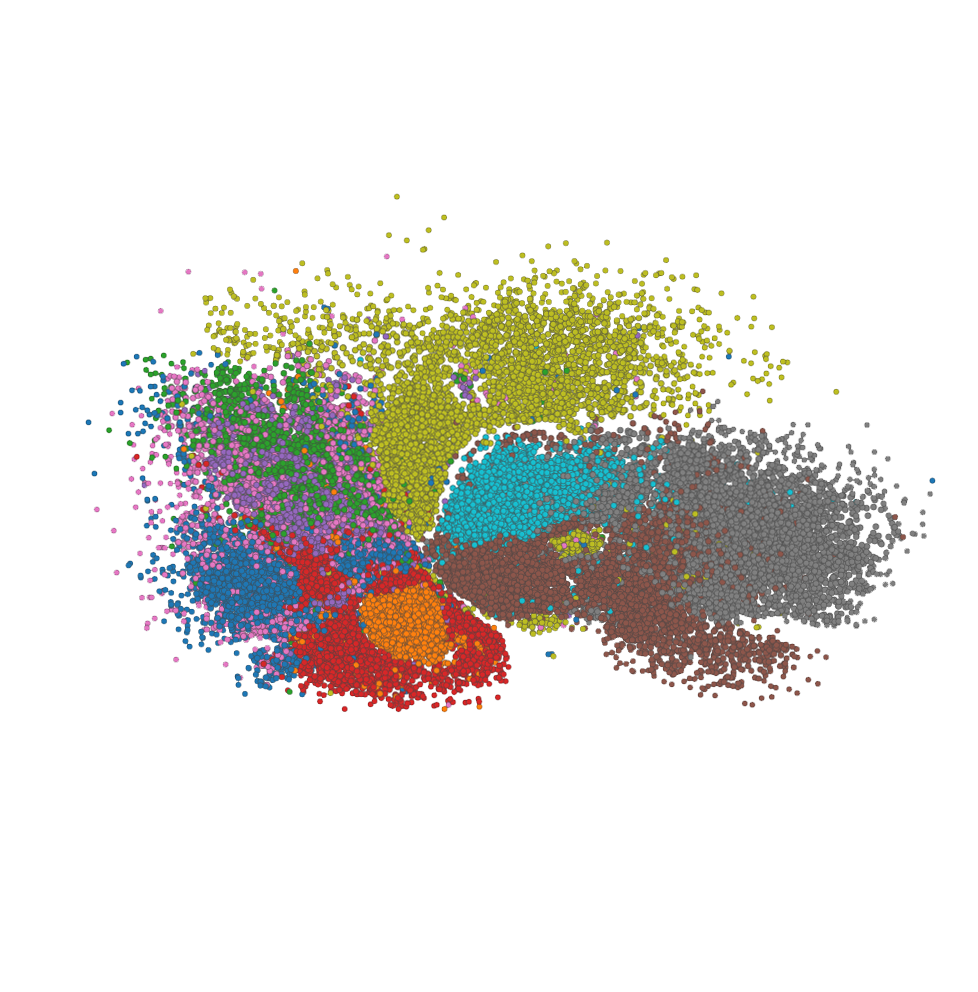}&
        \includegraphics[width=0.1\textwidth]{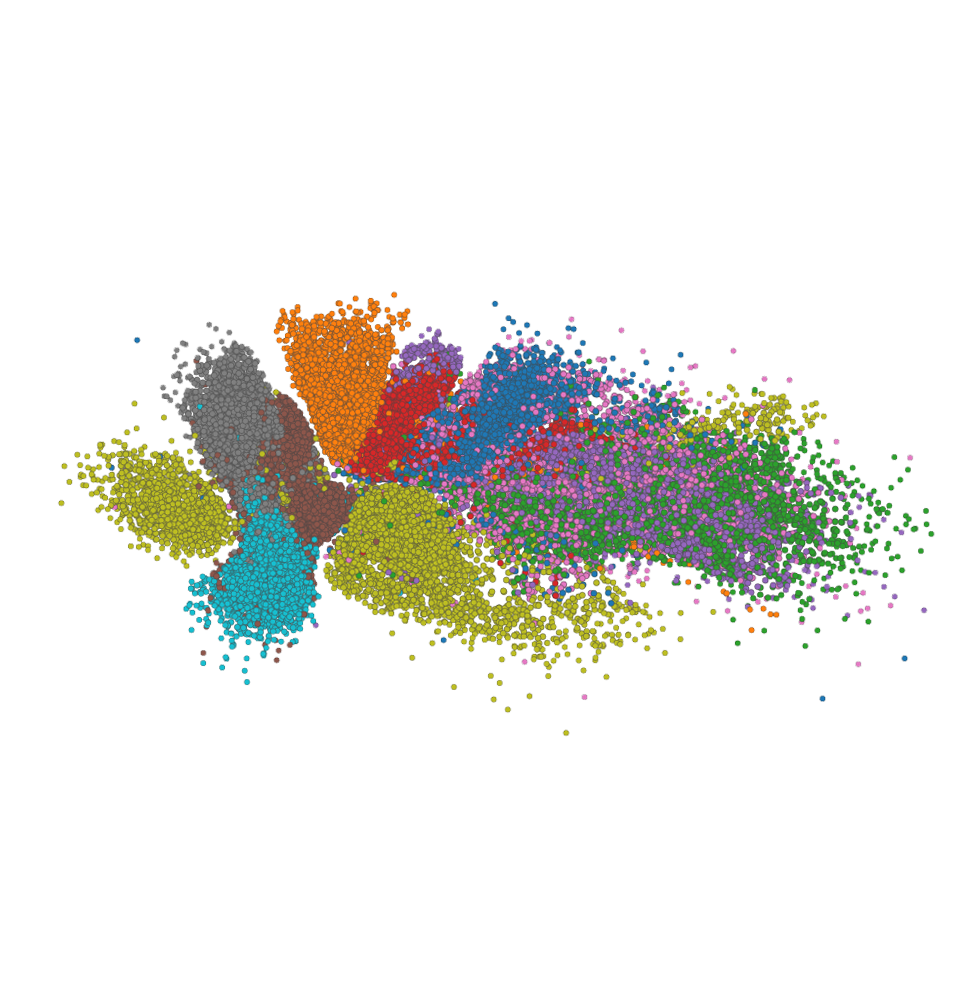}&
        \includegraphics[width=0.1\textwidth]{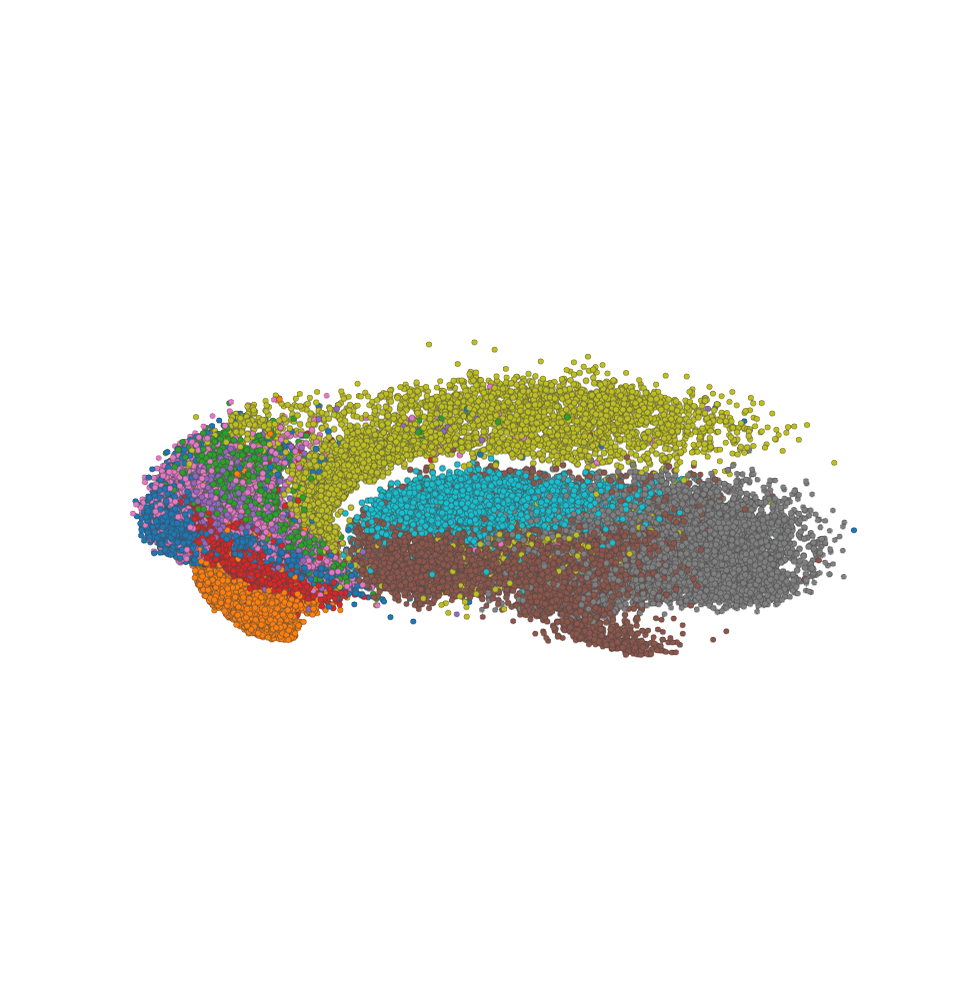}\\
        
\end{tabular}
    \setlength\tabcolsep{0.3pt}

\begin{tabular}{c c c}
    $ L_{\text{rec}}$ & $ L_{\text{iso}}$ & $ L_{\text{piso}}$ 
    \\
    \includegraphics[width=0.33\textwidth]{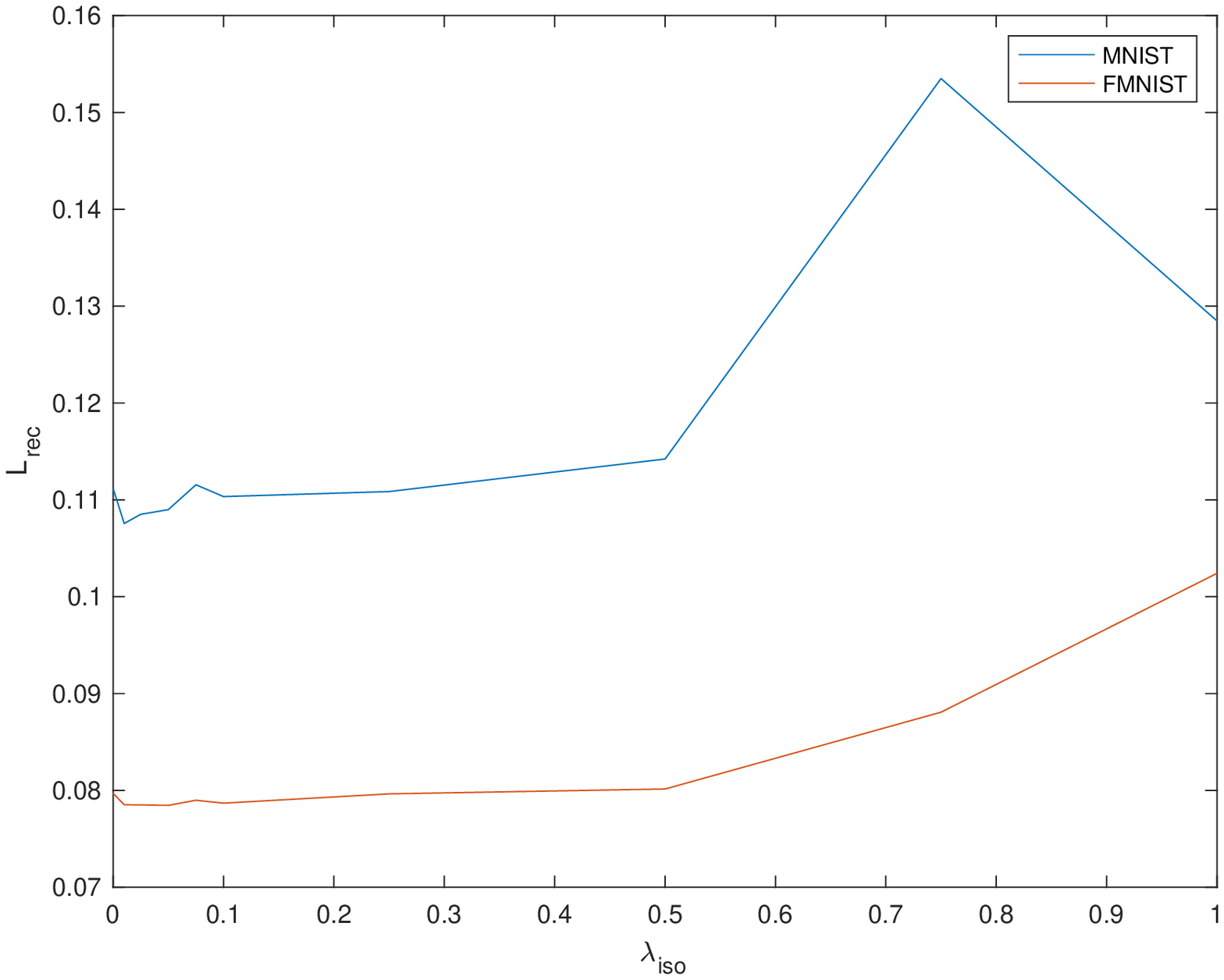} &
    \includegraphics[width=0.33\textwidth]{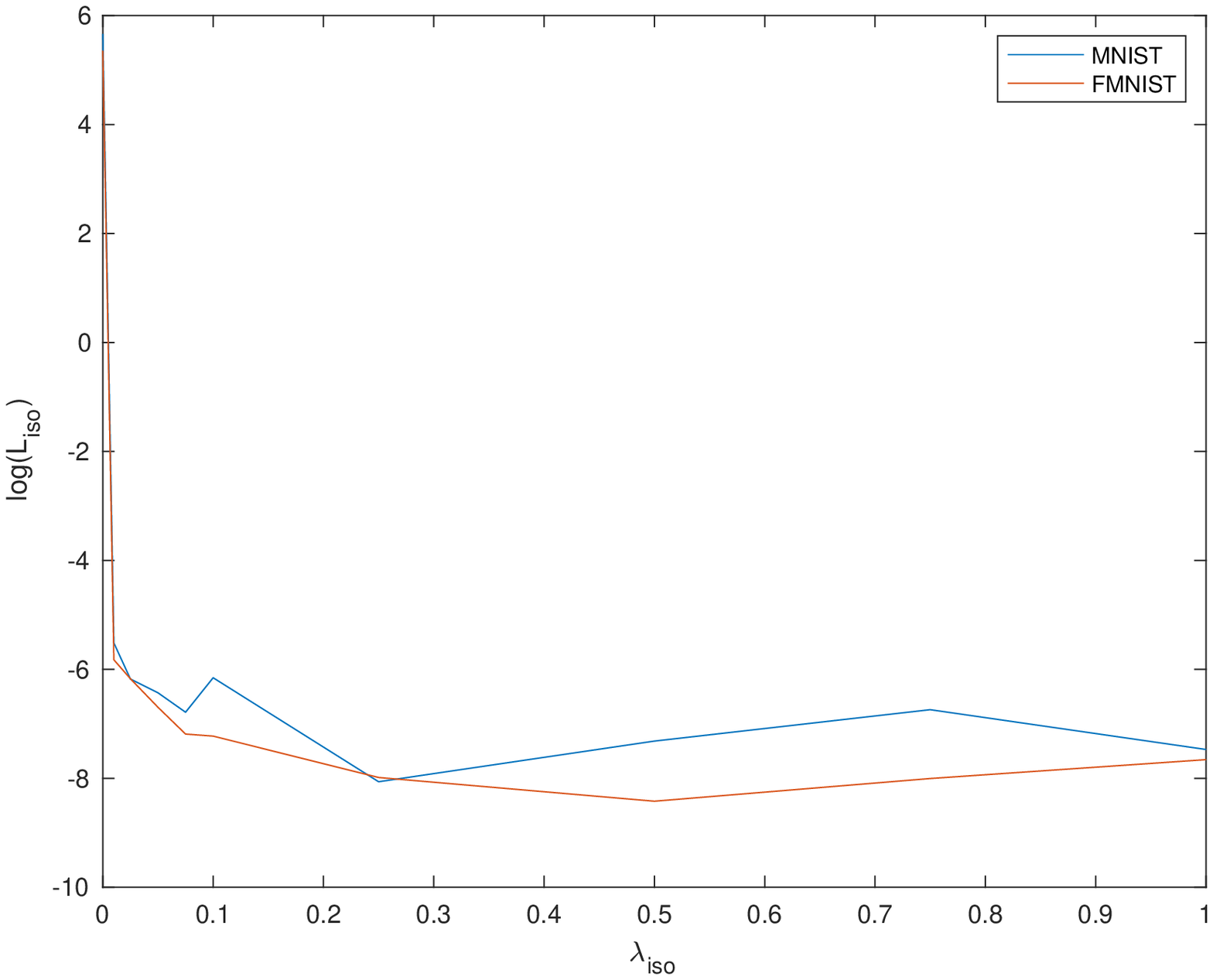} &
    \includegraphics[width=0.33\textwidth]{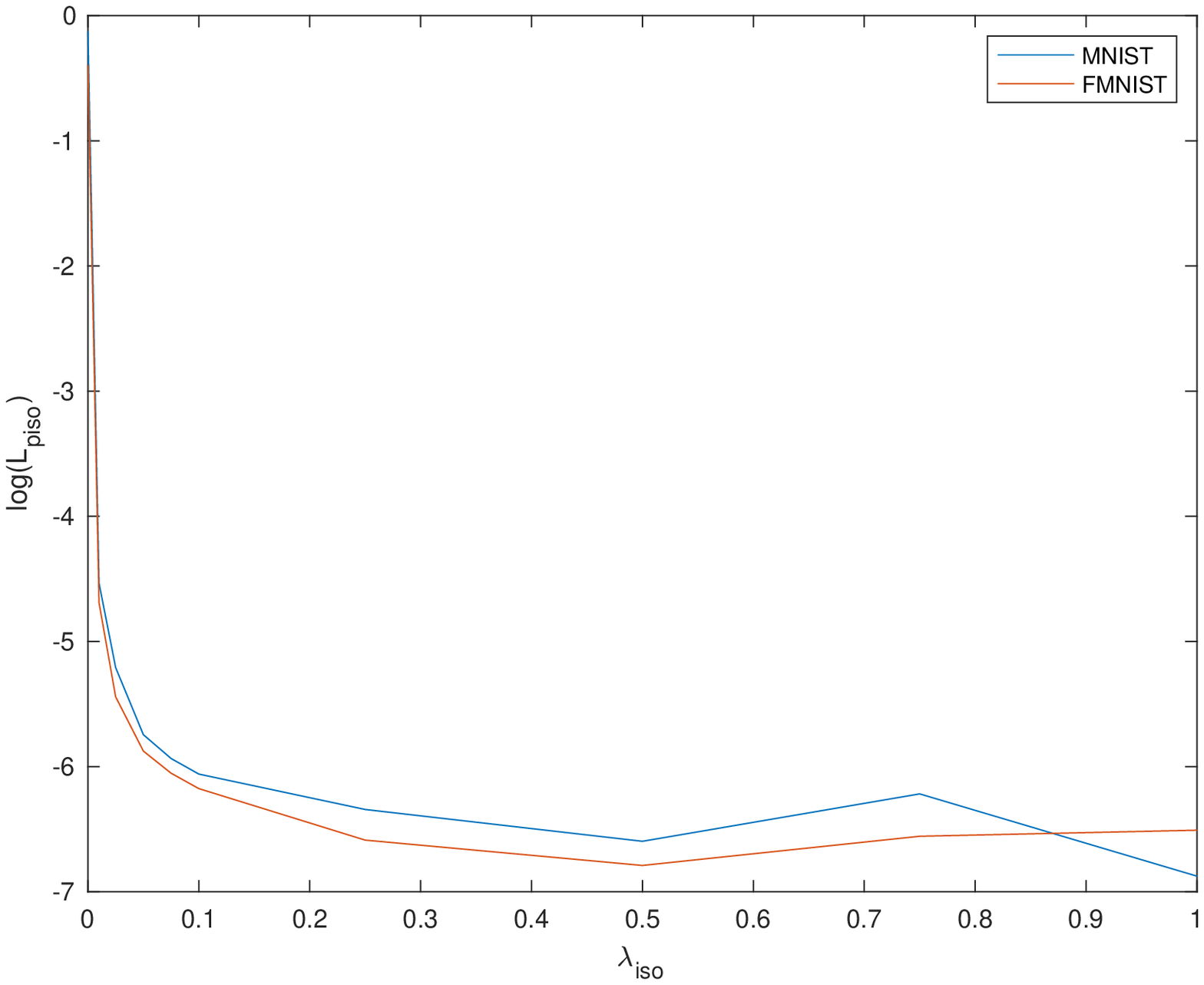}
\end{tabular}

    \caption{Sensitivity to hyper-parameters. Top: visualizations of MNIST (1st row) and FMNIST (2nd row) datasets trained with different $\lambda_{\text{iso}}$ values. Bottom: plots of the final train losses as a function of $\lambda_{\text{iso}}$; left to right: $ L_{\text{rec}} $ (linear scale), $L_{\text{iso}}$ (log scale), and $L_{\text{piso}}$ (log scale). }
    
    \label{fig:hyperparameters}
\end{figure}

\section{Conclusions}
\vspace{1pt}
We have introduced I-AE, a regularizer for autoencoders that promotes isometry of the decoder and pseudo-inverse of the encoder. Our goal was two-fold: (i) producing a favorable low dimensional manifold approximation to high dimensional data, isometrically parameterized for preserving, as much as possible, its geometric properties; and (ii) avoiding complex isometric solutions based on the notion of psuedo-inverse. Our regularizers are simple to implement and can be easily incorporated into existing autoencoders architectures. We have tested I-AE on common manifold learning tasks, demonstrating the usefulness of isometric autoencoders.

An interesting future work venue is to consider task (ii) from section \ref{s:intro}, namely incorporating I-AE losses in a probabilistic model and examine the potential benefits of the isometry prior for generative models. One motivation is the fact that isometries push probability distributions by a simple change of coordinates, $P(\vz) = P(f(\vz))$.

% \begin{ack}
% Use unnumbered first level headings for the acknowledgments. All acknowledgments
% go at the end of the paper before the list of references. Moreover, you are required to declare 
% funding (financial activities supporting the submitted work) and competing interests (related financial activities outside the submitted work). 
% More information about this disclosure can be found at: \url{https://neurips.cc/Conferences/2020/PaperInformation/FundingDisclosure}.

% Do {\bf not} include this section in the anonymized submission, only in the final paper. You can use the \texttt{ack} environment provided in the style file to autmoatically hide this section in the anonymized submission.
% \end{ack}

%\section*{References}

\bibliography{iclr_paper}

\begin{thebibliography}{47}
\providecommand{\natexlab}[1]{#1}
\providecommand{\url}[1]{\texttt{#1}}
\expandafter\ifx\csname urlstyle\endcsname\relax
  \providecommand{\doi}[1]{doi: #1}\else
  \providecommand{\doi}{doi: \begingroup \urlstyle{rm}\Url}\fi

\bibitem[Alain \& Bengio(2014)Alain and Bengio]{alain2014regularized}
Guillaume Alain and Yoshua Bengio.
\newblock What regularized auto-encoders learn from the data-generating
  distribution.
\newblock \emph{The Journal of Machine Learning Research}, 15\penalty0
  (1):\penalty0 3563--3593, 2014.

\bibitem[Ambrosio \& Soner(1994)Ambrosio and Soner]{ambrosio1994level}
Luigi Ambrosio and H~Mete Soner.
\newblock Level set approach to mean curvature flow in arbitrary codimension.
\newblock 1994.

\bibitem[Belkin \& Niyogi(2002)Belkin and Niyogi]{belkin2002laplacian}
Mikhail Belkin and Partha Niyogi.
\newblock Laplacian eigenmaps and spectral techniques for embedding and
  clustering.
\newblock In \emph{Advances in neural information processing systems}, pp.\
  585--591, 2002.

\bibitem[Burda et~al.(2016)Burda, Grosse, and
  Salakhutdinov]{Burda2016ImportanceWA}
Yuri Burda, Roger~B. Grosse, and Ruslan Salakhutdinov.
\newblock Importance weighted autoencoders.
\newblock \emph{CoRR}, abs/1509.00519, 2016.

\bibitem[Coifman \& Lafon(2006)Coifman and Lafon]{coifman2006diffusion}
Ronald~R Coifman and St{\'e}phane Lafon.
\newblock Diffusion maps.
\newblock \emph{Applied and computational harmonic analysis}, 21\penalty0
  (1):\penalty0 5--30, 2006.

\bibitem[Coifman et~al.(2005)Coifman, Lafon, Lee, Maggioni, Nadler, Warner, and
  Zucker]{coifman2005geometric}
Ronald~R Coifman, Stephane Lafon, Ann~B Lee, Mauro Maggioni, Boaz Nadler,
  Frederick Warner, and Steven~W Zucker.
\newblock Geometric diffusions as a tool for harmonic analysis and structure
  definition of data: Diffusion maps.
\newblock \emph{Proceedings of the national academy of sciences}, 102\penalty0
  (21):\penalty0 7426--7431, 2005.

\bibitem[Donoho \& Grimes(2003)Donoho and Grimes]{donoho2003hessian}
David~L Donoho and Carrie Grimes.
\newblock Hessian eigenmaps: Locally linear embedding techniques for
  high-dimensional data.
\newblock \emph{Proceedings of the National Academy of Sciences}, 100\penalty0
  (10):\penalty0 5591--5596, 2003.

\bibitem[F.R.S.(1901)]{doi:10.1080/14786440109462720}
Karl~Pearson F.R.S.
\newblock Liii. on lines and planes of closest fit to systems of points in
  space.
\newblock \emph{The London, Edinburgh, and Dublin Philosophical Magazine and
  Journal of Science}, 2\penalty0 (11):\penalty0 559--572, 1901.
\newblock \doi{10.1080/14786440109462720}.

\bibitem[Ghosh et~al.(2020)Ghosh, Sajjadi, Vergari, Black, and
  Scholkopf]{ghosh2020from}
Partha Ghosh, Mehdi S.~M. Sajjadi, Antonio Vergari, Michael Black, and Bernhard
  Scholkopf.
\newblock From variational to deterministic autoencoders.
\newblock In \emph{International Conference on Learning Representations}, 2020.
\newblock URL \url{https://openreview.net/forum?id=S1g7tpEYDS}.

\bibitem[Glorot et~al.(2011)Glorot, Bordes, and Bengio]{glorot2011deep}
Xavier Glorot, Antoine Bordes, and Yoshua Bengio.
\newblock Deep sparse rectifier neural networks.
\newblock In \emph{Proceedings of the fourteenth international conference on
  artificial intelligence and statistics}, pp.\  315--323, 2011.

\bibitem[Heusel et~al.(2017)Heusel, Ramsauer, Unterthiner, Nessler, and
  Hochreiter]{heusel2017gans}
Martin Heusel, Hubert Ramsauer, Thomas Unterthiner, Bernhard Nessler, and Sepp
  Hochreiter.
\newblock Gans trained by a two time-scale update rule converge to a local nash
  equilibrium.
\newblock In \emph{Advances in neural information processing systems}, pp.\
  6626--6637, 2017.

\bibitem[Higgins et~al.(2017)Higgins, Matthey, Pal, Burgess, Glorot, Botvinick,
  Mohamed, and Lerchner]{higgins2017beta}
Irina Higgins, Loic Matthey, Arka Pal, Christopher Burgess, Xavier Glorot,
  Matthew Botvinick, Shakir Mohamed, and Alexander Lerchner.
\newblock beta-vae: Learning basic visual concepts with a constrained
  variational framework.
\newblock \emph{Iclr}, 2\penalty0 (5):\penalty0 6, 2017.

\bibitem[Hinton \& Roweis(2003)Hinton and Roweis]{hinton2003stochastic}
Geoffrey~E Hinton and Sam~T Roweis.
\newblock Stochastic neighbor embedding.
\newblock In \emph{Advances in neural information processing systems}, pp.\
  857--864, 2003.

\bibitem[Kato et~al.(2019)Kato, Zhou, Sasaki, and Nakagawa]{kato2019rate}
Keizo Kato, Jing Zhou, Tomotake Sasaki, and Akira Nakagawa.
\newblock Rate-distortion optimization guided autoencoder for isometric
  embedding in euclidean latent space.
\newblock \emph{arXiv preprint arXiv:1910.04329}, 2019.

\bibitem[Kingma \& Ba(2014)Kingma and Ba]{kingma2014adam}
Diederik~P Kingma and Jimmy Ba.
\newblock Adam: A method for stochastic optimization.
\newblock \emph{arXiv preprint arXiv:1412.6980}, 2014.

\bibitem[Kingma \& Welling(2014)Kingma and Welling]{Kingma2014AutoEncodingVB}
Diederik~P. Kingma and Max Welling.
\newblock Auto-encoding variational bayes.
\newblock \emph{CoRR}, abs/1312.6114, 2014.

\bibitem[Krizhevsky et~al.(2009)Krizhevsky, Hinton,
  et~al.]{krizhevsky2009learning}
Alex Krizhevsky, Geoffrey Hinton, et~al.
\newblock Learning multiple layers of features from tiny images.
\newblock 2009.

\bibitem[Kruskal(1964)]{kruskal1964multidimensional}
Joseph~B Kruskal.
\newblock Multidimensional scaling by optimizing goodness of fit to a nonmetric
  hypothesis.
\newblock \emph{Psychometrika}, 29\penalty0 (1):\penalty0 1--27, 1964.

\bibitem[Kumar \& Poole(2020)Kumar and Poole]{kumar2020implicit}
Abhishek Kumar and Ben Poole.
\newblock On implicit regularization in $\beta$-vaes.
\newblock \emph{arXiv preprint arXiv:2002.00041}, 2020.

\bibitem[Kumar et~al.(2020)Kumar, Poole, and Murphy]{kumar2020regularized}
Abhishek Kumar, Ben Poole, and Kevin Murphy.
\newblock Regularized autoencoders via relaxed injective probability flow.
\newblock \emph{arXiv preprint arXiv:2002.08927}, 2020.

\bibitem[LeCun(1998)]{lecun1998mnist}
Yann LeCun.
\newblock The mnist database of handwritten digits.
\newblock \emph{http://yann. lecun. com/exdb/mnist/}, 1998.

\bibitem[Liu et~al.(2015)Liu, Luo, Wang, and Tang]{liu2015deep}
Ziwei Liu, Ping Luo, Xiaogang Wang, and Xiaoou Tang.
\newblock Deep learning face attributes in the wild.
\newblock In \emph{Proceedings of the IEEE international conference on computer
  vision}, pp.\  3730--3738, 2015.

\bibitem[Maaten \& Hinton(2008)Maaten and Hinton]{maaten2008visualizing}
Laurens van~der Maaten and Geoffrey Hinton.
\newblock Visualizing data using t-sne.
\newblock \emph{Journal of machine learning research}, 9\penalty0
  (Nov):\penalty0 2579--2605, 2008.

\bibitem[Makhzani et~al.(2015)Makhzani, Shlens, Jaitly, Goodfellow, and
  Frey]{makhzani2015adversarial}
Alireza Makhzani, Jonathon Shlens, Navdeep Jaitly, Ian Goodfellow, and Brendan
  Frey.
\newblock Adversarial autoencoders.
\newblock \emph{arXiv preprint arXiv:1511.05644}, 2015.

\bibitem[McInnes et~al.(2018)McInnes, Healy, and Melville]{mcinnes2018umap}
Leland McInnes, John Healy, and James Melville.
\newblock Umap: Uniform manifold approximation and projection for dimension
  reduction.
\newblock \emph{arXiv preprint arXiv:1802.03426}, 2018.

\bibitem[McQueen et~al.(2016)McQueen, Meila, and Joncas]{mcqueen2016nearly}
James McQueen, Marina Meila, and Dominique Joncas.
\newblock Nearly isometric embedding by relaxation.
\newblock In \emph{Advances in Neural Information Processing Systems}, pp.\
  2631--2639, 2016.

\bibitem[Nash(1956)]{nash1956imbedding}
John Nash.
\newblock The imbedding problem for riemannian manifolds.
\newblock \emph{Annals of mathematics}, pp.\  20--63, 1956.

\bibitem[Nene et~al.(1996)Nene, Nayar, Murase, et~al.]{nene1996columbia}
Sameer~A Nene, Shree~K Nayar, Hiroshi Murase, et~al.
\newblock Columbia object image library (coil-20).
\newblock 1996.

\bibitem[Pai et~al.(2019)Pai, Talmon, Bronstein, and Kimmel]{pai2019dimal}
Gautam Pai, Ronen Talmon, Alex Bronstein, and Ron Kimmel.
\newblock Dimal: Deep isometric manifold learning using sparse geodesic
  sampling.
\newblock In \emph{2019 IEEE Winter Conference on Applications of Computer
  Vision (WACV)}, pp.\  819--828. IEEE, 2019.

\bibitem[Park et~al.(2019)Park, Kim, and Kim]{park2019}
Yookoon Park, Chris~Dongjoo Kim, and Gunhee Kim.
\newblock Variational laplace autoencoders.
\newblock In \emph{ICML}, 2019.

\bibitem[Paszke et~al.(2017)Paszke, Gross, Chintala, Chanan, Yang, DeVito, Lin,
  Desmaison, Antiga, and Lerer]{paszke2017automatic}
Adam Paszke, Sam Gross, Soumith Chintala, Gregory Chanan, Edward Yang, Zachary
  DeVito, Zeming Lin, Alban Desmaison, Luca Antiga, and Adam Lerer.
\newblock Automatic differentiation in pytorch.
\newblock 2017.

\bibitem[Peterfreund et~al.(2020)Peterfreund, Lindenbaum, Dietrich, Bertalan,
  Gavish, Kevrekidis, and Coifman]{peterfreund2020loca}
Erez Peterfreund, Ofir Lindenbaum, Felix Dietrich, Tom Bertalan, Matan Gavish,
  Ioannis~G Kevrekidis, and Ronald~R Coifman.
\newblock Loca: Local conformal autoencoder for standardized data coordinates.
\newblock \emph{arXiv preprint arXiv:2004.07234}, 2020.

\bibitem[Poole et~al.(2014)Poole, Sohl-Dickstein, and
  Ganguli]{poole2014analyzing}
Ben Poole, Jascha Sohl-Dickstein, and Surya Ganguli.
\newblock Analyzing noise in autoencoders and deep networks.
\newblock \emph{arXiv preprint arXiv:1406.1831}, 2014.

\bibitem[Ranzato et~al.(2007)Ranzato, Poultney, Chopra, and
  Cun]{ranzato2007efficient}
Marc'Aurelio Ranzato, Christopher Poultney, Sumit Chopra, and Yann~L Cun.
\newblock Efficient learning of sparse representations with an energy-based
  model.
\newblock In \emph{Advances in neural information processing systems}, pp.\
  1137--1144, 2007.

\bibitem[Ranzato et~al.(2008)Ranzato, Boureau, and Cun]{ranzato2008sparse}
Marc'Aurelio Ranzato, Y-Lan Boureau, and Yann~L Cun.
\newblock Sparse feature learning for deep belief networks.
\newblock In \emph{Advances in neural information processing systems}, pp.\
  1185--1192, 2008.

\bibitem[Rifai et~al.(2011{\natexlab{a}})Rifai, Muller, Glorot, Mesnil, Bengio,
  and Vincent]{rifai2011learning}
Salah Rifai, Xavier Muller, Xavier Glorot, Gr{\'e}goire Mesnil, Yoshua Bengio,
  and Pascal Vincent.
\newblock Learning invariant features through local space contraction.
\newblock \emph{arXiv preprint arXiv:1104.4153}, 2011{\natexlab{a}}.

\bibitem[Rifai et~al.(2011{\natexlab{b}})Rifai, Vincent, Muller, Glorot, and
  Bengio]{Rifai2011ContractiveAE}
Salah Rifai, Pascal Vincent, Xavier Muller, Xavier Glorot, and Yoshua Bengio.
\newblock Contractive auto-encoders: Explicit invariance during feature
  extraction.
\newblock In \emph{ICML}, 2011{\natexlab{b}}.

\bibitem[Roweis \& Saul(2000)Roweis and Saul]{roweis2000nonlinear}
Sam~T Roweis and Lawrence~K Saul.
\newblock Nonlinear dimensionality reduction by locally linear embedding.
\newblock \emph{science}, 290\penalty0 (5500):\penalty0 2323--2326, 2000.

\bibitem[Sammon(1969)]{sammon1969nonlinear}
John~W Sammon.
\newblock A nonlinear mapping for data structure analysis.
\newblock \emph{IEEE Transactions on computers}, 100\penalty0 (5):\penalty0
  401--409, 1969.

\bibitem[S{\o}nderby et~al.(2016)S{\o}nderby, Raiko, Maal{\o}e, S{\o}nderby,
  and Winther]{sonderby2016ladder}
Casper~Kaae S{\o}nderby, Tapani Raiko, Lars Maal{\o}e, S{\o}ren~Kaae
  S{\o}nderby, and Ole Winther.
\newblock Ladder variational autoencoders.
\newblock In \emph{Advances in neural information processing systems}, pp.\
  3738--3746, 2016.

\bibitem[Tenenbaum et~al.(2000)Tenenbaum, De~Silva, and
  Langford]{tenenbaum2000global}
Joshua~B Tenenbaum, Vin De~Silva, and John~C Langford.
\newblock A global geometric framework for nonlinear dimensionality reduction.
\newblock \emph{science}, 290\penalty0 (5500):\penalty0 2319--2323, 2000.

\bibitem[Tolstikhin et~al.(2018)Tolstikhin, Bousquet, Gelly, and
  Sch{\"{o}}lkopf]{abc}
Ilya~O. Tolstikhin, Olivier Bousquet, Sylvain Gelly, and Bernhard
  Sch{\"{o}}lkopf.
\newblock Wasserstein auto-encoders.
\newblock In \emph{6th International Conference on Learning Representations,
  {ICLR} 2018, Vancouver, BC, Canada, April 30 - May 3, 2018, Conference Track
  Proceedings}. OpenReview.net, 2018.
\newblock URL \url{https://openreview.net/forum?id=HkL7n1-0b}.

\bibitem[Ulyanov(2016)]{Ulyanov2016}
Dmitry Ulyanov.
\newblock Multicore-tsne.
\newblock \url{https://github.com/DmitryUlyanov/Multicore-TSNE}, 2016.

\bibitem[Vincent et~al.(2010)Vincent, Larochelle, Lajoie, Bengio, and
  Manzagol]{vincent2010stacked}
Pascal Vincent, Hugo Larochelle, Isabelle Lajoie, Yoshua Bengio, and
  Pierre-Antoine Manzagol.
\newblock Stacked denoising autoencoders: Learning useful representations in a
  deep network with a local denoising criterion.
\newblock \emph{Journal of machine learning research}, 11\penalty0
  (Dec):\penalty0 3371--3408, 2010.

\bibitem[Xiao et~al.(2017)Xiao, Rasul, and Vollgraf]{xiao2017/online}
Han Xiao, Kashif Rasul, and Roland Vollgraf.
\newblock Fashion-mnist: a novel image dataset for benchmarking machine
  learning algorithms, 2017.

\bibitem[Zhan et~al.(2018)Zhan, Yu, Yu, Zhang, Tao, and
  Tian]{10.1145/3240508.3240607}
Yibing Zhan, Jun Yu, Zhou Yu, Rong Zhang, Dacheng Tao, and Qi~Tian.
\newblock Comprehensive distance-preserving autoencoders for cross-modal
  retrieval.
\newblock In \emph{Proceedings of the 26th ACM International Conference on
  Multimedia}, MM ’18, pp.\  1137–1145, New York, NY, USA, 2018.
  Association for Computing Machinery.
\newblock ISBN 9781450356657.
\newblock \doi{10.1145/3240508.3240607}.
\newblock URL \url{https://doi.org/10.1145/3240508.3240607}.

\bibitem[Zhao et~al.(2019)Zhao, Song, and Ermon]{zhao2019infovae}
Shengjia Zhao, Jiaming Song, and Stefano Ermon.
\newblock Infovae: Balancing learning and inference in variational
  autoencoders.
\newblock In \emph{Proceedings of the AAAI Conference on Artificial
  Intelligence}, volume~33, pp.\  5885--5892, 2019.

\end{thebibliography}
\bibliographystyle{iclr2021_conference}

\appendix
\section{Appendix}

\subsection[Implementation details]{Implementation details}
All experiments were conducted on a Tesla V100 Nvidia GPU using \textsc{pytorch} framework \cite{paszke2017automatic}.
\subsubsection{Notations}
Table \ref{tab:notation} describes the notation for the different network layers.
\begin{table}[h!]
\centering
\scriptsize
\begin{adjustbox}{max width=0.99\textwidth}
\begin{tabular}{c|c}
\hline
Notation      & Description   \\
\hline
\textsc{Lin} $n$      & Linear layer. $n$ denotes the output dimension.\\
\textsc{FC} $n$       & FullyConnected layer with SoftPlus ($\beta=100$) non linear activation. $n$ denotes the output dimension.\\
\textsc{FC\_B} $n$    & Block consisting of Lin $n$, followed by a batch normalization layer and SoftPlus ($\beta=100$) non linear activation.    \\
\textsc{Conv} $c,k,s,p$       & Convolutional layer with kernel of size $k\times k$, $c$ output channels, $s$ stride, and $p$ padding.  \\
\textsc{Conv\_B} $c,k,s,p$       & Block consisting of \textsc{Conv} $c,k,s,p$, followed by a batch normalization layer and SoftPlus($\beta=100$)  non linear activation. \\
\textsc{ConvT} $c,k,s,p$       & Convolutional transpose layer with kernel of size $k\times k$, $c$ output channels, $s$ stride, and $p$ padding.\\
\textsc{ConvT\_B} $c,k,s,p$       & Block consisting of \textsc{ConvT} $c,k,s,p$, followed by a batch normalization layer and SoftPlus($\beta=100$)  non linear activation.
\end{tabular}
\end{adjustbox}
 \vspace{5pt}
    \caption{Layers notation.}
    \label{tab:notation}
\end{table}
\subsubsection{Evaluation}
\paragraph{Architecture.} We used an autoencoder consisted of 5 \textsc{FC} 256 layers followed by a \textsc{Lin} $2$ layer for the encoder; similarly, 5 \textsc{FC} 256 layers followed by a \textsc{Lin} 3 layer were used for the decoder.

\paragraph{Training details.} All methods were trained for a relatively long period of 100K epochs. Training was done with the \textsc{Adam} optimizer \cite{kingma2014adam}, setting a fixed learning rate of $0.001$ and a full batch. I-AE parameter was set to $\lambda_{\text{iso}}=0.01$.
% \lambda_{\text{rec}}L_{\text{rec}}(\theta,\phi) + \lambda_{\text{inv}}L_{\text{inv}}(\theta,\phi) +  \lambda_{\text{iso}}L_{\text{iso}}(\theta) + \lambda_{\text{piso}}L_{\text{piso}}(\phi),
% \end{equation}

\paragraph{Baselines.} The following regularizers were used as baselines: Contractive autoencoder (CAE) \cite{Rifai2011ContractiveAE}; Contractive autoencoder with decoder weights tied to the encoder weights (TCAE) \cite{rifai2011learning}; Gradient penalty on the decoder (RAE-GP) \cite{ghosh2020from}; Denoising autoencoder with gaussian noise (DAE) \cite{vincent2010stacked}. For both CAE, and TCAE the regularization term is $\norm{ dg(\vx)}^2$. For RAE-GP the regularization term is $\norm{ df(\vz)}^2$. For U-MAP \cite{mcinnes2018umap}, we set the number of neighbours to $30$. For t-SNE \cite{maaten2008visualizing}, we set perplexity$=50$.

\subsubsection{Data visualization}

\paragraph{Architecture.} Table \ref{tab:arch_vis} lists the complete architecture details of this experiment. Both MNIST and FMNIST were trained with \textsc{FC-NN} and \textsc{S-CNN}, and COIL20 was trained with \textsc{L-CNN}.

\begin{table}[h!]
\centering
\scriptsize
\begin{adjustbox}{max width=0.99\textwidth}
\begin{tabular}{cc|cc|cc}
\hline
\multicolumn{2}{c}{\textsc{FC-NN}} & \multicolumn{2}{c}{\textsc{S-CNN}}            & \multicolumn{2}{c}{\textsc{L-CNN}}             \\ \hline
Encoder      & Decoder    & Encoder           & Decoder           & Encoder           & Decoder            \\ \hline
\textsc{FC\_B} 128       & \textsc{FC\_B} 1024    & \textsc{Conv\_B} 32,4,2,1  & \textsc{FC} 256 & \textsc{Conv\_B} 128,4,2,1  & \textsc{ConvT\_B} 2048,4,1,0 \\
\textsc{FC\_B} 256       & \textsc{FC\_B} 512     & \textsc{Conv\_B} 64,4,2,1  & \textsc{ConvT\_B} 128,4,1,0 & \textsc{Conv\_B} 256,4,2,1  & \textsc{ConvT\_B} 1024,4,2,1 \\
\textsc{FC\_B} 512       & \textsc{FC\_B} 256     & \textsc{Conv\_B} 128,4,2,1  & \textsc{ConvT\_B} 64,4,2,1 & \textsc{Conv\_B} 512,4,2,1  & \textsc{ConvT\_B} 512,4,2,1  \\
\textsc{FC\_B} 1024      & \textsc{FC\_B} 128     & \textsc{Conv\_B} 256,4,2,0 & \textsc{ConvT} 32,4,2,1   & \textsc{Conv\_B} 1024,4,2,1 & \textsc{ConvT\_B} 256,4,2,1  \\
\textsc{Lin} 2         & \textsc{Lin} 784     & \textsc{Lin} 2    &                 \textsc{ConvT} 1,4,2,3  & \textsc{Conv\_B} 2048,4,2,1 & \textsc{ConvT\_B} 128,4,2,1  \\
\textbf{}    &            &                   &                   & \textsc{Conv\_B} 4096,4,2,1 & \textsc{ConvT} 1,4,2,1    \\
             &            &                   &                   & \textsc{Conv} 2,2,2,1    &                   
\end{tabular}
\end{adjustbox}
    \vspace{10pt}
    \caption{High dimensional visualization experiment architectures.}
    \label{tab:arch_vis}
\end{table}
%\vspace{-30pt}
\paragraph{Training details.} Training was done using \textsc{Adam} optimizer \cite{kingma2014adam}. The rest of the training details are on table \ref{tab:specs_vis}.
\begin{table}[h!]
\centering
\begin{adjustbox}{max width=0.99\textwidth}
\begin{tabular}{llllll}
\hline
                                  & \multicolumn{2}{l}{MNIST}          & \multicolumn{2}{l}{FMNIST}         & COIL20 \\ \hline
\multicolumn{1}{l|}{Architecture} & \textsc{FC-NN} & \multicolumn{1}{l|}{\textsc{S-NN}}  & \textsc{FC-NN} & \multicolumn{1}{l|}{\textsc{S-CNN}} & \textsc{L-CNN}  \\ \hline
\multicolumn{1}{l|}{Batch Size}   & 128   & \multicolumn{1}{l|}{512}   & 128   & \multicolumn{1}{l|}{512}   & 144    \\
\multicolumn{1}{l|}{$\lambda_{\text{iso}}$}       & 0.1   & \multicolumn{1}{l|}{0.075} & 0.01  & \multicolumn{1}{l|}{0.075} & 0.1    \\
\multicolumn{1}{l|}{Epochs}       & 1000  & \multicolumn{1}{l|}{500}   & 1000  & \multicolumn{1}{l|}{500}   & 1000  
\end{tabular}
\end{adjustbox}
\caption{High dimensional visualization training details.}
\label{tab:specs_vis}
\end{table}

\paragraph{Baselines.} The following regularizers were used as baselines: Contractive autoencoder (CAE) \cite{Rifai2011ContractiveAE}; Gradient penalty on the decoder (RAE-GP) \cite{ghosh2020from}; Denoising autoencoder with gaussian noise (DAE) \cite{vincent2010stacked}. For CAE the regularization term is $\norm{dg(\vx)}^2$. For RAE-GP the regularization term is $\norm{ df(\vz)}^2$. We used U-MAP \cite{mcinnes2018umap} official implementation with $ \text{random\_state}=42 $, and \cite{Ulyanov2016} multicore implementation for t-SNE \cite{maaten2008visualizing} with default parameters.

\subsection{Additional Experiments}
\subsubsection{Generalization in high dimensional space}
\vspace{5pt}
Next, we evaluate how well our suggested isometric prior induces manifolds that generalizes well to unseen data. We experimented with three different images datasets: MNIST \citep{lecun1998mnist}; CIFAR10 \citep{krizhevsky2009learning}; and CelebA \citep{liu2015deep}. We quantitatively estimate methods performance by measuring the $L_2$ distance and the \emph{Fr\'{e}chet Inception Distance} (FID) \cite{heusel2017gans} on a held out test set. For each dataset, we used the official train-test splits.

For comparison versus baselines we have selected among relevant existing AE based methods the following: Vanilla AE (AE); autoencoder trained with weight decay (AEW); Contractive autoencoder  (CAE); autoencoder with spectral weights normalization (RAE-SN); and autoencoder with $L_2$ regularization on decoder weights (RAE-SN). RAE-$L_2$ and RAE-SN were recently successfully applied to this data in \citep{ghosh2020from}, demonstrating state-of-the-art performance on this task. In addition, we compare versus the Wasserstein Auto-Encoder (WAE) \cite{abc}, chosen as state-of-the-art among generative autoencoders.

For evaluation fairness, all methods were trained using the same training hyper-parameters: network architecture, optimizer settings, batch size, number of epochs for training and learning rate scheduling. See the appendix for specific hyper-parameters values. In addition, we generated a validation set out of the training set using 10k samples for the MNIST and CIFAR-10 experiment, whereas for the CelebA experiment we used the official validation set. For each training epoch, we evaluated the reconstruction $L_2$ loss on the validation set and chose the final network weights to be the one that achieves the minimum reconstruction. 
We experimented with two variants of I-AE regularizers: $L_{\text{piso}}$ and $L_{\text{piso}}+L_{\text{iso}}$. Table \ref{tab:gen} logs the results. Note that I-AE produced competitive results with the current SOTA on this task. 

\paragraph{Architecture.} For all methods, we used an autoencoder with Convolutional and Convolutional transpose layers. Table \ref{tab:arch_gen} lists the complete details.

\begin{table}[h!]
\centering
\scriptsize
\begin{adjustbox}{max width=0.99\textwidth}%{max width=\textwidth}
\begin{tabular}{cc|cc|cc}
    
    \hline
    \multicolumn{2}{c}{MNIST} & \multicolumn{2}{c}{CIFAR-10} & \multicolumn{2}{c}{CelebA} \\ \hline
    Encoder      & Decoder    & Encoder           & Decoder     & Encoder           & Decoder      \\ \hline
    \textsc{Conv\_B} $128,4,2,1$ & \textsc{FC} $16384$    & \textsc{Conv\_B} $128,4,2,1$ & \textsc{FC} $16384$ & \textsc{Conv\_B} $128,5,2,1$ & \textsc{FC} $65536$ \\
    \textsc{Conv\_B} $256,4,2,1$       & \textsc{ConvT\_B} $512,4,2,1$     & \textsc{Conv\_B} $256,4,2,1$       & \textsc{ConvT\_B} $512,4,2,1$ & \textsc{Conv\_B} $256,5,2,1$ & \textsc{ConvT\_B} $512,4,2,1$ \\
    \textsc{Conv\_B} $512,4,2,1$       & \textsc{ConvT\_B} $256,4,2,1$     & \textsc{Conv\_B} $512,4,2,1$       & \textsc{ConvT\_B} $256,4,2,1$ &  \textsc{Conv\_B} $512,5,2,1$       & \textsc{ConvT\_B} $256,4,2,1$\\
    \textsc{Conv\_B} $1024,4,2,1$      & \textsc{ConvT\_B} $128,4,2,1$     & \textsc{Conv\_B} $1024,4,2,1$      & \textsc{ConvT\_B} $128,4,2,1$ & \textsc{Conv\_B} $1024,5,2,1$      & \textsc{ConvT\_B} $128,4,2,1$    \\
    \textsc{Lin} $16$         & \textsc{ConvT} $1,1,0,0$     & \textsc{Lin} $128$         & \textsc{ConvT} $3,1,0,0$  & \textsc{Lin} 128         & \textsc{ConvT} $3,1,0,0$\\
\end{tabular}
\end{adjustbox}
    \vspace{10pt}
    \caption{High dimensional generalization experiment architectures.}
    \label{tab:arch_gen}
\end{table}
%\vspace{-30pt}

\paragraph{Training details.} Training was done with the \textsc{Adam} optimizer \cite{kingma2014adam}, setting a learning rate of $0.0005$ and  batch size $100$. I-AE parameter was set to $\lambda_{\text{iso}}=0.1$.

\begin{table}[H]
    \centering
    %\hspace{-15pt}
    %\scriptsize
    \setlength\tabcolsep{8.5pt} % default value: 6pt
    \vspace{4pt}
    \begin{tabular}{c}
        \begin{adjustbox}{max width=0.95\textwidth}%{max width=\textwidth}
            \aboverulesep=0ex
            \belowrulesep=0ex
            \renewcommand{\arraystretch}{1.9}
            \begin{tabular}[t]{c|c|c|c|c|c |c|c|c|c}
            \multicolumn{2}{c}{} & 
            \multicolumn{8}{|c}{Methods} \\
            \cmidrule{3-10}
            
                Dataset & Distance & $L_{\text{piso}}$ & $L_{\text{piso}}+L_{\text{iso}}$ & AE & AEW & CAE & RAE-SN & RAE-$L_2$ & WAE \\
                \midrule
                \multirow{ 2}{*}{MNIST} &
                $L_2$ &  \textbf{0.96}    & \color{blue}{\textbf{0.99}}  & 1.14   & 1.0   & 1.15      & 1.35  & 1.14 & 1.64 \\
                & FID  & 6.09    & 7.94  & \textbf{4.95}   & \color{blue}{\textbf{5.59}}   & 6.46 & 10.72      & 11.41  & 6.99 \\
                 \midrule
                 \multirow{2}{*}{CIFAR-10} &
                 $L_2$ & \color{blue}{\textbf{20.19}}    & 21.05  & \textbf{20.16}   & 20.33   & 20.23  & 21.02  & 20.2 & 21.08 \\
                & FID & 70.14    & \textbf{56.04}  & 74.79   & \color{blue}{\textbf{68.71}}   & 71.71      & 70.79  & 71.05 & 74.2 \\
                \midrule
                \multirow{2}{*}{CelebA} &
                 $L_2$  & 20.38 &  \color{blue}{\textbf{19.93}} & 20.51 & \textbf{19.74} & 20.46 & 20.78 & 20.58 & 20.88 \\
                & FID  & \textbf{34.68} & 40.73 & 40.53 & 40.00   & 39.52 & 40.45 & \color{blue}{\textbf{38.86}} & 38.98 \\
                %\cmidrule{3-10}
                \bottomrule
        \end{tabular} 
        \end{adjustbox}
      
    \end{tabular}
    \vspace{10pt}
    \caption{Manifold approximation quality on test images.  We log the $L_2$ and FID distances (lower is better) from reconstructed images to the input images. The $L_2$ numbers are reported $* 10^3$. The top performance scores are highlighted as: {\color{black}\textbf{First}}, {\color{blue}\textbf{Second}}.}
    \label{tab:gen}
\end{table}

\begin{figure}[h]
    \centering
    \setlength\tabcolsep{0.3pt}
    \begin{tabular}{cc}
        input & \includegraphics[width=0.8\textwidth]{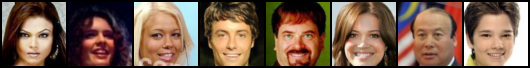} \\
        \hline
         $L_{\text{piso}}$ & \includegraphics[width=0.8\textwidth]{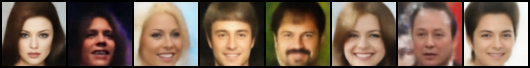} \\
         \hline
         $L_{\text{piso}}+L_{\text{iso}}$  & \includegraphics[width=0.8\textwidth]{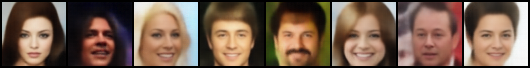} \\
         \hline
         AE & \includegraphics[width=0.8\textwidth]{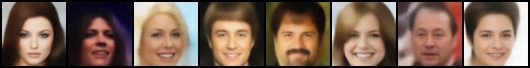} \\
         \hline
         AEW & \includegraphics[width=0.8\textwidth]{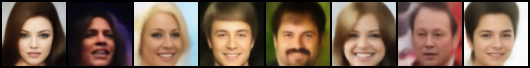} \\
         \hline
         CAE & \includegraphics[width=0.8\textwidth]{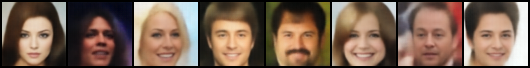} \\
         \hline
          RAE-SN & \includegraphics[width=0.8\textwidth]{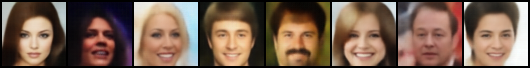} \\
         \hline
         RAE-$L_2$ & \includegraphics[width=0.8\textwidth]{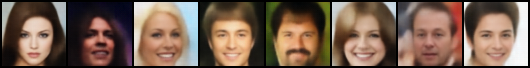} \\
         \hline
         WAE & \includegraphics[width=0.8\textwidth]{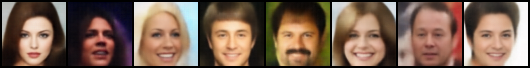} \\
        \hline
    \end{tabular}
    
    \caption{CelebA reconstructions.}
    \label{fig:recon_celeb}
\end{figure}
%\pagebreak
\begin{figure}[h]
    \centering
    \setlength\tabcolsep{0.3pt}
    \begin{tabular}{cc}
        input & \includegraphics[width=0.8\textwidth]{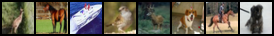} \\
        \hline
         $L_{\text{piso}}$ & \includegraphics[width=0.8\textwidth]{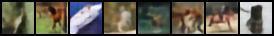} \\
         \hline
         $L_{\text{piso}}+L_{\text{iso}}$  & \includegraphics[width=0.8\textwidth]{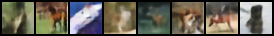} \\
         \hline
         AE & \includegraphics[width=0.8\textwidth]{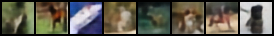} \\
         \hline
         AEW & \includegraphics[width=0.8\textwidth]{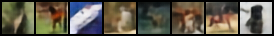} \\
         \hline
         CAE & \includegraphics[width=0.8\textwidth]{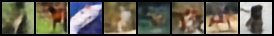} \\
         \hline
          RAE-SN & \includegraphics[width=0.8\textwidth]{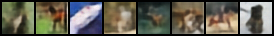} \\
         \hline
         RAE-$L_2$ & \includegraphics[width=0.8\textwidth]{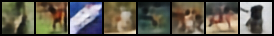} \\
         \hline
         WAE & \includegraphics[width=0.8\textwidth]{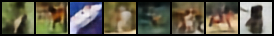} \\
        \hline
    \end{tabular}
    
    \caption{CIFAR-10 reconstructions.}
    \label{fig:recon_cifar}
\end{figure}
% \pagebreak
 
 \begin{figure}[h]
    \centering
    \setlength\tabcolsep{0.3pt}
    \begin{tabular}{cc}
        input & \includegraphics[width=0.8\textwidth]{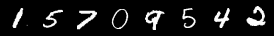} \\
        \hline
         $L_{\text{piso}}$ & \includegraphics[width=0.8\textwidth]{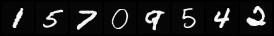} \\
         \hline
         $L_{\text{piso}}+L_{\text{iso}}$  & \includegraphics[width=0.8\textwidth]{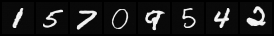} \\
         \hline
         AE & \includegraphics[width=0.8\textwidth]{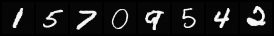} \\
         \hline
         AEW & \includegraphics[width=0.8\textwidth]{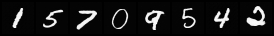} \\
         \hline
         CAE & \includegraphics[width=0.8\textwidth]{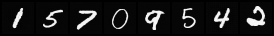} \\
         \hline
          RAE-SN & \includegraphics[width=0.8\textwidth]{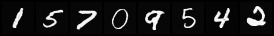} \\
         \hline
         RAE-$L_2$ & \includegraphics[width=0.8\textwidth]{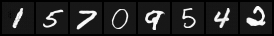} \\
         \hline
         WAE & \includegraphics[width=0.8\textwidth]{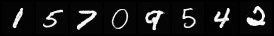} \\
        \hline
    \end{tabular}
    
    \caption{MNIST reconstructions.}
    \label{fig:recon_mnist}
\end{figure}
\end{document}